\documentclass[twoside,11pt]{article}

\usepackage{jmlr2e}
\usepackage{natbib}

\usepackage{amssymb,amsmath}
\usepackage{empheq}
\usepackage{mathtools}
\mathtoolsset{showonlyrefs}
\usepackage{mathbbol}
\usepackage{bm}
\usepackage{nicefrac}
\usepackage{booktabs}

\usepackage{enumitem}

\usepackage{parskip}

\usepackage{xcolor}
\usepackage{xspace}
\usepackage{algorithm}
\usepackage{algorithmic}
\usepackage{pgf}
\usepackage{wrapfig}

\usepackage{mdframed}
\definecolor{theoremcolor}{rgb}{0.94, 0.97, 1.0}  
\mdfsetup{
    backgroundcolor=theoremcolor,
    linewidth=0pt,
}

\newmdtheoremenv{definition}{Definition} 
\newmdtheoremenv{proposition}{Proposition}
\newmdtheoremenv{corollary}{Corollary} 
\newmdtheoremenv{theorem}{Theorem} 
\newmdtheoremenv{lemma}{Lemma} 

\usepackage{hyperref}
\ifdefined\nohyperref\else\ifdefined\hypersetup
  \definecolor{mydarkblue}{rgb}{0,0.08,0.45}
  \hypersetup{ %
    pdftitle={},
    pdfauthor={},
    pdfsubject={},
    pdfkeywords={},
    pdfborder=0 0 0,
    pdfpagemode=UseNone,
    colorlinks=true,
    linkcolor=mydarkblue,
    citecolor=mydarkblue,
    filecolor=mydarkblue,
    urlcolor=mydarkblue,
    pdfview=FitH}

  \ifdefined\isaccepted \else
    \hypersetup{pdfauthor={Anonymous Submission}}
  \fi
\fi\fi

\usepackage[open=true]{bookmark}

\DeclareMathOperator*{\argmin}{\mathsf{argmin}}
\DeclareMathOperator*{\argmax}{\mathsf{argmax}}
\DeclareMathOperator*{\conv}{\mathsf{conv}}
\DeclareMathOperator*{\cone}{\mathsf{cone}}
\DeclareMathOperator*{\dom}{\mathsf{dom}}
\DeclareMathOperator*{\interior}{\mathsf{int}}
\DeclareMathOperator*{\relint}{\mathsf{relint}}
\def\prox{{\mathsf{prox}}}

\DeclareMathOperator*{\sigmoid}{\mathsf{sigmoid}}
\DeclareMathOperator*{\softmax}{\mathsf{softmax}}
\DeclareMathOperator*{\sparsemax}{\mathsf{sparsemax}}
\DeclareMathOperator*{\marginals}{\mathsf{marginals}}
\DeclareMathOperator*{\sparsemap}{\mathsf{SparseMAP}}
\DeclareMathOperator*{\map}{\mathsf{MAP}}
\def\prox{\mathsf{prox}}

\def\RR{{\mathbb{R}}}
\def\EE{{\mathbb{E}}}
\def\II{{\mathbb{I}}}

\def\RRY{\RR^{|\cY|}}
\def\triangleY{\triangle^{|\cY|}}
\def\sizeY{{|\cY|}}

\def\cC{{\mathcal{C}}}

\def\cX{{\mathcal{X}}}
\def\cY{{\mathcal{Y}}}

\def\ones{\bm{1}}
\def\zeros{\bm{0}}
\def\e{\bm{e}}
\def\m{{\bm{m}}}
\def\q{\bm{s}}
\def\p{{\bm{p}}}
\def\w{{\bm{w}}}
\def\x{{\bm{x}}}
\def\y{{\bm{y}}}

\def\qtheta{\q_{\s}}

\def\cc{\bm{c}}
\def\balpha{\bm{\alpha}}
\def\bpi{\bm{\pi}}
\def\bpsi{\bm{\psi}}
\def\bmu{\bm{\mu}}

\def\s{\bm{\theta}}

\def\pp{p}
\def\qq{s}
\def\ss{\theta}

\def\yHatOmega{\widehat{\y}_{\Omega}}
\def\yHatPsi{\widehat{\y}_{\Psi}}

\newcommand{\DP}[2]{{\langle #1, #2\rangle}}
\newcommand\logls{logistic\xspace}
\newcommand\wrt{w.r.t.\ }

\def\thresh{\tau}

\def\KL{\mathsf{KL}}
\def\HH{\mathsf{H}}
\def\HHs{\HH^{\textsc{s}}}
\def\HHt{\HH^{\textsc{t}}}
\def\HHr{\HH^{\textsc{r}}}
\def\HHn{\HH^{\textsc{n}}}
\def\HHsq{\HH^{\textsc{sq}}}

\def \mytitle {Learning with Fenchel-Young Losses}

\ShortHeadings{\mytitle}{Blondel, Martins, and Niculae}
\firstpageno{1}

\usepackage{lastpage}
\jmlrheading{21}{2020}{1-\pageref{LastPage}}{1/19; Revised
10/19}{1/20}{19-021}{Mathieu Blondel, Andr\'e F.T. Martins, Vlad Niculae}

\begin{document}

\title{\mytitle}

\author{\name Mathieu Blondel \email mathieu@mblondel.org \\
       \addr NTT Communication Science Laboratories \\
       Kyoto, Japan
       \AND
       \name Andr\'e F. T. Martins \email andre.martins@unbabel.com \\
       \addr Unbabel \& Instituto de Telecomunica\c{c}\~oes \\
    Lisbon, Portugal
       \AND
       \name Vlad Niculae \email vlad@vene.ro \\
       \addr Instituto de Telecomunica\c{c}\~oes \\
    Lisbon, Portugal}

\editor{Sathiya Keerthi}

\maketitle

\begin{abstract}
Over the past decades, numerous loss functions have been been proposed for a
variety of supervised learning tasks, including regression, classification,
ranking, and more generally structured prediction. Understanding the core
principles and theoretical properties underpinning these losses is key to choose the 
right loss for the right problem, as well as to  create new losses which combine their
strengths.  In this paper, we introduce Fenchel-Young losses, a generic way to
\textit{construct} a convex loss function for a regularized prediction function. 
We provide an in-depth study of
their properties in a very broad setting, covering all the aforementioned
supervised learning tasks, and revealing new connections between sparsity, 
generalized entropies, and separation margins. 
We show that Fenchel-Young losses unify many
well-known loss functions and allow to create useful new ones easily.  Finally,
we derive efficient predictive and training algorithms, making Fenchel-Young
losses appealing both in theory and practice.
\end{abstract}


\section{Introduction}
\label{sec:intro}

Loss functions are a cornerstone of statistics and machine learning: They
measure the difference, or ``loss,'' between a ground truth and a prediction.
As such, much work has been devoted to designing loss functions for a variety
of supervised learning tasks, including regression \citep{huber_1964},
classification \citep{multiclass_svm}, ranking \citep{joachims_2002} and
structured prediction
\citep{Lafferty2001,structured_perceptron,structured_hinge}, to name only a few
well-known directions.

For the case of probabilistic classification, proper composite loss functions~
\citep{reid_composite_binary,vernet_2016} offer a principled framework unifying
various existing loss functions.  A proper composite loss is the composition of 
a proper loss between two probability distributions, with an invertible mapping 
from real vectors to probability distributions.  The theoretical properties of
proper loss functions, also known as proper scoring rules
(\citet{grunwald_2004}; \citet{gneiting_2007}; and references therein), such as
their Fisher consistency (classification calibration) and correspondence with
Bregman divergences, are now well-understood. However, not all existing losses
are proper composite loss functions; a notable example is the hinge loss used in 
support vector machines. In fact, we shall see that any loss function enjoying a
separation margin, a prevalent concept in statistical learning theory which has
been used to prove the famous perceptron mistake bound \citep{perceptron} and
many other generalization bounds \citep{Vapnik1998,Schoelkopf2002}, cannot be
written in composite proper loss form.

At the same time, loss functions are often intimately related to an underlying
statistical model and prediction function.  For instance, the logistic loss
corresponds to the multinomial distribution and the softmax operator, while the
conditional random field (CRF) loss \citep{Lafferty2001} for structured
prediction is tied with marginal inference \citep{wainwright_2008}. 
Both are instances of generalized linear models \citep{glm,mccullagh_1989},
associated with exponential family distributions.  More recently,
\citet{sparsemax} proposed a new classification loss based on the projection
onto the simplex.  Unlike the logistic loss, this ``sparsemax'' loss induces
probability distributions with \textbf{sparse} support, which is desirable in
some applications for interpretability or computational efficiency reasons.
However, the sparsemax loss was derived in a relatively ad-hoc manner and it is
still relatively poorly understood. Is it one of a kind or can we generalize it
in a principled manner? Thorough understanding of the core
principles underpinning existing losses and their associated predictive model,
potentially enabling the creation of useful new losses, is one of the main
quests of this paper.

\paragraph{This paper.}

The starting point of this paper are the notions of \textbf{output}
regularization and regularized prediction functions, which we use to
provide a variational perspective on many existing prediction 
functions, including the aforementioned softmax, sparsemax
and marginal inference.  Based on simple convex duality
arguments, we then introduce Fenchel-Young losses, a new way to automatically
construct a loss function associated with any regularized prediction function.
As we shall see, our proposal recovers many existing loss functions, which is in
a sense surprising since many of these losses were originally proposed by
independent efforts. Our framework goes beyond the simple probabilistic
classification setting: We show how to create loss functions over various
structured domains, including convex polytopes and convex cones.  Our framework encourages
the loss designer to think \textbf{geometrically} about the outputs desired for
the task at hand.  Once a (regularized) prediction function has been designed,
our framework generates a corresponding loss function automatically. 
We will demonstrate the ease of creating loss functions, including useful new
ones, using abundant examples throughout this paper.

\paragraph{Previous papers.}

This paper builds upon two previously published shorter conference papers. The
first \citep{sparsemap} introduced Fenchel-Young losses in the structured
prediction setting but only provided a limited analysis of their properties.
The second \citep{fy_losses} provided a more in-depth analysis but focused on
unstructured probabilistic classification. This paper provides a comprehensive
study of Fenchel-Young losses across various domains. Besides a much more
thorough treatment of previously covered topics, this paper contributes entirely
new sections, including \S\ref{sec:positive_measures} on losses for positive
measures, \S\ref{sec:training_algorithms} on primal and dual training
algorithms, and \S\ref{sec:loss_fenchel_yougization} on loss
``Fenchel-Youngization''. We provide in \S\ref{sec:structured_prediction} a new
unifying view between structured predictions losses,
and discuss at length various convex polytopes, promoting a
geometric approach to structured prediction loss design; we also 
provide novel results in this section regarding structured separation margins
(Proposition~\ref{prop:structured_margin}), proving the unit margin of the
SparseMAP loss. 
We demonstrate how to use our framework to create useful new losses,
including ranking losses, not covered in the previous two papers. 

\paragraph{Notation.}

We denote the $(d-1)$-dimensional probability simplex by $\triangle^d \coloneqq
\{\p \in \RR_+^d \colon \|\p\|_1 = 1\}$.
We denote the convex hull of a set $\cY$ by $\conv(\cY)$ and the conic hull by
$\cone(\cY)$.
We denote the domain of a function $\Omega \colon \RR^d \rightarrow
\RR\cup\{\infty\}$ by $\dom(\Omega) \coloneqq \{\bmu \in \RR^d \colon
\Omega(\bmu) < \infty\}$. 
We denote the Fenchel conjugate of $\Omega$ by 
$\Omega^*(\s) \coloneqq \sup_{\bmu \in \dom(\Omega)} \DP{\s}{\bmu} -
\Omega(\bmu)$. 
We denote the indicator function of a set $\cC$ by
\begin{equation}
I_\cC(\bmu) \coloneqq
\begin{cases}
    0 & \mbox { if } \bmu \in \cC \\
    \infty & \mbox{ otherwise }
\end{cases}
\label{eq:indicator_function}
\end{equation}
and its support function by $\sigma_\cC(\s) \coloneqq \sup_{\bmu \in \cC}
\DP{\s}{\bmu}$.
We define the proximity operator (a.k.a. proximal operator) of $\Omega$ by
\begin{equation}
    \prox_{\Omega}(\bm{\eta}) \coloneqq
\argmin_{\bmu \in \dom(\Omega)} 
\frac{1}{2} \|\bmu - \bm{\eta} \|^2 + \Omega(\bmu).
\label{eq:prox}
\end{equation}
We denote the interior and relative interior of $\cC$ by $\interior(\cC)$ and
$\relint(\cC)$, respectively.
We denote $[\x]_+ \coloneqq \max(\x, \bm{0})$, evaluated element-wise.

\paragraph{Table of contents.}

\begin{itemize}[topsep=0pt,itemsep=3pt,parsep=3pt,leftmargin=15pt]

\item[\S\ref{sec:intro}] Introduction;

\item[\S\ref{sec:reg_pred}] Regularized prediction functions;

\item[\S\ref{sec:fy_losses}] Fenchel-Young losses;

\item[\S\ref{sec:proba_clf}] Probabilistic prediction with Fenchel-Young losses;

\item[\S\ref{sec:margin}] Separation margin of Fenchel-Young losses;

\item[\S\ref{sec:positive_measures}] Positive measure prediction with
    Fenchel-Young losses;

\item[\S\ref{sec:structured_prediction}] Structured prediction with
    Fenchel-Young losses;

\item[\S\ref{sec:training_algorithms}] Algorithms for learning with
    Fenchel-Young losses;

\item[\S\ref{sec:experiments}] Experiments;

\item[\S\ref{sec:related_work}] Related work.

\end{itemize}

\section{Regularized prediction functions}
\label{sec:reg_pred}

In this section, we introduce the concept of regularized prediction function
(\S\ref{sec:reg_pred_def}), which is central to this paper.  We then give
simple and well-known examples of such functions (\S\ref{sec:reg_pred_examples})
and discuss their properties in a general setting (\S\ref{sec:gradient_mapping},\S\ref{sec:reg_pred_prop}).

\subsection{Definition}
\label{sec:reg_pred_def}

We consider a general predictive setting with input variables $\x \in \cX$, and
a parametrized model $\bm{f}_{W}:\cX \rightarrow \RR^d$ (which could be a linear
model or a neural network), producing a score
vector $\s \coloneqq \bm{f}_{W}(\x) \in \RR^d$. In a simple multi-class
classification setting, the score vector is typically used to pick the
highest-scoring class among $d$ possible ones
\begin{equation}
\widehat{y}(\s) \in \argmax_{j \in [d]} \theta_j.
\label{eq:prediction_multiclass}
\end{equation}
This can be generalized to an arbitrary output space $\cY \subseteq \RR^d$
by using instead
\begin{equation}
\widehat{\y}(\s) \in
\argmax_{\y \in \cY} {\DP{\s}{\y}},
\label{eq:map_oracle}
\end{equation}
where intuitively $\DP{\s}{\y}$ captures the affinity between $\x$ (since $\s$
is produced by $f_W(\x)$) and $\y$.
Therefore, \eqref{eq:map_oracle} seeks the output $\y$ with greatest
affinity with $\x$.
The support function $\sigma_\cY(\s) = \max_{\y \in \cY} {\DP{\s}{\y}} =
\DP{\s}{\widehat{\y}(\s)}$ can be
interpreted as the largest projection of any element of $\cY$ onto the line
generated by $\s$.

Clearly, \eqref{eq:map_oracle} recovers \eqref{eq:prediction_multiclass} with
$\cY = \{\e_1, \dots, \e_d\}$, where $\e_j$ is a standard basis vector,
$\e_j \coloneqq [0, \dots 0, \underbrace{1}_{j}, 0, \dots 0]$.
In this case, the cardinality $\sizeY$ and the dimensionality
$d$ coincide, but this need not be the case in general.
Eq.\ \eqref{eq:map_oracle} is often called a linear maximization oracle or
maximum a-posteriori (MAP) oracle \citep{wainwright_2008}. The latter name comes
from the fact that \eqref{eq:map_oracle} coincides with the mode of the Gibbs
distribution defined by
\begin{equation}
p(\y; \s) \propto \exp \DP{\s}{\y}.
\end{equation}

\paragraph{Prediction over convex hulls.}

We now extend the prediction function \eqref{eq:map_oracle} by replacing $\cY$
with its convex hull $\conv(\cY) \coloneqq \{\EE_\p[Y] \colon \p \in
\triangleY\}$ and introducing a regularization function $\Omega$ into the
optimization problem:
\begin{equation}
\yHatOmega(\s) 
\in \argmax_{\bmu \in \conv(\cY)} {\DP{\s}{\bmu}} - \Omega(\bmu).
\label{eq:prediction_convex_hull}
\end{equation}
We emphasize that the regularization is \wrt predictions (outputs) and not \wrt
model parameters (denoted by $W$ in this paper), as is usually the case in the
literature. We illustrate the regularized prediction function pipeline in Figure
\ref{fig:pipeline}.

Unsurprisingly, the choice $\Omega=0$ recovers the unregularized prediction
function \eqref{eq:map_oracle}.
This follows from the fundamental
theorem of linear programming
\citep[Theorem 6]{dantzig}, which states
that the maximum of a linear form over a
convex polytope is always achieved at one
of its vertices:
\begin{equation}
\widehat{\y}_0(\s) \in 
\argmax_{\bmu \in \conv(\cY)} \DP{\s}{\bmu}
= \argmax_{\y \in \cY} \DP{\s}{\y}.
\end{equation}

\paragraph{Why regularize outputs?}

The regularized prediction function \eqref{eq:prediction} casts computing a
prediction as a
variational problem. It involves an
optimization problem that balances between two terms: an ``affinity'' term
$\DP{\s}{\bmu}$, and a ``confidence'' term $\Omega(\bmu)$ which should be low if
$\bmu$ is ``uncertain.''  Two important classes of convex $\Omega$ are (squared)
norms and, when $\dom(\Omega)$ is the probability simplex, generalized negative
entropies.  However, our framework does \textbf{not} require $\Omega$ to be
convex in general.

Introducing $\Omega(\bmu)$ in \eqref{eq:prediction} tends to move
the prediction away from the vertices of $\conv(\cY)$.  Unless the
regularization term $\Omega(\bmu)$ is negligible compared to the affinity term
$\DP{\s}{\bmu}$, a prediction becomes a convex combination of several vertices. As
we shall see in \S\ref{sec:structured_prediction}, which is dedicated to
structured prediction over convex hulls, we can interpret
this prediction as the mean under some underlying distribution. This
contrasts with \eqref{eq:map_oracle}, which always outputs the most
likely vertex, i.e., the mode.

\begin{figure}[t]
\centering
\includegraphics[width=0.8\linewidth]{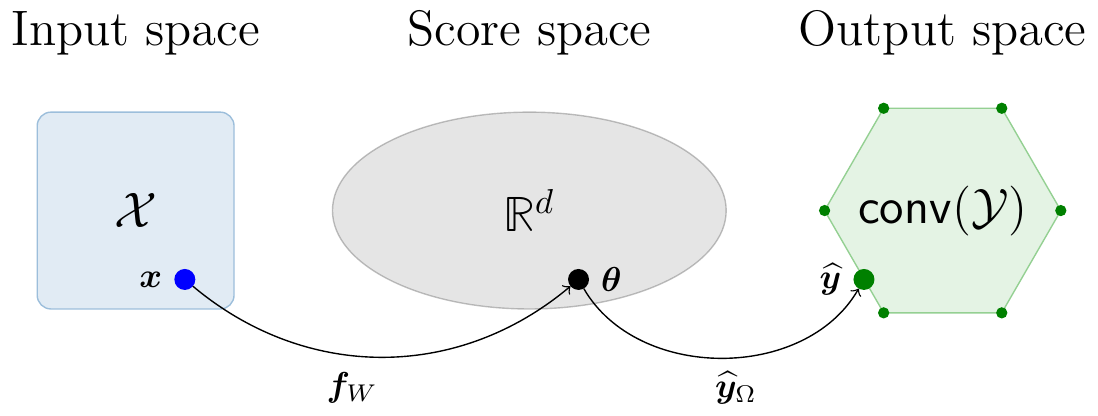}\\
\caption{{\bf Illustration of the proposed regularized prediction framework over
    a convex hull $\conv(\cY)$}. A parametrized model
$\bm{f}_W \colon \cX \to \RR^d$ (linear model, neural network, etc.)
produces a score vector $\s \in \RR^d$.
The regularized prediction function $\yHatOmega$ produces a prediction
$\widehat{\y} \in \conv(\cY)$. Regularized
prediction functions are not limited to convex hulls and can be defined over
arbitrary domains (Def.\ \ref{def:regularized_pred}).}
\label{fig:pipeline}
\end{figure}

\paragraph{Prediction over arbitrary domains.}

Regularized prediction functions are in fact not limited to convex hulls. We now
state their precise definition in complete generality.
\vspace{0.5em}
\begin{definition}{\label{def:regularized_pred}Prediction function regularized by $\Omega$}

Let $\Omega \colon \RR^d \to \RR \cup \{\infty\}$ be a regularization function,
with $\dom(\Omega) \subseteq \RR^d$.  The prediction function regularized by
$\Omega$ is defined by
\begin{equation}
\widehat{\y}_{\Omega}(\s) 
\in \argmax_{\bmu \in
\dom(\Omega)} {\DP{\s}{\bmu}} - \Omega(\bmu).
\label{eq:prediction}
\end{equation}
\end{definition}
Allowing extended-real $\Omega$ permits
general
domain constraints in \eqref{eq:prediction} via indicator functions.
For instance, choosing $\Omega = I_{\conv(\cY)}$,
where $I_\cC$ is the indicator function defined in
\eqref{eq:indicator_function},
recovers the MAP oracle \eqref{eq:map_oracle}.
Importantly, the choice of domain $\dom(\Omega)$ is not limited to convex hulls.
For instance, we will also consider conic hulls, $\cone(\cY)$, later in this
paper.

\paragraph{Choosing $\mathbf{\Omega}$.}

Regularized prediction functions $\yHatOmega$ involve two main design choices:
the domain $\dom(\Omega)$ over which $\Omega$ is defined and $\Omega$ itself.
The choice of $\dom(\Omega)$ is mainly dictated by the type of output
we want from $\yHatOmega$, such as
$\dom(\Omega)=\conv(\cY)$ for convex combinations of elements of $\cY$, and
$\dom(\Omega)=\cone(\cY)$ for conic combinations. The choice of the
regularization $\Omega$ itself further governs certain properties of
$\yHatOmega$, including, as we shall see in the sequel, its sparsity or its use
of prior knowledge regarding the importance or misclassification cost of certain
outputs.
The choices of $\dom(\Omega)$ and $\Omega$ may also be constrained by
computational considerations. Indeed, while computing $\yHatOmega(\s)$ involves
a potentially challenging constrained maximization problem in general, we will
see that certain choices of $\Omega$ lead to closed-form expressions.
The power of our framework is that the user can focus solely on designing and
computing $\yHatOmega$: We will see in \S\ref{sec:fy_losses} how to
automatically construct a loss function associated with $\yHatOmega$.

\subsection{Examples}
\label{sec:reg_pred_examples}

To illustrate regularized prediction functions, we give several concrete
examples enjoying a \textbf{closed-form expression}.

When $\Omega = I_{\triangle^d}$,
$\yHatOmega(\s)$ is a one-hot representation of the argmax
prediction
\begin{equation}
\yHatOmega(\s) \in
\argmax_{\p \in \triangle^d} ~ \langle \s, \p \rangle =
\argmax_{\y \in \{\e_1,\dots,\e_d\}} \langle \s, \y \rangle.
\end{equation}
We can see that output as a probability distribution that assigns all probability
mass on the same class.
When $\Omega = -\HHs + I_{\triangle^d}$, where
$\HHs(\p) \coloneqq -\sum_i p_i \log p_i$ is Shannon's entropy,
$\yHatOmega(\s)$ is the well-known softmax
\begin{equation}
\yHatOmega(\s)
= \softmax(\s) \coloneqq
\frac{\exp(\s)}{\sum_{j=1}^d \exp(\ss_j)}.
\label{eq:softmax}
\end{equation}
See \citet[Ex.\ 3.25]{boyd_book} for a derivation. The resulting distribution
always has \textbf{dense} support.
When $\Omega=\frac{1}{2} \|\cdot\|^2 + I_{\triangle^d}$, 
$\yHatOmega$ is the
Euclidean projection onto the probability simplex
\begin{equation}
\yHatOmega(\s) = \sparsemax(\s) \coloneqq
\argmin_{\p \in \triangle^d} \|\p - \s\|^2,
\label{eq:sparsemax}
\end{equation}
a.k.a.\ the sparsemax transformation \citep{sparsemax}.
It is well-known that
\begin{equation}
\argmin_{\p \in \triangle^d} \|\p - \s\|^2
= [\s - \tau \ones]_+,
\label{eq:sparsemax_thresholded}
\end{equation}
for some threshold $\tau \in \RR$.
Hence, the predicted distribution can have \textbf{sparse} support (it may
assign exactly zero probability to low-scoring classes). The threshold $\tau$
can be computed \textbf{exactly} in $O(d)$ time
\citep{Brucker1984,duchi,Condat2016}.

The regularized prediction function paradigm is, however, not limited to the
probability simplex: When $\Omega(\p) = -\sum_i \HHs([p_i,1-p_i]) +
I_{[0,1]^d}(\p)$, we get
\begin{equation}
\yHatOmega(\s) = \sigmoid(\s)
\coloneqq \frac{\ones}{\ones + \exp(-\s)},
\label{eq:sigmoid}
\end{equation}
i.e., the sigmoid function evaluated coordinate-wise. We can think of its
output as a positive measure (unnormalized probability distribution).

We will see in \S\ref{sec:proba_clf} that the first three examples (argmax,
softmax and sparsemax) are particular instances of a broader family of
prediction functions, using the notion of \textbf{generalized entropy}. The last
example is a special case of regularized prediction function over positive
measures, developed in \S\ref{sec:positive_measures}.  Regularized prediction
functions also encompass more complex convex polytopes for structured
prediction, as we shall see in \S\ref{sec:structured_prediction}.

\begin{figure}[t]
\centering
\includegraphics[scale=3.0]{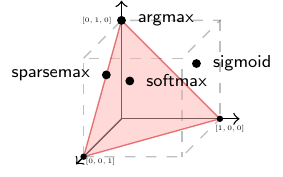}
\caption{{\bf Examples of regularized prediction functions}. The unregularized
$\argmax$ prediction always hits a vertex of the probability simplex,
leading to a probability distribution that puts all probability mass on the same
class. Unlike, $\softmax$ which
always occurs in the relative interior of the simplex and thus leads to a dense
distribution, $\sparsemax$ (Euclidean
projection onto the probability simplex) may hit the boundary, leading
to a \textbf{sparse} probability distribution. We also display the $\sigmoid$
operator which lies in the unit cube and is thus
not guaranteed to output a valid probability distribution.}
\label{fig:simplex}
\end{figure}

\subsection{Gradient mapping and dual objective}
\label{sec:gradient_mapping}

From Danskin's theorem
\citep{danskin_theorem} (see also \citet[Proposition B.25]{bertsekas_book})
$\yHatOmega(\s)$ is a subgradient of $\Omega^*$ at $\s$, i.e.,
$\yHatOmega(\s) \in \partial \Omega^*(\s)$.
If, furthermore, $\Omega$ is strictly convex,
then $\yHatOmega(\s)$ is the gradient of $\Omega^*$ at $\s$, i.e.,
$\yHatOmega(\s) = \nabla \Omega^*(\s)$.
This interpretation of $\yHatOmega(\s)$ as a (sub)gradient mapping
will play a crucial role in the next section for deriving a
loss function associated with $\yHatOmega(\s)$.

Viewing $\yHatOmega(\s)$ as the (sub)gradient of $\Omega^*(\s)$ is also
useful to derive the dual of \eqref{eq:prediction}.
Let $\Omega \coloneqq \Psi + \Phi$. It is well-known
\citep{borwein_2010,beck_2012} that
    \begin{equation}
        \Omega^*(\s) =
        (\Psi + \Phi)^*(\s) =
        \inf_{\bm{u} \in \RR^d} \Phi^*(\bm{u}) + \Psi^*(\s - \bm{u})
        \eqqcolon (\Phi^* \square \Psi^*)(\s),
        \label{eq:inf_convolution}
    \end{equation}
where $f \square g$ denotes the infimal convolution of $f$ with $g$.
Furthermore, from Danskin's theorem, $\yHatOmega(\s) = 
\nabla \Psi^*(\s - \bm{u}^\star)$, where $\bm{u}^\star$ denotes an optimal
solution of the infimum in \eqref{eq:inf_convolution}.
We can think of that infimum as the dual
of the optimization problem in \eqref{eq:prediction}.
When $\Psi = \Psi^* = \frac{1}{2} \|\cdot\|^2$,
$\Omega^*(\s)$ is known as the Moreau envelope of $\Phi^*$
\citep{moreau_proximite_1965} and
using Moreau's decomposition, we obtain 
$\yHatOmega(\s) = \s - \prox_{\Phi^*}(\s) = \prox_\Phi(\s)$.
As another example, when $\Omega = \Psi + I_\cC$, we obtain
\begin{equation}
\Omega^*(\s) = \inf_{\bm{u} \in \RR^d} \sigma_\cC(\bm{u}) + \Psi^*(\s -
\bm{u}), 
\end{equation}
where we used $I_\cC^* = \sigma_\cC$, the support function of $\cC$.
In particular, when $\cC = \conv(\cY)$, we have $\sigma_\cC(\bm{u}) = \max_{\y
\in \cY} \DP{\bm{u}}{\y}$. This dual view is informative insofar as it suggests
that regularized prediction functions $\yHatOmega(\s)$ with $\Omega = \Psi +
I_\cC$ minimize a trade-off between maximizing the value achieved
by the unregularized prediction function $\sigma_\cC(\bm{u})$, and a proximity
term $\Psi^*(\s - \bm{u})$. 

\subsection{Properties}
\label{sec:reg_pred_prop}

We now discuss simple yet useful properties of regularized prediction functions.
The first two assume that $\Omega$ is a symmetric function, i.e.,
that it satisfies
\begin{equation}
\Omega(\bmu)=\Omega(\bm{P}\bmu)
\quad \forall \bmu \in \dom(\Omega), \forall \bm{P} \in \mathcal{P},
\end{equation}
where $\mathcal{P}$ is the set of $d \times d$ permutation matrices.

\vspace{0.5em}
\begin{proposition}{Properties of regularized prediction functions
    $\yHatOmega(\s)$}

\begin{enumerate}[topsep=0pt,itemsep=3pt,parsep=3pt,leftmargin=15pt]

\item {\bf Effect of a permutation.} If $\Omega$ is symmetric, then
$\forall \bm{P} \in \mathcal{P}$:
    $\yHatOmega(\bm{P} \s) = \bm{P} \yHatOmega(\s)$.

\item {\bf Order preservation.}
 Let $\bmu = \yHatOmega(\s)$. If $\Omega$ is symmetric, then the coordinates of
    $\bmu$ and $\s$ are sorted the same way, i.e., $\ss_i > \ss_j \Rightarrow
    \mu_i \ge \mu_j$ and $\mu_i > \mu_j \Rightarrow \ss_i > \ss_j$.

\item {\bf Approximation error.} 
    Assume $\mathcal{Y} \subseteq \dom(\Omega)$.
    If $\Omega$ is $\gamma$-strongly convex 
    and bounded with
    $L \le \Omega(\bmu) \le U$ for all $\bmu \in \dom(\Omega)$,
    then
    $\frac{1}{2} \|\widehat{\y}(\s) - \yHatOmega(\s)\|^2 \le
    \frac{U-L}{\gamma}$.

\item {\bf Temperature scaling.} For any constant $t > 0$,
$\widehat{\y}_{t \Omega}(\s) \in \partial \Omega^*(\nicefrac{\s}{t})$.
If $\Omega$ is strictly convex,
$\widehat{\y}_{t \Omega}(\s) = \yHatOmega(\nicefrac{\s}{t}) = \nabla
\Omega^*(\nicefrac{\s}{t})$.

\item {\bf Constant invariance.} For any constant $c \in \RR$, 
    $\widehat{\y}_{\Omega + c}(\s) = \yHatOmega(\s; \y)$.

\end{enumerate}

\label{prop:prediction_func}
\end{proposition}
The proof is given in Appendix
\ref{appendix:proof_prediction_func}.

For classification, the order-preservation property ensures that the
highest-scoring class according to $\s$ and $\yHatOmega(\s)$ agree 
with each other:
\begin{equation}
\argmax_{i \in [d]} \ss_i = \argmax_{i \in [d]} \left(\yHatOmega(\s)\right)_i.
\end{equation}
Temperature scaling is useful to control how close we are to unregularized
prediction functions. Clearly, $\widehat{\y}_{t\Omega}(\s) \to \widehat{\y}(\s)$
as $t \to 0$, where $\widehat{\y}(\s)$ is defined in \eqref{eq:map_oracle}.

\section{Fenchel-Young losses}
\label{sec:fy_losses}

In the previous section, we introduced regularized prediction functions
over arbitrary domains, as
a generalization of classical (unregularized) decision functions.
In this section, we introduce Fenchel-Young losses for learning
models whose output layer is a regularized prediction function.
We first give their definitions and state their properties
(\S\ref{sec:fy_def_prop}). We then discuss
their relationship with Bregman divergences (\S\ref{sec:relation_Bregman}) and
their Bayes risk (\S\ref{sec:Bregman_info}). Finally, we
show how to construct a cost-sensitive loss from any Fenchel-Young loss
(\S\ref{sec:cost_sensitive}).

\subsection{Definition and properties}
\label{sec:fy_def_prop}

Given a regularized prediction function $\yHatOmega$, we define its
associated loss as follows.
\vspace{0.5em}
\begin{definition}{\label{def:FY_loss}Fenchel-Young loss generated by $\Omega$}

Let $\Omega \colon \RR^d \to \RR \cup \{\infty\}$ be a regularization
function such that the maximum in
\eqref{eq:prediction} is achieved for all $\s \in \RR^d$.
Let $\y \in \cY
\subseteq \dom(\Omega)$ be a ground-truth label and $\s \in \dom(\Omega^*)
= \RR^d$ be a vector of prediction scores. 

The {\bf Fenchel-Young loss} $L_\Omega \colon \dom(\Omega^*)
\times \dom(\Omega) \to \RR_+$ generated by $\Omega$ is
\begin{equation}
L_{\Omega}(\s; \y) 
\coloneqq \Omega^*(\s) + \Omega(\y) - \DP{\s}{\y}.
\label{eq:fy_losses}
\end{equation}
\end{definition}
It is easy to see that Fenchel-Young losses can be rewritten as
\begin{equation}
    L_{\Omega}(\s; \y) = f_{\s}(\y) - f_{\s}(\yHatOmega(\s)),
\end{equation}
where $f_{\s}(\bmu) \coloneqq \Omega(\bmu) - \DP{\s}{\bmu}$,
highlighting the relation with regularized prediction functions.
Therefore, as long as we can compute a regularized prediction function
$\yHatOmega(\s)$, we can automatically obtain an associated Fenchel-Young loss
$L_\Omega(\s; \y)$.
Conversely, we also have that $\yHatOmega$ outputs the prediction minimizing the
loss:
\begin{equation}
    \yHatOmega(\s) \in \argmin_{\bmu \in \dom(\Omega)} L_\Omega(\s; \bmu).
\end{equation}
Examples of existing losses that fall into the Fenchel-Young loss family are
given in Table~\ref{tab:fy_losses_examples}. Some of these examples will be
discussed in more details in the sequel of this paper.
Note that we will allow $\Omega$ in some cases to depend on the ground-truth
$\y$. Since we do not know $\y$ at test time, this requires us to use another
prediction function as a replacement for the regularized prediction function.
As we explain in \S\ref{sec:cost_sensitive}, this discrepancy is well motivated
and allows us to express popular cost-sensitive losses, such as the structured
hinge loss.

\begin{table}[t]
    \caption{{\bf Examples of regularized prediction functions and their corresponding
        Fenchel-Young losses.}  For multi-class classification, we assume
    $\cY=\{\e_i\}_{i=1}^d$ and the ground-truth is $\y=\e_k$, where $\e_i$
    denotes a standard basis (``one-hot'') vector.  For structured
    classification, we assume that elements of $\cY$ are $d$-dimensional binary
    vectors with $d \ll \sizeY$, and we denote by $\conv(\cY) = \{\EE_\p[Y]
    \colon \p \in \triangleY\}$ the corresponding marginal polytope
\citep{wainwright_2008}. We denote by $\HHs(\p) \coloneqq -\sum_i \pp_i\log
\pp_i$ the Shannon entropy of a distribution $\p \in \triangleY$.}
\begin{center}
\begin{small}
\begin{tabular}{@{\hskip 0pt}l@{\hskip 0pt}c@{\hskip 0pt}c@{\hskip 5pt}c@{\hskip 10pt}c@{\hskip 0pt}}
\toprule
Loss & $\dom(\Omega)$ & $\Omega(\bmu)$ & $\widehat{\y}_{\Omega}(\s)$ & $L_{\Omega}(\s; \y)$ \\
\midrule
Squared 
& $\RR^d$ & $\frac{1}{2}\|\bmu\|^2$ & $\s$ & $\frac{1}{2}\|\y-\s\|^2$ \smallskip
\\[0.5em]
Perceptron 
& $\triangleY$ & $0$ & $\argmax(\s)$ 
& $\max_i \ss_i -\ss_k$ 
\smallskip
\\
Logistic 
& $\triangleY$ & $-\HHs(\bmu)$ & $\softmax(\s)$ & 
$\log\sum_i\exp \ss_i -\ss_k$ 
\smallskip
\\
Hinge 
& $\triangleY$ & $\DP{\bmu}{\e_k - \ones}$ & 
$\argmax(\ones\!-\!\e_k\!+\!\s)$
& $\max_i ~ [[i \neq k]] + \ss_i -\ss_k$ 
\smallskip
\\
Sparsemax 
& $\triangleY$ & $\frac{1}{2}\|\bmu\|^2$ & $\sparsemax(\s)$ & 
$\frac{1}{2}\|\y-\s\|^2- \frac{1}{2}\|\widehat{\y}_{\Omega}(\s) - \s\|^2$ 
\smallskip
\\[0.5em]
Logistic (one-vs-all) 
& $[0,1]^\sizeY$ 
& $-\sum_i \HHs([\mu_i,1-\mu_i])$ & $\sigmoid(\s)$ & 
$\sum_i \log(1 + \exp(-(2 y_i-1) \ss_i))$
\smallskip
\\[0.5em]
Structured perceptron ~ 
& $\conv(\cY)$ & $0$ & $\map(\s)$ & 
$\max_{\y'} \DP{\s}{\y'} -\DP{\s}{\y}$ 
\smallskip
\\
Structured hinge 
& $\conv(\cY)$ & $-\DP{\bmu}{\bm{c}_\y}$ & $\map(\s + \bm{c}_\y)$ & 
$\max_{\y'} \DP{\bm{c}_\y}{\y'} + \DP{\s}{\y'} -\DP{\s}{\y}$ 
\smallskip
\\
CRF 
& $\conv(\cY)$ 
& ~ $\min\limits_{\p \in \triangleY \colon \EE_\p[Y]=\bmu}-\HHs(\p)$ 
& $\marginals(\s)$ 
& 
$\log\sum_{\y'}\exp \DP{\s}{\y'} - \DP{\s}{\y}$
\smallskip
\\
SparseMAP 
& $\conv(\cY)$ & $\frac{1}{2}\|\bmu\|^2$ & $\sparsemap(\s)$ &$\frac{1}{2}\|\y-\s\|^2 -\frac{1}{2}\|\widehat{\y}_{\Omega}(\s) - \s\|^2$ 
\smallskip
\\
\bottomrule
\end{tabular}
\end{small}
\end{center}
\label{tab:fy_losses_examples}
\end{table}

\paragraph{Properties.}

As the name indicates, this family of loss functions is grounded in the
Fenchel-Young inequality
\citep[Proposition 3.3.4]{borwein_2010}
\begin{equation}
\Omega^*(\s) + \Omega(\bmu) \ge \DP{\s}{\bmu} 
\quad \forall \s \in \dom(\Omega^*), \bmu \in \dom(\Omega).
\label{eq:fenchel_young_inequality}
\end{equation}
The inequality, together with well-known results regarding convex conjugates,
imply the following properties of Fenchel-Young losses.
\vspace{0.5em}
\begin{proposition}{Properties of Fenchel-Young losses}
\label{prop:fy_losses}

\begin{enumerate}[topsep=0pt,itemsep=3pt,parsep=3pt,leftmargin=15pt]

\item {\bf Non-negativity.} $L_{\Omega}(\s; \y) \ge 0$ for any $\s \in
    \dom(\Omega^*)=\RR^d$
and $\y \in \cY \subseteq \dom(\Omega)$. 


\item {\bf Zero loss.} If $\Omega$ is a lower semi-continuous proper convex
    function, then\\
    $\min_{\s} L_\Omega(\s; \y) = 0$, and $L_\Omega(\s; \y) = 0 \Leftrightarrow
    \y \in \partial \Omega^*(\s)$.  If $\Omega$ is strictly convex, 
then $L_\Omega(\s; \y) = 0 \Leftrightarrow \y
= \yHatOmega(\s) = \nabla \Omega^*(\s) = \argmin_{\s \in \RR^d} 
L_\Omega(\s; \y)$.

\item {\bf Convexity \& subgradients.} $L_{\Omega}$ is convex in $\s$ and the
    residual vectors are its subgradients: $\widehat{\y}_{\Omega}(\s) - \y \in
    \partial L_{\Omega}(\s; \y)$. 
    
\item {\bf Differentiability \& smoothness.}
    If $\Omega$ is strictly convex, then $L_{\Omega}$
    is differentiable and $\nabla L_{\Omega}(\s; \y) = \widehat{\y}_{\Omega}(\s) - \y$.
    If $\Omega$ is strongly convex, then $L_\Omega$ is smooth, i.e., $\nabla
    L_\Omega(\s; \y)$ is Lipschitz continuous.

\item {\bf Temperature scaling.} For any constant $t > 0$, 
    $L_{t \Omega}(\s; \y) = t L_{\Omega}(\nicefrac{\s}{t}; \y)$.

\item {\bf Constant invariance.} For any constant $c \in \RR$, 
$L_{\Omega + c}(\s; \y) = L_{\Omega}(\s; \y)$.

\end{enumerate}

\end{proposition}
Remarkably, the non-negativity, convexity and constant invariance properties
hold even if $\Omega$ is not convex.  The zero loss property 
follows from the fact that, if $\Omega$ is l.s.c.\ proper convex, then
\eqref{eq:fenchel_young_inequality} becomes an equality (i.e., the duality gap
is zero) if and only if $\s \in \partial \Omega(\bmu)$.  
It suggests that
the minimization of Fenchel-Young losses attempts to adjust the model to produce
predictions $\widehat{\y}_\Omega(\s)$ that are close to the target $\y$,
reducing the duality gap. This is illustrated with
$\Omega=\frac{1}{2}\|\bmu\|^2$ and $\dom(\Omega) = \RR^d$ (leading to the squared
loss) in Figure \ref{fig:duality_gap}.

\paragraph{Domain of $\mathbf{\Omega^*}$.}

Our assumption that the maximum in the regularized prediction
function \eqref{eq:prediction} is achieved for all $\s \in \RR^d$ implies that
$\dom(\Omega^*) = \RR^d$. 
This assumption is quite mild and does not require $\dom(\Omega)$ to be bounded.
Minimizing $L_\Omega(\s; \y)$ \wrt $\s$ is therefore an
\textbf{unconstrained} convex optimization problem.
This contrasts with proper loss functions, which are defined over the
probability simplex,
as discussed in \S\ref{sec:related_work}.

\begin{figure}[t]
\centering
\includegraphics[width=0.6\linewidth]{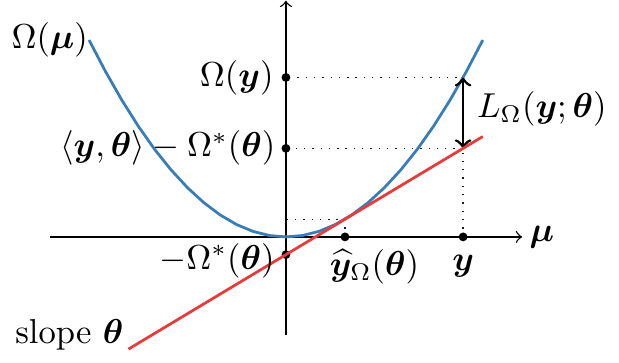}\\
\caption{{\bf Illustration of the Fenchel-Young loss} $L_\Omega(\s;
    \y)=\Omega^*(\s)+\Omega(\y)-\DP{\s}{\y}$, here with
    $\Omega(\bmu)=\frac{1}{2}\|\bmu\|^2$ and $\dom(\Omega)=\RR^d$. 
    Minimizing $L_\Omega(\s; \y)$ \wrt
    $\s$ can be seen as minimizing the duality gap, the difference between
    $\Omega(\bmu)$ and the tangent $\bmu \mapsto \DP{\s}{\bmu} - \Omega^*(\s)$,
at $\bmu=\y$ (the ground truth). The regularized prediction $\yHatOmega(\s)$ is
the value of $\bmu$ at which the tangent touches $\Omega(\bmu)$.
When $\Omega$ is of Legendre type, $L_\Omega(\s; \y)$ is equal to the Bregman
divergence generated by $\Omega$ between $\y$ and $\yHatOmega(\s)$ (cf.
\S\ref{sec:relation_Bregman}). However,
we do not require that assumption in this paper.
}
\label{fig:duality_gap}
\end{figure}

\subsection{Relation with Bregman divergences}
\label{sec:relation_Bregman}

Fenchel-Young losses seamlessly work when $\cY = \dom(\Omega)$ instead of $\cY
\subset \dom(\Omega)$. 
For example, in the case of the \logls loss, where
$-\Omega$ is the Shannon entropy restricted to $\triangle^d$, allowing $\y \in
\triangle^d$ instead of  $\y \in \{\e_i\}_{i=1}^d$
yields the cross-entropy loss, 
$L_{\Omega}(\s; \y) = \KL(\y \| \softmax(\s))$, where $\KL$ denotes the
(generalized) 
Kullback-Leibler divergence
\begin{equation}
\KL(\y \| \bmu) \coloneqq \sum_i y_i \log \frac{y_i}{\mu_i} - \sum_i y_i 
+ \sum_i \mu_i.
\end{equation}
This can be useful in a multi-label setting with supervision in the form of
\emph{label proportions}.  

From this example, it is tempting to conjecture that a similar result holds for
more general Bregman divergences \citep{bregman_1967}. 
Recall that the Bregman divergence $B_\Omega \colon
\dom(\Omega) \times \relint(\dom(\Omega)) \to \RR_+$ generated by a strictly
convex and differentiable $\Omega$ is
\begin{equation}
B_\Omega(\y \| \bmu) 
\coloneqq \Omega(\y) - \Omega(\bmu) - 
\DP{\nabla \Omega(\bmu)}{\y - \bmu}.
\label{eq:Bregman_div}
\end{equation}
In other words, this is the difference at $\y$ between $\Omega$ and its
linearization around $\bmu$. It turns out that $L_\Omega(\s; \y)$ is not in
general equal to $B_\Omega(\y \| \yHatOmega(\s))$. In fact the latter is 
\textbf{not necessarily convex} in $\s$ while the former always is.
However, there is a duality relationship
between Fenchel-Young losses and Bregman divergences, as we now discuss.

\paragraph{A ``mixed-space'' Bregman divergence.}

Letting $\s = \nabla \Omega(\bmu)$ (i.e., $(\s, \bmu)$ is a dual pair), we have 
$\Omega^*(\s) = \DP{\s}{\bmu} - \Omega(\bmu)$. Substituting in 
\eqref{eq:Bregman_div}, we get 
$B_\Omega(\y \| \bmu) = L_{\Omega}(\s; \y).$ 
In other words, 
Fenchel-Young losses can be
viewed as a ``mixed-form Bregman divergence'' 
\cite[Theorem 1.1]{amari_2016}
where the
argument $\bmu$ in \eqref{eq:Bregman_div} is \textbf{replaced by its dual point} $\s$.
This difference is best seen by comparing
the function signatures, $L_\Omega \colon \dom(\Omega^*) \times \dom(\Omega) \to
\RR_+$ vs. $B_\Omega \colon \dom(\Omega) \times \relint(\dom(\Omega)) \to
\RR_+$. An important consequence is that 
Fenchel-Young losses do not impose any restriction on their left argument $\s$:
Our assumption that the maximum in the prediction
function \eqref{eq:prediction} is achieved for all $\s \in \RR^d$ implies
$\dom(\Omega^*) = \RR^d$. 
In contrast, a Bregman divergence would typically need to be composed with a
mapping from $\RR^d$ to $\dom(\Omega)$, such as $\yHatOmega$, resulting in a
possibly non-convex function.

\paragraph{Case of Legendre-type functions.}

We can make the relationship with Bregman divergences further precise 
when $\Omega = \Psi + I_\cC$, where $\Psi$ is restricted to the class of
so-called Legendre-type functions \citep{Rockafellar1970,wainwright_2008}.  We
first recall the definition of this class of functions and then state our results.
\vspace{0.5em}
\begin{definition}{\label{def:essentially_smooth_Legendre_type}Essentially smooth and Legendre type functions}

A function $\Psi$ is essentially smooth if
\begin{itemize}
    \item $\dom(\Psi)$ is non-empty,
    \item $\Psi$ is differentiable throughout $\interior(\dom(\Psi))$,
    \item and $\lim_{i \to \infty} \nabla \Psi(\bmu^i) = +\infty$ for any
        sequence $\{\bmu^i\}$ contained in $\dom(\Psi)$, and converging to a
        boundary point of $\dom(\Psi)$.
\end{itemize}
A function $\Psi$ is of Legendre type if
\begin{itemize}
    \item it is strictly convex on $\interior(\dom(\Psi))$
    \item and essentially smooth.
\end{itemize}

\end{definition}
For instance, $\Psi(\bmu) = \frac{1}{2} \|\bmu\|^2$ is Legendre-type with
$\dom(\Psi)=\RR^d$, and $\Psi(\bmu) = \sum_i \mu_i \log \mu_i$ is
Legendre-type with $\dom(\Psi)=\RR^d_+$.
However, $\Omega(\bmu) = \frac{1}{2} \|\bmu\|^2 + I_{\RR_+^d}(\bmu)$ is not
Legendre-type, since the gradient of $\Omega$ does not explode everywhere on the
boundary of $\RR_+^d$.
The Legendre-type assumption crucially implies that
\begin{equation}
\nabla \Psi(\nabla
\Psi^*(\s)) = \s
\quad \text{for all} \quad \s \in \dom(\Psi^*).
\end{equation}
We can use this fact to derive the following results, proved in 
Appendix \ref{appendix:proof_Bregman_div}.
\vspace{0.5em}
\begin{proposition}{Relation with Bregman divergences}
\label{prop:Bregman_div}

Let $\Psi$ be of Legendre type with $\dom(\Psi^*)=\RR^d$ and 
let $\cC \subseteq \dom(\Psi)$ be a convex set. \\
Let $\Omega$ be the restriction of $\Psi$ to $\cC \subseteq
\dom(\Psi)$, i.e.,
$\Omega \coloneqq \Psi + I_\cC$.

\begin{enumerate}[topsep=0pt,itemsep=3pt,parsep=3pt,leftmargin=15pt]

\item {\bf Bregman projection.} The prediction function regularized by
    $\Omega$, $\yHatOmega(\s)$, reduces to the Bregman projection of
    $\yHatPsi(\s)$ onto $\cC$:
    \begin{equation}
\yHatOmega(\s)
= \argmax_{\bmu \in \cC} \DP{\s}{\bmu} - \Psi(\bmu)
= \argmin_{\bmu \in \cC} 
B_\Psi(\bmu \| \yHatPsi(\s)).
    \end{equation}

\item {\bf Difference of divergences.} For all $\s \in \RR^d$ and $\y \in \cC$:
    \begin{equation}
L_{\Omega}(\s; \y) = 
B_\Psi(\y \| \yHatPsi(\s)) -
B_\Psi(\yHatOmega(\s) \| \yHatPsi(\s)).
    \end{equation}

\item {\bf Bound.} For all $\s \in \RR^d$ and $\y \in \cC$:
\begin{equation}
    0 \le \underbrace{B_\Psi(\y \| \yHatOmega(\s))}_{\text{possibly
    non-convex in } \s} \le \underbrace{L_{\Omega}(\s; \y)}_{\text{convex in } \s}
\end{equation}
with equality when the loss is minimized
\begin{equation}
\yHatOmega(\s) = \y 
\Leftrightarrow
L_{\Omega}(\s; \y) = 0
\Leftrightarrow
B_\Psi(\y \| \yHatOmega(\s)) = 0.
\end{equation}

\item {\bf Composite form.}
When $\cC = \dom(\Psi)$, i.e., $\Omega = \Psi$, we have equality for all
$\s \in \RR^d$
\begin{equation}
    L_\Omega(\s; \y) = B_\Omega(\y \| \yHatOmega(\s)).
\end{equation}

\end{enumerate}

\end{proposition}
We illustrate these properties using $\Psi = \frac{1}{2} \|\cdot\|^2$ as a
running example. From the first property, since
$\yHatPsi(\s) = \s$ 
and $B_\Psi(\y \| \bmu) = \frac{1}{2} \|\y - \bmu\|^2$, we get
\begin{equation}
\yHatOmega(\s) 
= \argmin_{\bmu \in \cC} B_\Psi(\bmu \| \yHatPsi(\s)) 
= \argmin_{\bmu \in \cC} B_\Psi(\bmu \| \s) 
= \argmin_{\bmu \in \cC} \|\bmu - \s\|^2,
\end{equation}
recovering the Euclidean projection onto $\cC$.
The reduction of regularized prediction functions to Bregman projections
(when $\Psi$ is of Legendre type) is
useful because there exist efficient algorithms for computing the Bregman
projection onto various convex sets
\citep{yasutake_2011,online_submodular,bregmanproj,projection_permutahedron}.
Therefore, we can use these algorithms to compute $\yHatOmega(\s)$ provided
that $\yHatPsi(\s)$ is available.

From the second property, we obtain for all $\s \in \RR^d$ and $\y \in \cC$
\begin{equation}
L_{\Omega}(\s; \y) = \frac{1}{2}\|\y-\s\|^2-
\frac{1}{2}\|\widehat{\y}_{\Omega}(\s) - \s\|^2.
\end{equation}
This recovers the expression of the sparsemax loss given in
Table \ref{tab:fy_losses_examples} with $\cC = \triangle^d$.

From the third claim, we obtain for all $\s \in \RR^d$ and $\y \in \cC$
\begin{equation}
    \frac{1}{2} \|\y - \yHatOmega(\s)\|^2 \le L_{\Omega}(\s; \y).
\end{equation}
This shows that $L_{\Omega}(\s; \y)$ provides a
convex upper-bound for the possibly non-convex composite function $B_\Psi(\y
\| \yHatOmega(\s))$. In particular, when $\cC = \triangle^d$,
we get $\frac{1}{2} \|\y - \sparsemax(\s)\|^2 \le L_{\Omega}(\s;
\y)$. This suggests that the sparsemax loss is useful for sparse label
proportion estimation, as confirmed in our experiments
(\S\ref{sec:experiments}).

Finally, from the last property,
if $\Omega = \Psi = \frac{1}{2} \|\cdot\|^2$,
we obtain $L_\Omega(\s; \y) = \frac{1}{2} \|\y
- \s\|^2$, which is indeed the squared loss given in Table
\ref{tab:fy_losses_examples}.

\subsection{Expected loss, Bayes risk and Bregman information}
\label{sec:Bregman_info}

In this section, we discuss the relation between the pointwise Bayes risk
(minimal achievable loss) of a Fenchel-Young loss and Bregman information
\citep{bregman_clustering}.

\paragraph{Expected loss.}

Let $Y$ be a random variable taking values in $\cY$ following the distribution
$\p \in \triangleY$. The expected loss (a.k.a. expected risk) is then
\begin{align}
\EE_\p[L_\Omega(\s; Y)]
&= \sum_{\y \in \cY} p(\y) L_\Omega(\s; \y) \\
&= \sum_{\y \in \cY} p(\y) (\Omega^*(\s) + \Omega(\y) - \DP{\s}{\y}) \\
&= \EE_\p[\Omega(Y)] + \Omega^*(\s) - \DP{\s}{\EE_\p[Y]} \\
&= L_\Omega(\s; \EE_\p[Y]) + \II_\Omega(Y; \p).
\label{eq:expected_loss}
\end{align}
Here, we defined the \textbf{Bregman information} of $Y$ by
\begin{equation}
\II_\Omega(Y; \p) 
\coloneqq \min_{\bmu \in \dom(\Omega)} \EE_\p [B_\Omega(Y \| \bmu)]
= \EE_\p \left[ B_\Omega(Y \| \EE_\p[Y]) \right]
= \EE_\p[\Omega(Y)] - \Omega(\EE_\p[Y]).
\end{equation}
We refer the reader to \citet{bregman_clustering} for a detailed discussion as
to why the last two equalities hold.
The r.h.s.\ is exactly equal to the difference between the two sides of
Jensen's inequality $\EE[\Omega(Y)] \ge \Omega(\EE[Y])$ and is therefore
non-negative. For this reason, it is sometimes also called Jensen gap
\citep{reid_2011}. 

\paragraph{Bayes risk.}

From Proposition \ref{prop:fy_losses}, we know that $\min_{\s} L_\Omega(\s; \y)
= 0$ for all $\y \in \dom(\Omega)$. Therefore, the pointwise Bayes risk
coincides precisely with the Bregman information of $Y$,
\begin{equation}
\min_{\s \in \RR^d} \EE_\p[L_\Omega(\s; Y)]
= \min_{\s \in \RR^d} L_\Omega(\s; \EE_\p[Y]) + \II_\Omega(Y; \p)
= \II_\Omega(Y; \p),
\label{eq:Bayes_risk}
\end{equation}
provided that $\EE_\p[Y] \in \dom(\Omega)$.
A similar relation between Bayes risk and Bregman
information exists for proper losses \citep{reid_2011}. 
We can think of \eqref{eq:Bayes_risk} as a measure of the ``difficulty'' of the
task.  
Combining \eqref{eq:expected_loss} and \eqref{eq:Bayes_risk}, we obtain
\begin{equation}
\EE_\p[L_\Omega(\s; Y)] - \min_{\s \in \RR^d} \EE_\p[L_\Omega(\s; Y)]
= L_\Omega(\s; \EE_\p[Y]),
\end{equation}
the pointwise ``regret'' of $\s \in \RR^d$ \wrt $\p \in \triangleY$.
If $\cY = \{\e_i\}_{i=1}^d$, 
$L_\Omega(\s; \EE_\p[Y]) = L_\Omega(\s; \p)$.

\subsection{Cost-sensitive losses}
\label{sec:cost_sensitive}

Fenchel-Young losses also include the hinge loss of support vector machines.
Indeed, from any classification loss $L_\Omega$, we can construct a
cost-sensitive version of it as follows.  Define 
\begin{equation}
\Psi(\bmu; \y) \coloneqq \Omega(\bmu) - \DP{\bm{c}_{\y}}{\bmu},
\end{equation}
where $\bm{c}_{\y} \in \RR^{|\cY|}_+$ is a fixed
cost vector that may depend on the ground truth $\y$. For example, $\bm{c}_{\y} =
\mathbf{1} - \y$ corresponds to the 0/1 cost and can be used to impose a margin.  
Then, $L_\Psi$ is a cost-sensitive version of $L_\Omega$, which can be written
as
\begin{equation}
L_{\Psi(\cdot;\y)}(\s; \y) =
L_\Omega(\s + \bm{c}_\y; \y) = \Omega^*(\s + \bm{c}_\y) + \Omega(\y)
- \DP{\s + \bm{c}_{\y}}{\y}.  
\end{equation}
This construction recovers the multi-class hinge loss (\cite{multiclass_svm};
$\Omega=0$), the softmax-margin loss (\cite{gimpel_2010}; $\Omega=-\HHs$), and
the cost-augmented sparsemax (\citet[Eq. (13)]{accelerated_sdca},
\citet{sparsemap}; $\Omega=\frac{1}{2}\|\cdot\|^2$). 
It is easy to see that the associated regularized prediction function is
\begin{equation}
    \widehat{\y}_{\Psi}(\s) =  \yHatOmega(\s + \bm{c}_\y).
\end{equation}
For the 0/1 cost $\bm{c}_{\y} = \mathbf{1} - \y$ and $\y = \e_k$,
we have
\begin{equation}
\argmax_{i \in [d]} ~ (\widehat{\y}_{\Psi}(\s))_i = k
\quad \Longrightarrow \quad
\argmax_{i \in [d]} ~ (\yHatOmega(\s))_i = k,
\end{equation}
justifying the use of $\yHatOmega$ at prediction time.

\section{Probabilistic prediction with Fenchel-Young losses}
\label{sec:proba_clf}

In the previous section, we presented Fenchel-Young losses in a broad setting.
We now restrict to classification over the probability simplex.
More precisely, we restrict to the case $\cY = \{\bm{e}_i\}_{i=1}^d$,
(i.e., unstructured multi-class classification), and assume
that $\dom(\Omega) \subseteq \conv(\cY) = \triangle^d$.
In this case, the regularized prediction function \eqref{eq:prediction} becomes
\begin{equation}
\yHatOmega(\s) 
\in \argmax_{\p \in \triangle^d} {\DP{\s}{\p}} - \Omega(\p),
\label{eq:prediction_function_simplex}
\end{equation}
where $\s \in \RR^d$ is a vector of (possibly negative) prediction scores
produced by a model $\bm{f}_{W}(\x)$
and $\p \in \triangle^d$ is a discrete probability distribution.
It is a generalized exponential family distribution
\citep{grunwald_2004,frongillo_2014} with natural parameter $\s \in \RR^d$ and
regularization $\Omega$.
Of particular interest is the case where $\yHatOmega(\s)$ is
{\bf sparse}, meaning that there are scores $\s$ for which the
resulting $\yHatOmega(\s)$ assigns zero probability to some classes. As seen in
\S\ref{sec:reg_pred_examples}, this happens for example with the sparsemax
transformation, but not with the softmax. 
Later, in \S\ref{sec:margin}, we will
establish conditions for the regularized prediction function to be sparse
and will connect it to the notion of separation margin.

We first discuss generalized entropies and the properties of the Fenchel-Young
losses they induce (\S\ref{sec:generalized_ent}). 
We then discuss their expected loss, Bayes risk and Fisher consistency
(\S\ref{sec:proba_expected_loss}).
We then give examples of
generalized entropies and corresponding loss functions, several of them new to
our knowledge (\S\ref{sec:examples_entropy}).
Finally, we discuss the binary classification setting, recovering several
examples of commonly-used loss functions (\S\ref{sec:simplex_binary_case}).

\subsection{Fenchel-Young loss generated by a generalized entropy}
\label{sec:generalized_ent}

\paragraph{Generalized entropies.}

A natural choice of regularization function $\Omega$ over the probability
simplex is $\Omega=-\HH$, where $\HH$ is a generalized entropy
\citep{grunwald_2004}, also called uncertainty function by \citet{degroot_1962}:
a concave function over $\triangle^d$, used to measure the ``uncertainty'' in a
distribution $\p \in \triangle^d$.

\textbf{Assumptions:}
We will make the following assumptions about $\HH$.

\begin{enumerate}[topsep=0pt,itemsep=3pt,parsep=3pt,leftmargin=25pt]
    \item[\textbf{A.1.}] 
    Zero entropy: $\HH(\p) = 0$ if $\p$ is a delta distribution,
    i.e., $\p \in \{\e_i\}_{i=1}^d$.
    \item[\textbf{A.2.}] 
        Strict concavity: $\HH\big((1-\alpha)\p + \alpha \p'\big) > {(1-\alpha)} \HH(\p) +
        \alpha \HH(\p')$, for $\p \neq \p'$, $\alpha \in (0,1)$.
    \item[\textbf{A.3.}] 
        Symmetry: $\HH(\p) =
        \HH(\bm{P}\p)$ for any $\bm{P} \in \mathcal{P}$.

\end{enumerate}

Assumptions A.2 and A.3 imply that $\HH$ is Schur-concave \citep{Bauschke2017},
a common requirement in generalized entropies. This in turn implies assumption
A.1, up to a constant (that constant can easily be subtracted so as to satisfy
assumption A.1).
As suggested by the next result, 
proved in \S\ref{appendix:proof_generalized_entropy},
together, these assumptions
imply that $\HH$ can be used as a sensible
uncertainty measure.
\vspace{0.4em}
\begin{proposition}
If $\HH$ satisfies assumptions A.1--A.3, then it is non-negative and uniquely
maximized by the uniform distribution $\p = \mathbf{1}/d$.
\label{prop:generalized_entropy}
\end{proposition}
That is, assumptions A.1-A.3 ensure that the uniform distribution is the maximum
entropy distribution.
A particular case of generalized entropies satisfying assumptions A.1--A.3 are
uniformly separable functions of the form $\HH(\p) = \sum_{j=1}^d h(p_j)$,
where $h:[0,1]\rightarrow\RR_+$ is a non-negative strictly concave function such
that $h(0)=h(1)=0$. However, our framework is not restricted to this form.

\paragraph{Induced Fenchel-Young loss.}

If the ground truth is $\y = \e_k$ and assumption
A.1 holds, the Fenchel-Young loss definition \eqref{eq:fy_losses} becomes
\begin{equation}
L_{-\HH}(\s; \e_k) 
= (-\HH)^*(\s) - \ss_k.
\label{eq:fy_loss_multiclass2}
\end{equation}
This form was also recently proposed by \citet[Proposition 3]{duchi_2016}.
By using the fact that $\Omega^*(\s + c \ones) = \Omega^*(\s) + c$ for all $c
\in \RR$ if $\dom(\Omega) \subseteq \triangle^d$,
we can further rewrite it as
\begin{equation}
L_{-\HH}(\s; \e_k) 
= (-\HH)^*(\s - \ss_k \ones).
\label{eq:fy_loss_multiclass}
\end{equation}
This expression shows that Fenchel-Young losses over $\triangle^d$ can be
written solely in terms of the generalized ``cumulant function'' $(-\HH)^*$. 
Indeed, when $\HH$ is Shannon's entropy, we recover the cumulant (a.k.a.
log-partition) function
$(-\HHs)^*(\s) = \log \sum_{i=1}^d \exp(\theta_i)$.
When $\HH$ is strongly concave over $\triangle^d$,
we can also see $(-\HH)^*$ as a smoothed max operator
\citep{nesterov_smooth,Niculae2017,differentiable_dp} and hence $L_{-\HH}(\s;
\e_k)$ can be seen as a smoothed upper-bound of the perceptron loss
$(\s; \e_k) \mapsto \max_{i \in [d]} \theta_i - \theta_k$.

It is well-known that minimizing the logistic loss $(\s; \e_k) \mapsto
\log\sum_i\exp \ss_i -\ss_k$ is equivalent to minimizing the $\KL$ divergence
between $\e_k$ and $\softmax(\s)$, which in turn is equivalent to maximizing the
likelihood of the ground-truth label, $\softmax(\s)_k$.  Importantly, this
equivalence does not carry over for generalized entropies $\HH$: minimizing
$L_{-\HH}(\s; \e_k)$ is not in general equivalent to minimizing $B_{-\HH}(\e_k,
\widehat{\y}_{-\HH}(\s))$ or maximizing $\widehat{\y}_{-\HH}(\s)_k$. In fact,
maximizing the likelihood is generally a non-concave problem.  Fenchel-Young
losses can be seen as a principled way to construct a convex loss regardless of
$\HH$.

\subsection{Expected loss, Bayes risk and Fisher consistency}
\label{sec:proba_expected_loss}

Let $Y$ be a random variable taking values in $\cY = \{\e_i\}_{i=1}^d$.
If $\HH$ satisfies assumption A.1,
we can use \eqref{eq:fy_loss_multiclass2} to obtain simpler expressions of
the expected loss and pointwise Bayes risk
than the ones derived in \S\ref{sec:Bregman_info}.

\paragraph{Expected loss.}

Indeed, the expected loss (risk) for all $\p \in \triangle^d$ simplifies to:
\begin{equation}
\EE_\p[L_{-\HH}(\s; Y)] 
= \sum_{i=1}^d p_i L_{-\HH}(\s; \e_i)
= \sum_{i=1}^d p_i ((-\HH)^*(\s) - \ss_j)
= (-\HH)^*(\s) - \DP{\p}{\s}.
\label{eq:expected_loss_simplex}
\end{equation}

\paragraph{Bayes risk and entropy.}

The pointwise (conditional) Bayes risk thus becomes
\begin{equation}
\min_{\s \in \RR^d} \EE_\p[L_{-\HH}(\s; Y)] 
= \min_{\s \in \RR^d} (-\HH)^*(\s) - \DP{\p}{\s}
= \HH(\p).
\label{eq:Bayes_risk_simplex}
\end{equation}
Therefore, the pointwise Bayes risk is equal to the generalized entropy $\HH$
generating $L_{-\HH}$, evaluated at $\p$. This is consistent with
\eqref{eq:Bayes_risk}, which states that the pointwise Bayes risk is equal to
the Bregman information of $Y$ under $\p$, because
\begin{equation}
    \II_{-\HH}(Y; \p) = -\II_{\HH}(Y; \p) = \HH(\EE_\p[Y]) - \EE_\p[\HH(Y)] =
    \HH(\p),
\end{equation}
where we used $\EE_\p[Y] = \p$ and assumption A.1.

A similar relation
between generalized entropies and pointwise (conditional) Bayes risk
is well-known in the proper loss (scoring rule) literature
\citep{grunwald_2004,gneiting_2007,reid_composite_binary,vernet_2016}.
The main difference is that the minimization above is over $\RR^d$, while it is
over $\triangle^d$ in that literature (\S\ref{sec:proper_losses}). 
As noted by \citet[\S 4.6]{reid_2011}, Bregman information can also be connected
the notion of statistical information developed by \citet{degroot_1962}, the
reduction between prior and posterior uncertainty $\HH$, of which mutual
information is a special case.

\paragraph{Fisher consistency.}

From \eqref{eq:expected_loss_simplex}, $\EE_\p[L_{-\HH}(\s; Y)] = L_{-\HH}(\s;
\p) + \Omega(\p)$.  Combined with Proposition \ref{prop:fy_losses}, we have that
the pointwise Bayes risk \eqref{eq:Bayes_risk_simplex} is achieved if and only
if $\widehat{\y}_{-\HH}(\s) = \p$. Such losses are Fisher consistent estimators
of probabilities \citep{vernet_2016}.

\subsection{Examples}
\label{sec:examples_entropy}

We now give examples of generalized entropies over the simplex $\triangle^d$ (we
omit the indicator function $I_{\triangle^d}$ from the definitions since there
is no ambiguity).  We illustrate them together with the regularized prediction
function and loss they produce in Figure~\ref{fig:all_entropies}. Several of the
resulting loss functions are new to our knowledge.

\paragraph{Shannon entropy \citep{Shannon1949}.} 

This is the foundation of information theory, defined as 
\begin{equation}
\HHs(\p) \coloneqq -\sum_{j=1}^d p_j \log p_j. 
\end{equation}
As seen in 
Table~\ref{tab:fy_losses_examples}, the resulting Fenchel-Young loss $L_{-\HHs}$
corresponds to the \logls loss. The associated distribution is the classical
softmax, Eq.\ \eqref{eq:softmax}.

\paragraph{Tsallis $\alpha$-entropies \citep{Tsallis1988}.} 

These entropies are defined as 
\begin{equation}
\HHt_{\alpha}(\p) \coloneqq k(\alpha-1)^{-1} \left(1 - \sum_{j=1}^d
p_j^{\alpha}\right), 
\end{equation}
where $\alpha \ge 0$ and $k$ is an arbitrary positive constant. 
They arise as a generalization of the Shannon-Khinchin axioms to
non-extensive systems \citep{Suyari2004} and have numerous scientific
applications \citep{GellMannTsallis2004,Martins2009JMLR}. 
For convenience, we set $k = \alpha^{-1}$ for the rest of this paper. 
Tsallis entropies satisfy assumptions A.1--A.3 and can also be written in separable form:
\begin{equation}
\HHt_{\alpha}(\p) \coloneqq \sum_{j=1}^d h_{\alpha}(p_j)
\quad \text{with} \quad
h_{\alpha}(t) \coloneqq \frac{t-t^\alpha}{\alpha(\alpha - 1)}.
\label{eq:tsallis_separable}
\end{equation}
The limit case $\alpha \rightarrow 1$ corresponds to the Shannon entropy.
When $\alpha=2$, we recover the Gini index \citep{gini_index}, a 
popular ``impurity measure'' for decision trees:
\begin{equation}
\HHt_2(\p) = \frac{1}{2}\sum_{j=1}^d p_j (1-p_j) =
\frac{1}{2}(1-\|\p\|_2^2)
\quad \forall \p \in \triangle^d.
\label{eq:gini_entropy}
\end{equation}
Using the constant invariance property in Proposition \ref{prop:fy_losses},
it can be checked that $L_{-\HHt_2}$ recovers
the sparsemax loss \citep{sparsemax} (cf.\ Table~\ref{tab:fy_losses_examples}). 

Another interesting case is $\alpha\rightarrow +\infty$, which gives
$\HHt_\infty(\p) = 0$, hence $L_{-\HHt_\infty}$ is the perceptron loss in
Table~\ref{tab:fy_losses_examples}. The resulting ``$\argmax$'' distribution puts
all probability mass on the top-scoring classes.
In summary, the prediction functions for
$\alpha = 1, 2, \infty$ are respectively $\softmax$, $\sparsemax$, and
$\argmax$.  
Tsallis entropies can therefore be seen as a \textbf{continuous
parametric family} subsuming these important cases.  

\paragraph{Norm entropies.}

An interesting class of non-separable entropies are entropies generated by a
$q$-norm, defined as
\begin{equation}
\HHn_q(\p) \coloneqq 1 - \|\p\|_q.
\label{eq:norm_entropy}
\end{equation}
We call them {\bf norm entropies}. 
By the Minkowski inequality, $q$-norms with $q>1$ are strictly convex on the
simplex, so $\HHn_q$ satisfies assumptions A.1--A.3 for $q > 1$.  The resulting
norm entropies differ from Tsallis entropies in that the norm is not raised to
the power of $q$: a subtle but important difference.  The limit case $q
\rightarrow \infty$ is particularly interesting: in this case, we obtain
$\HHn_{\infty} = 1-\|\cdot\|_{\infty}$, recovering the Berger-Parker dominance
index  \citep{Berger1970}, widely used in ecology to measure species diversity.  
We surprisingly encounter $\HHn_\infty$ again in Section~\ref{sec:margin}, 
as a limit case for the existence of separation margins.

\paragraph{Squared norm entropies.}

Inspired by \citet{Niculae2017}, as a simple extension of the Gini index
\eqref{eq:gini_entropy}, we consider the generalized entropy based on
squared $q$-norms:
\begin{equation}
\HHsq_{q}(\p) 
\coloneqq \frac{1}{2} (1 - \|\p\|^2_q) 
= \frac{1}{2} - \frac{1}{2} \left(\sum_{j=1}^d p_j^q \right)^{\frac{2}{q}}.
\end{equation}
The constant term $\frac{1}{2}$, omitted by \citet{Niculae2017},
ensures satisfaction of A.1. 
For $q \in (1,2]$, it is known that the squared $q$-norm is strongly convex \wrt
$\|\cdot\|_{q}$ \citep{ball_1994}, implying that $(-\HHsq_{q})^*$, and therefore
$L_{-\HHsq_{q}}$, is smooth.
Although $\widehat{\y}_{-\HHsq_q}(\s)$ cannot to our knowledge be solved in
closed form for $q \in (1,2)$, efficient iterative algorithms such as projected
gradient are available.

\paragraph{R\'enyi {\boldmath $\beta$}-entropies.} 

R\'enyi entropies \citep{Renyi1960} are defined for any $\beta\ge 0$ as:
\begin{equation}\label{eq:renyi_entropies}
\HHr_{\beta}(\p) \coloneqq \frac{1}{1-\beta} \log \sum_{j=1}^d
p_j^\beta.
\end{equation}
Unlike Shannon and Tsallis entropies, R\'enyi entropies are not separable, with
the exception of $\beta \rightarrow 1$, which also recovers Shannon entropy as a
limit case. The case $\beta \rightarrow +\infty$ gives $\HHr_{\beta}(\p) = -\log \|\p\|_{\infty}$. 
For $\beta \in [0, 1]$, R\'enyi entropies satisfy assumptions
A.1--A.3; for $\beta > 1$, R\'enyi entropies fail to be concave.  They are
however pseudo-concave \citep{Mangasarian1965}, meaning that, for all $\p,
\bm{q} \in \triangle^d$,  $\DP{\nabla \HHr_\beta(\p)}{\bm{q} - \p} \le 0$ implies
$\HHr_\beta(\bm{q}) \le \HHr_\beta(\p)$.  This implies, among other things, that
points $\p \in \triangle^d$ with zero gradient are maximizers of $\DP{\p}{\s} +
\HHr_{\beta}(\p)$, which allows us to compute the predictive distribution
$\widehat{\y}_{-\HHr_{\beta}}$ with gradient-based methods.  

\begin{figure}[p]
\centering
\includegraphics[width=\linewidth]{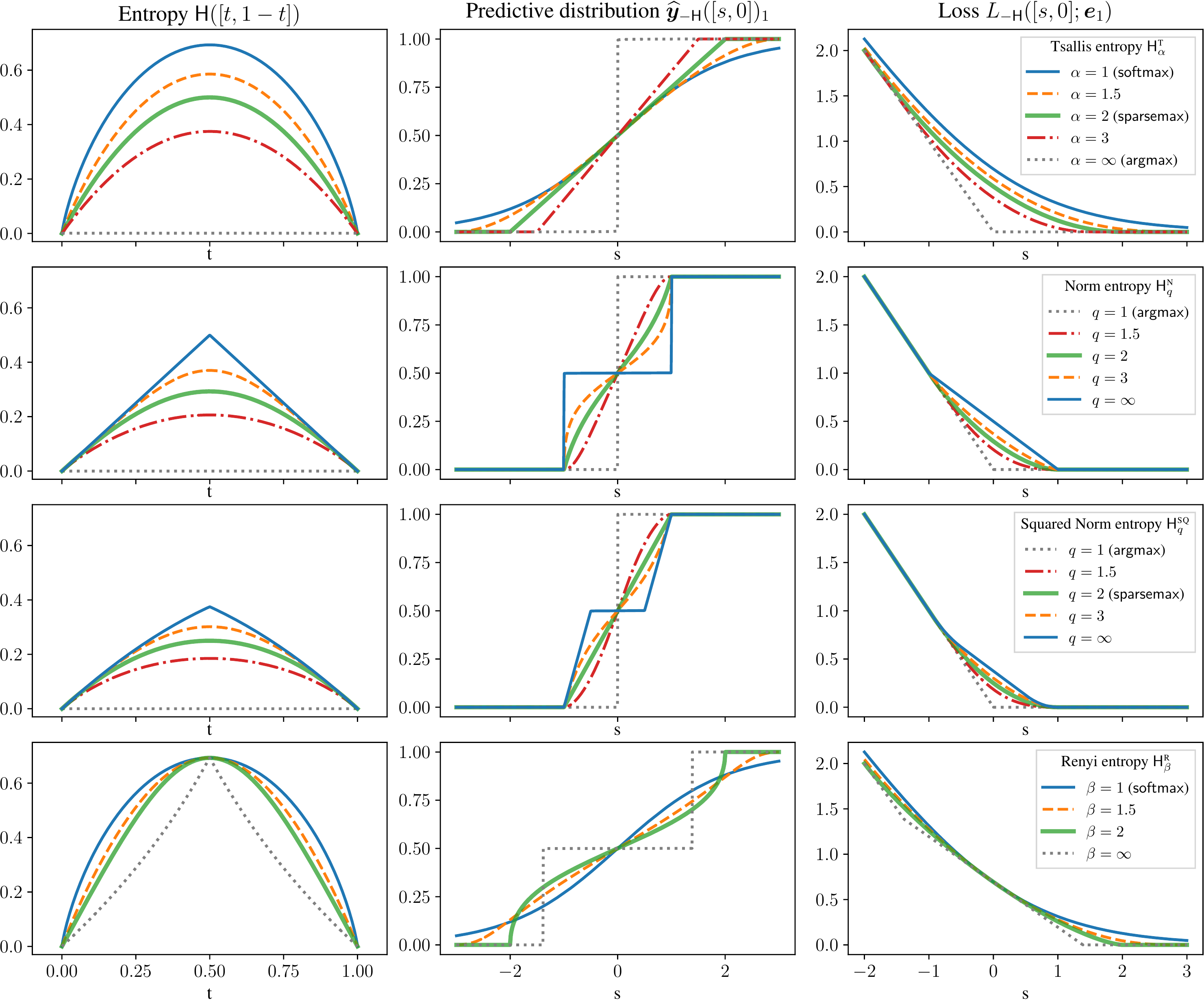}\\
\caption{Examples of \textbf{generalized entropies} (left) along with their
    \textbf{prediction distribution} (middle) and \textbf{Fenchel-Young losses} (right) for the binary case,
    where $\p = [t, 1-t] \in \triangle^2$ and $\s = [s, 0] \in \RR^2$.  Except
    for softmax, which never exactly reaches 0, all distributions shown on the center
can have \textbf{sparse support}.
}
\label{fig:all_entropies}
\end{figure}

\subsection{Binary classification case} 
\label{sec:simplex_binary_case}

When $d = 2$,
the convex conjugate expression simplifies and reduces to a
univariate maximization problem over $[0,1]$:
\begin{equation}
\Omega^*(\s) = \max_{\p \in \triangle^2} \DP{\s}{\p} - \Omega(\p)
= \max_{p \in [0,1]} \ss_1 p + \ss_2 (1-p) - \Omega([p,1-p])
= \phi^*(\ss_1 - \ss_2) + \ss_2,
\end{equation}
where we defined $\phi(p) \coloneqq \Omega([p,1-p])$.
Likewise, it is easy to verify that we also have $\Omega^*(\s) = 
\phi^*(\ss_2 - \ss_1) + \ss_1$.
Let us choose $\s = [s, -s] \in \RR^2$ for some $s \in \RR$.
Since the ground truth is $\y \in \{ \e_1, \e_2 \}$, 
where $\e_1$ and $\e_2$ represent positive and
negative classes, 
from \eqref{eq:fy_loss_multiclass}, we can write
\begin{equation}
L_\Omega(\s; \e_j) 
= \Omega^*(\s - \ss_j \ones) 
= \begin{cases}
    \Omega^*([0, -s]) = \phi^*(-s) & \mbox { if } j=1 \\
    \Omega^*([s, 0]) = \phi^*(s) & \mbox { if } j=2
\end{cases}.
\end{equation}
This can be written concisely as $\phi^*(-y s)$ where $y \in \{+1, -1\}$.
Losses that can be written in this form are sometimes called
margin losses \citep{reid_composite_binary}. 
These are losses that treat the
positive and negative classes symmetrically. %
We have thus made a connection between margin losses and 
the regularization function $\phi(p) \coloneqq \Omega([p,1-p])$.
Note, however, that this is different from the notion of margin 
we develop in \S\ref{sec:margin}.

\paragraph{Examples.} Choosing $\phi(p) = -\HHs([p,1-p])$ leads to
the binary logistic loss
\begin{equation}
    \phi^*(-ys) = \log(1 + \exp(-ys)).
\end{equation}
Choosing $\phi(p) = -\HHt_2([p,1-p])=p^2 - p$
leads to 
\begin{equation}
\phi^*(u) = 
\begin{cases}
    0 & \mbox { if } u \le -1 \\
     u & \mbox { if } u \ge 1 \\
    \frac{1}{4} (u+1)^2 & \mbox{ o.w. }
\end{cases}.
\end{equation}
The resulting loss, $\phi(-ys)$, is known as the modified Huber loss in the
literature \citep{zhang_2004}. Hence, the sparsemax loss is the multiclass
extension of the modified Huber loss, as as already noted in \citep{sparsemax}.

Finally, we note that the modified Huber loss is closely related to the smoothed
hinge loss \citep{accelerated_sdca}.
Indeed, choosing $\phi(p) = \frac{1}{2}p^2 - p$ (notice the $\frac{1}{2}$
factor), we obtain
\begin{equation}
\phi^*(u) = 
\begin{cases}
    0 & \mbox { if } u \le -1 \\
     u + \frac{1}{2} & \mbox { if } u \ge 0 \\
    \frac{1}{2} (1 + u)^2 & \mbox{ o.w. }
\end{cases}.
\end{equation}
It can be verified that $\phi^*(-ys)$ is indeed the smoothed hinge loss.

\section{Separation margin of Fenchel-Young losses}
\label{sec:margin}

In this section, we are going to see that the simple assumptions A.1--A.3 about
a generalized entropy $\HH$ are enough to obtain results about the {\bf separation margin} associated with $L_{-\HH}$. 
The notion of margin is well-known in machine learning, lying at the heart of support vector machines and leading to generalization error bounds \citep{Vapnik1998,Schoelkopf2002,guermeur_2007}. 
We provide a definition and will see that many other Fenchel-Young losses also have a ``margin,'' for suitable conditions on $\HH$.  
Then, we take a step further, and connect the existence of a margin with the {\bf sparsity} of the regularized prediction function, 
providing necessary and sufficient conditions for Fenchel-Young losses to have a margin.  Finally, we show how this margin can be computed analytically.
\vspace{0.5em}
\begin{definition}{\label{def:margin}Separation margin}

Let $L(\s; \e_k)$ be a loss function over $\RR^d \times \{\e_i\}_{i=1}^d$.
We say that $L$ has
the \emph{separation margin property} if there exists $m>0$ such that:
\begin{equation}
\ss_k \ge m + \max_{j \ne k} \ss_j \quad \Rightarrow \quad L(\s; \e_k) = 0.
\label{eq:margin_definition}
\end{equation}
The smallest possible $m$ that satisfies \eqref{eq:margin_definition} is called 
the \emph{margin} of $L$, denoted $\mathrm{margin}(L)$.
\end{definition}

\paragraph{Examples.}  
The most famous example of a loss with a separation margin is the  
{\bf multi-class hinge loss}, 
$L(\s; \e_k) = \max\{0, \max_{j \ne k} 1 + \ss_j - \ss_k\}$, 
which we saw in Table \ref{tab:fy_losses_examples} to be 
a Fenchel-Young loss: it is immediate from the definition that its margin is $1$. 
Less trivially, \citet[Prop.~3.5]{sparsemax} showed that the {\bf sparsemax loss} also 
has the separation margin property. 
On the negative side, the logistic loss does not have a margin, as it is
strictly positive.
Characterizing which Fenchel-Young losses have a margin is an important question
which we address next.

\paragraph{Conditions for existence of margin.} 
To accomplish our goal, we need to characterize the gradient mappings $\partial(-\HH)$ and $\nabla(-\HH)^{*}$ associated with generalized entropies
(note that $\partial(-\HH)$
is never single-valued:  
if $\s$ is in $\partial (-\HH)(\bm{p})$, then so is $\s+c\mathbf{1} $, for any
constant $c \in \mathbb{R}$).
Of particular importance is the subdifferential set 
$\partial(-\HH)(\e_k)$. 
The next proposition, whose proof we defer to 
\S\ref{appendix:proof_margin}, uses this set to provide a necessary and sufficient condition for the existence of a separation margin, along with a formula for computing it. 
\vspace{0.5em}
\begin{proposition}\label{prop:margin}
Let $\HH$ satisfy A.1--A.3. Then:
\begin{enumerate}
\item The loss $L_{-\HH}$ has a separation margin iff there is  a $m>0$ such that $m\e_k \in \partial (-\HH)(\e_k)$. 
\item If the above holds, then the margin of $L_{-\HH}$ is given by the smallest
    such $m$ or, equivalently,
\begin{equation}\label{eq:margin}
\mathrm{margin}(L_{-\HH}) = \sup_{\p \in \triangle^d} \frac{\HH(\p)}{1 - \|\p\|_{\infty}}. 
\end{equation}
\end{enumerate}
\end{proposition}
Reassuringly, the first part confirms that
the logistic loss does not have a margin,
since $\partial(-\HHs)(\e_k)=\varnothing$.
%
A second interesting fact is that the denominator of \eqref{eq:margin} is the generalized
entropy $\HHn_{\infty}(\p)$ introduced in \S\ref{sec:proba_clf}: the {\bf $\infty$-norm entropy}. 
As Figure~\ref{fig:all_entropies} suggests, this entropy provides an upper bound for convex losses with unit margin. 
This provides some intuition to the formula \eqref{eq:margin}, which seeks a
distribution $\p$ maximizing the {\bf entropy ratio} between $\HH(\p)$ and
$\HHn_{\infty}(\p)$.

\paragraph{Relationship between sparsity and margins.} 
The next result, proved in \S\ref{appendix:proof_full_simplex},
characterizes more precisely the image of $\nabla (-\HH)^*$. 
In doing so, it establishes a key result in this paper: 
{\bf a sufficient condition for the existence of a separation margin in $L_{-\HH}$ is the sparsity of the regularized prediction function $\widehat{\y}_{-\HH} \equiv \nabla (-\HH)^*$},   
i.e., its ability to reach the entire simplex, including the boundary points. If $\HH$ is uniformly separable, this is also a necessary condition.  
\vspace{0.5em}
\begin{proposition}{Relationship between
margin losses and 
sparse predictive probabilities}%
\label{prop:full_simplex}%

Let $\HH$ satisfy A.1--A.3 and be 
uniformly separable, i.e., $\HH(\p) = \sum_{i=1}^d h(p_i)$. Then the following statements are all equivalent:
\begin{enumerate}
\item $\partial (-\HH)(\p) \ne \varnothing$ 
for any $\p \in \triangle^d$;
\item The mapping $\nabla (-\HH)^*$ covers the full simplex, i.e., $\nabla
    (-\HH)^*(\RR^d) = \triangle^d$;
\item $L_{-\HH}$ has the separation margin property.
\end{enumerate}
For a general $\HH$ (not necessarily separable) satisfying A.1--A.3, 
we have (1) $\Leftrightarrow$ (2) $\Rightarrow$ (3). 
\end{proposition}

Let us reflect for a moment on the three conditions stated in  Proposition~\ref{prop:full_simplex}. 
The first two conditions involve the subdifferential and gradient of $-\HH$ and its conjugate; the third condition is the margin property of $L_{-\HH}$. To provide some intuition, consider the case where $\HH$ is separable with $\HH(\p) = \sum_i h(\pp_i)$ and $h$ is differentiable in $(0,1)$. Then, from the concavity of $h$, its derivative $h'$ is decreasing, hence the first condition  is met if $\lim_{t=0^+} h'(t) < \infty$ and $\lim_{t=1^-} h'(t) > -\infty$. 
This is the case with Tsallis entropies for $\alpha > 1$, but not Shannon
entropy, since
$h'(t) = -1-\log t$ explodes at $0$.
As stated in Definition \ref{def:essentially_smooth_Legendre_type},
functions whose gradient ``explodes'' in the boundary of their domain (hence
failing to meet the first condition in Proposition~\ref{prop:full_simplex}) are
called ``essentially smooth'' \citep{Rockafellar1970}.  
For those functions, $\nabla (-\HH)^*$ maps only to the relative interior of
$\triangle^d$, never attaining boundary points \citep{wainwright_2008}; this is expressed in the second condition.  
This prevents essentially smooth functions from generating a sparse $\y_{-\HH} \equiv \nabla (-\HH)^*$ or (if they are separable) a
loss $L_{-\HH}$ with a margin, as asserted by the third condition. 
Since Legendre-type functions (Definition
\ref{def:essentially_smooth_Legendre_type})
are strictly convex \textit{and} essentially smooth,
by Proposition \ref{prop:Bregman_div}, loss functions for which
the composite form $L_{-\HH}(\s; \y) = B_{-\HH}(\y \| \widehat{\y}_{-\HH}(\s))$ holds, which is the
case of the logistic loss but not of the sparsemax loss, do not enjoy a margin
and cannot induce a sparse probability distribution.
This is geometrically visible in Figure~\ref{fig:all_entropies}.  

\paragraph{Margin computation.} 

For Fenchel-Young losses that have the separation margin property, Proposition~\ref{prop:margin} provided a formula for determining the margin. 
While informative, formula \eqref{eq:margin} is not very practical, 
as it involves a generally non-convex optimization problem.
The next proposition, proved in
\S\ref{appendix:proof_margin_separable}, takes a step  further and
provides a remarkably simple closed-form expression for  generalized
entropies that are {\bf twice-differentiable}. 
To simplify notation, we denote by $\nabla_j \HH(\p) \equiv (\nabla \HH(\p))_j$
the $j^{\text{th}}$ component of $\nabla \HH(\p)$. 
\vspace{0.5em}
\begin{proposition}\label{prop:margin_separable}
Assume $\HH$ satisfies the conditions in Proposition~\ref{prop:full_simplex} 
and is twice-differ\-entiable on the simplex. Then, for arbitrary $j \ne k$:
\begin{equation}\label{eq:margin_general}
\mathrm{margin}(L_{-\HH}) = {\nabla_j{\HH}(\e_k)} - {\nabla_k{\HH}(\e_k)}.
\end{equation}
In particular, if $\HH$ is separable, i.e., $\HH(\p) = \sum_{i=1}^{|\cY|}
h(\pp_i)$, where $h:[0,1] \rightarrow \RR_{+}$ is concave, twice differentiable,
with $h(0) = h(1) = 0$, then
\begin{equation}\label{eq:margin_separable}
    \mathrm{margin}(L_{-\HH}) = h'(0) - h'(1) =-\hspace{-.2ex}\int_{0}^1\hspace{-.4ex}h''(t) dt.
\end{equation}\end{proposition}

The compact formula \eqref{eq:margin_general} provides a geometric characterization of separable entropies and their margins: 
\eqref{eq:margin_separable} tells us that only the slopes of $h$ at the two extremities of 
$[0,1]$ are relevant in determining the margin.

\paragraph{Example: case of Tsallis and norm entropies.} 
As seen in \S\ref{sec:proba_clf}, Tsallis entropies are separable with $h(t) = (t - t^\alpha)/(\alpha(\alpha-1))$. 
For $\alpha > 1$, $h'(t) = (1 - \alpha t^{\alpha -
1})/(\alpha(\alpha-1))$, hence $h'(0)=1/(\alpha(\alpha-1))$ and
$h'(1)=-1/\alpha$. Proposition~\ref{prop:margin_separable} then yields 
\begin{equation}
    \mathrm{margin}(L_{-\HHt_\alpha}) = h'(0) - h'(1) = (\alpha-1)^{-1}. 
\end{equation}
Norm entropies, while not separable, have gradient
$\nabla \HHn_q(\p) = -( \nicefrac{\p}{\|\p\|_q} )^{q-1}$, 
giving $\nabla \HHn_q(\e_k) = -\e_k$, so
\begin{equation}
\mathrm{margin}(\HHn_q) =
\nabla_j \HHn_q(\e_k) - \nabla_k \HHn_q(\e_k) = 1, 
\end{equation}
as confirmed visually in Figure~\ref{fig:all_entropies},
in the binary case. 

\section{Positive measure prediction with Fenchel-Young losses}
\label{sec:positive_measures}

In this section, we again restrict to classification and $\cY =
\{\e_i\}_{i=1}^d$ but now assume that $\dom(\Omega) \subseteq \cone(\cY) =
\RR^d_+$, where $\cone(\cY)$ is the conic hull of $\cY$.
In this case, the regularized prediction function \eqref{eq:prediction} becomes
\begin{equation}
\yHatOmega(\s) 
\in \argmax_{\m \in \RR^d_+} {\DP{\s}{\m}} - \Omega(\m),
\end{equation}
where $\s \in \RR^d$ is again a vector of prediction scores
and $\m \in \RR^d_+$ can be interpreted as a discrete positive measure
(unnormalized probability distribution).

We first demonstrate how Fenchel-Young losses over positive measures allow to
recover \textbf{one-vs-all reductions} (\S\ref{sec:one_vs_all}), theoretically
justifying this popular scheme. We then give examples of loss
instantiations (\S\ref{sec:ova_examples}).

\subsection{Uniformly separable regularizers and one-vs-all loss functions}
\label{sec:one_vs_all}

A particularly simple case is that of uniformly separable $\Omega$, i.e.,
$\Omega(\m) = \sum_{j=1}^d \phi(m_j) + I_{\RR_+^d}(\m)$, for some $\phi \colon
\RR_+ \to \RR$. In that case the regularized prediction
function can be computed in a coordinate-wise fashion:
\begin{equation}
(\yHatOmega(\s))_j = 
\argmax_{m \in \RR_+} ~ m \ss_j - \phi(m).
\end{equation}
As we shall later see, this simplified optimization problem often enjoys a
closed-form solution.
Intuitively, $(\yHatOmega(\s))_j$ 
can be interpreted as the ``unnormalized probability'' of class $j$.
Likewise, the corresponding Fenchel-Young loss is separable over classes:
\begin{equation}
L_\Omega(\s; \y) 
= \Omega^*(\s) + \Omega(\y) - \DP{\s}{\y}
= \sum_{j=1}^d \phi^*(\ss_j) + \phi(y_j) - \ss_j y_j
= \sum_{j=1}^d L_\phi(\ss_j; y_j), 
\end{equation}
where $y_j \in \{1,0\}$ indicates membership to class $j$. 
The separability allows to train the model producing each $\ss_j$ in an
embarrassingly parallel fashion.

Let us now consider the case $\phi(p) = -\HH([p,1-p])$, where $\HH$ is a
generalized entropy and $\phi$ is restricted to $[0,1]$.  From
\S\ref{sec:simplex_binary_case}, we know that $\phi^*(s) = \phi^*(-s) + s$ for
all $s \in \RR$. If
$\HH$ further satisfies assumption A.1, meaning that $\phi(1)=\phi(0)=0$, then
we obtain
\begin{equation}
L_\phi(s; 1) = \phi^*(s) + \phi(1) - s = \phi^*(-s) 
\quad \text{and} \quad
L_\phi(s; 0) = \phi^*(s) + \phi(0) = \phi^*(s).
\end{equation}
Combining the two, we obtain
\begin{equation}
    L_\Omega(\s; \y) = \sum_{j=1}^d \phi^*(-(2 y_i - 1) \ss_i),
\end{equation}
where we used that $(2 y_i - 1) \in \{+1, -1\}$.
Fenchel-Young losses thus
recover classical \textbf{one-vs-all} loss functions.
Since Fenchel-Young losses satisfy $L_\Omega(\s; \y) = 0 \Leftrightarrow
\yHatOmega(\s) = \y$, our framework provides a theoretical justification for
one-vs-all, a scheme that works well in practice \citep{defense_ova}.
Further, our framework justifies using $\yHatOmega(\s)$ as a measure of class
membership, as commonly implemented (with post-normalization) in software
packages \citep{sklearn,sklearn_api}.

Finally, we point out that in the binary case, choosing $\s = [\qq, -\qq]$, we
have the relationship
\begin{equation}
L_\Omega(\s; \e_1) = 2 \phi^*(-s)
\quad \text{and} \quad
L_\Omega(\s; \e_2) = 2 \phi^*(s).
\end{equation}
Thus, we recover the same loss as in \S\ref{sec:simplex_binary_case}, up to a
constant 2 factor (i.e., learning over the simplex $\triangle^d$ or over the
unit cube $[0,1]^d$ is equivalent for binary classification).

\subsection{Examples} 
\label{sec:ova_examples}

Recall that $\Omega(\m) = \sum_{j=1}^d \phi(m_j) + I_{\RR_+^d}(\m)$.
Choosing $\phi(p) = \frac{1}{2} p^2$ gives
\begin{equation}
\yHatOmega(\s)
= [\s]_+
\quad \text{and} \quad
L_\Omega(\s; \y) = \sum_{i=1}^d \frac{1}{2} [\ss_i]_+^2 + \frac{1}{2} y_i^2 -
\ss_i y_i.
\end{equation}

Choosing $\phi(p) = -\HHs([p,1-p]) + I_{[0,1]}(p)$ leads to
\begin{equation}
\yHatOmega(\s) = \sigmoid(\s)
\coloneqq \frac{\ones}{\ones + \exp(-\s)},
\quad \text{and} \quad
L_\Omega(\s; \y) = \sum_{i=1}^d \log(1 + \exp(-(2 y_i-1) \ss_i))
\end{equation}
the classical sigmoid function and the one-vs-all logistic function.
Note that $-\Omega(\p) = \sum_j \HHs([p_j,1-p_j]) = -\sum_j p_j \log p_j +
(1-p_j) \log (1-p_j)$ is sometimes known as the Fermi-Dirac entropy
\citep[Commentary of Section 3.3]{borwein_2010}.

Choosing $\phi(p) = -\HHt_2([p,1-p]) + I_{[0,1]}(p) = p^2 - p
+ I_{[0,1]}(p)$ leads to 
\begin{equation}
(\yHatOmega(\s))_j = 
(\phi^*)'(s_j) = 
\begin{cases}
    0 & \mbox { if } u \le -1 \\
     1 & \mbox { if } u \ge 1 \\
    \frac{1}{2} (u+1) & \mbox{ o.w. }
\end{cases} 
\quad \forall j \in [d].
\end{equation}
We can think of this function as a ``sparse sigmoid''. 

\section{Structured prediction with Fenchel-Young losses}
\label{sec:structured_prediction}

In this section, we now turn to prediction over the convex hull of a set $\cY$
of structured objects (sequences, trees, assignments, etc.), represented as
$d$-dimensional vectors, i.e. $\cY \subseteq \RR^d$. 
In this case, the regularized prediction function \eqref{eq:prediction} becomes
\begin{equation}
    \yHatOmega(\s) = \argmax_{\bmu \in \conv(\cY)} \DP{\bmu}{\s} - \Omega(\bmu).
\label{eq:structured_yhat}
\end{equation}
Typically, $\sizeY$ will be exponential in $d$, i.e., $d \ll \sizeY$.  
Because $\yHatOmega(\s) = \sum_{\y \in \cY} p(\y) \y = \EE_\p[Y]$ for some $\p
\in \triangleY$, $\yHatOmega(\s)$ can be interpreted as the expectation under
some (not necessarily unique) underlying distribution.

We first discuss the concepts of probability distribution regularization
(\S\ref{sec:proba_regul}) and mean regularization (\S\ref{sec:mean_regul}), and
the implications in terms of computational tractability and identifiability
(note that in the unstructured setting, probability distribution
and mean regularizations coincide).
We then give several examples of polytopes and show how a Fenchel-Young loss
can seamlessly be constructed over them given access to regularized prediction
functions (\S\ref{sec:polytope_examples}).
Finally, we extend the notion of separation margin to the structured setting
(\S\ref{sec:structured_margin}).

\subsection{Distribution regularization, marginal inference and structured sparsemax}
\label{sec:proba_regul}

In this section, we discuss the concept of \textbf{probability distribution
regularization}. Our treatment follows closely the variational formulations
of exponential families \citep{exponential_families,wainwright_2008} and
generalized exponential families \citep{grunwald_2004,frongillo_2014} but
adopts the novel viewpoint of regularized prediction functions.
We discuss two instances of that framework: the
structured counterpart of the softmax, marginal inference,
and a new structured counterpart of the sparsemax, structured sparsemax.

Let us start with a probability-space regularized prediction function
\begin{equation}
\widehat{\y}_{-\HH}(\qtheta) = \argmax_{\p \in \triangleY} \DP{\p}{\qtheta} +
\HH(\p),
\end{equation}
where $\HH$ is a generalized entropy and  $\qtheta \coloneqq (\DP{\s}{\y})_{\y
\in \cY} \in \RRY$
is a vector that contains the scores of all possible structures.
Note that in the unstructured setting, where $\cY = \{\e_1,\dots,\e_d\}$, 
we have $\qtheta = \s$ and hence the distinction between $\s$ and $\qtheta$
is not necessary.
In the structured prediction setting, however,
the above optimization problem is defined over a space of size $\sizeY \gg d$.
The regularized prediction function outputs a probability distribution over
structures and from Danskin's theorem, we have
\begin{equation}
    \widehat{\y}_{-\HH}(\qtheta) = \nabla_{\qtheta} (-\HH)^*(\qtheta) = \p^\star
    \in \triangleY.
\end{equation}
Using the chain rule, differentiating \wrt $\s$ instead of $\qtheta$ gives
\begin{equation}
    \nabla_{\s} (-\HH)^*(\qtheta) = \EE_{\p^\star}[Y] \in \conv(\cY).
\end{equation}
Therefore, the gradient \wrt $\s$ corresponds to
the mean under $\p^\star$. This can be equivalently expressed under our
framework by defining a
regularization function directly in mean space. Let us define the regularization
function $\Omega$ over $\conv(\cY)$ as
\begin{equation}
-\Omega(\bmu) \coloneqq \max_{\p \in \triangleY} \HH(\p) 
\quad \text{s.t.} \quad \EE_\p[Y]=\bmu.
\label{eq:generalized_max_entropy}
\end{equation}
That is, among all distributions 
satisfying the first-moment matching
condition $\EE_\p[Y]=\bmu$,
we seek the distribution $\p^\star$ with maximum (generalized) entropy. This
allows to make the underlying distribution unique and identifiable.
With this choice of $\Omega$, simple calculations show that the corresponding
regularized prediction function is precisely the mean under that distribution:
\begin{equation}
\yHatOmega(\s) = \nabla \Omega^*(\s) = \EE_{\p^\star}[Y] \in \conv(\cY).
\end{equation}
Similar results hold for higher order moments: If $\Omega^*$ is
twice-differentiable, the Hessian corresponds to the second moment under
$\p^\star$:
\begin{equation}
\textnormal{cov}_{\p^\star}[Y] = \nabla^2 \Omega^*(\s).
\end{equation}

The relation between these various maps
is summarized in Figure \ref{fig:map_diagram}.
\begin{figure}[t]
\begin{center}
\includegraphics[scale=1.2]{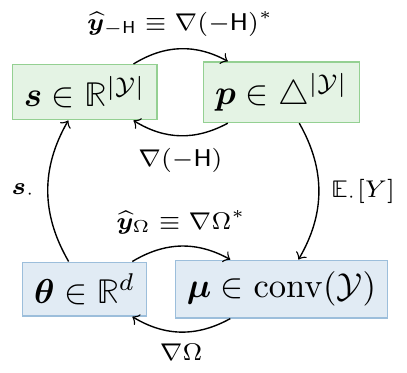}
\end{center}
\caption{{\bf Summary of maps between spaces.}
    We define $\Omega$ over $\conv(\cY)$ using a generalized maximum entropy
    principle:
$-\Omega(\bmu) \coloneqq \max_{\p \in \triangleY} \HH(\p) 
\text{ s.t. } \EE_\p[Y]=\bmu$
(cf. \S\ref{sec:proba_regul}).  
We also define $\qtheta \coloneqq (\DP{\s}{\y})_{\y
\in \cY} \in \RRY$,
the vector that contains the scores of all possible structures, a linear map
from $\RR^d$ to $\RRY$.
Note that the diagram is not necessarily commutative.} \label{fig:map_diagram}
\end{figure}

\paragraph{Marginal inference.}

A particular case of \eqref{eq:generalized_max_entropy} is 
\begin{equation}
-\Omega(\bmu) = \max_{\p \in \triangleY} \HHs(\p) 
\quad \text{s.t.} \quad \EE_\p[Y]=\bmu,
\label{eq:Omega_CRF}
\end{equation}
where $\HHs$ is the Shannon entropy. 
The distribution achieving
maximum Shannon entropy, which is unique, is the Gibbs distribution
$p(\y; \s) \propto \exp(\DP{\s}{\y})$.

As shown in Table \ref{tab:fy_losses_examples},
the resulting loss $L_\Omega$ is 
the CRF loss \citep{Lafferty2001}
and the resulting regularized prediction function $\yHatOmega$ is known as 
marginal inference in the literature \citep{wainwright_2008}:
\begin{equation}
\yHatOmega(\s) = \marginals(\s) \coloneqq    
\sum_{\y \in \cY} \exp(\DP{\s}{\y}) \y \big / Z(\s),
\end{equation}
where $Z(\s) \coloneqq \sum_{\y \in \cY} \exp(\DP{\s}{\y})$ is the partition
function (normalization constant) of the Gibbs distribution.
Although marginal inference is intractable in general, we will see in
\S\ref{sec:polytope_examples} that it can
be computed exactly and efficiently for specific polytopes.  
The conjugate of $\Omega(\bmu)$, $\Omega^*(\s)$, corresponds to the log
partition function: $\Omega^*(\s) = \log Z(\s)$.

\paragraph{Structured sparsemax.}

From the above perspective, it is tempting to define a sparsemax counterpart of
the CRF loss and of marginal inference by replacing Shannon's entropy $\HHs$
with the Gini index $\HHt_2$ (see \eqref{eq:gini_entropy} for a definition) in
\eqref{eq:Omega_CRF}:
\begin{equation}
-\Omega(\bmu) = \max_{\p \in \triangleY} \HHt_2(\p) 
\quad \text{s.t.} \quad \EE_\p[Y]=\bmu.
\end{equation}
This is equivalent to seeking a distribution $\p^\star$ with minimum squared norm
\begin{equation}
    \Omega(\bmu) + \frac{1}{2} = \min_{\p \in \triangleY} \frac{1}{2} \|\p\|^2 
\quad \text{s.t.} \quad \EE_\p[Y]=\bmu.
\end{equation}
The corresponding $\yHatOmega(\s)$ is then $\yHatOmega(\s) = \EE_{\p^\star}[Y]$.
Recalling that $\p^\star = \widehat{\y}_{-\HHt_2}(\qtheta)$,
a naive approach would be to first compute the score vector $\qtheta =
(\DP{\s}{\y})_{\y \in \cY} \in \RRY$ and then to
project that vector on the simplex to obtain $\p^\star$.
Unfortunately, $\qtheta$ could be exponentially large and $\p^\star$ could be
arbitrarily dense in the worst case, making that approach intractable in
general.
As also noted by \citet{smoother_way}, we can approximate the problem by further
imposing an upper-bound on the sparsity of the distribution, $\|\p\|_0 \le k$.
Since the optimal $k$-sparse distribution solely depends on the top-$k$ elements
of $\qtheta$ \citep{kyrillidis_2013}, we can first use a $k$-best oracle to
retrieve the top-$k$ elements of $\qtheta$ and project them onto the simplex. A
disadvantage of that approach is that $k$-best oracles are usually more complex
than MAP oracles and $k$ could be arbitrarily large in order to guarantee an
exact solution to the original problem. For shortest path problems over a
directed acyclic graph, an alternative approximation is to smooth the maximum
operator directly in the dynamic programming recursion
\citep{differentiable_dp}. 

\subsection{Mean regularization and SparseMAP}
\label{sec:mean_regul}

Marginal inference is computationally tractable only for certain 
polytopes and structured sparsemax requires approximations, even for
polytopes for which exact marginal inference is available. In this section,
we discuss an alternative approach which only requires access to a MAP oracle,
broadening the set of applicable polytopes.
The key idea is \textbf{mean regularization}: We directly regularize the mean
$\bmu$ rather than the distribution $\p$. 

The mean regularization counterpart of marginal inference is
\begin{equation}
\widehat{\y}_{-\HHs}(\s) = \argmax_{\bmu \in \conv(\cY)} \DP{\s}{\bmu} + \HHs(\bmu)
= \argmin_{\bmu \in \conv(\cY)} \KL(\bmu \| e^{\s - \ones}).
\end{equation}
It has been used for specific convex polytopes, most importantly in the optimal
transport literature \citep{cuturi_2013,peyre_2017} but also for learning to
predict permutation matrices \citep{helmbold_2009} or permutations
\citep{yasutake_2011,ailon_2016}.
The mean regularization counterpart of sparsemax is known as SparseMAP
\citep{sparsemap}: 
\begin{equation}
\yHatOmega(\s) 
= \sparsemap(\s) \coloneqq
\argmax_{\bmu \in \conv(\cY)} \DP{\s}{\bmu} 
- \frac{1}{2} \|\bmu\|^2 
= \argmin_{\bmu \in \conv(\cY)} \|\bmu - \s\|^2.
\label{eq:sparsemap}
\end{equation}

The main advantage of mean regularization is computational.
Indeed, computing $\yHatOmega(\s)$ now simply involves a $d$-dimensional
optimization problem instead of a $|\cY|$-dimensional one and
can be cast as a Bregman projection onto $\conv
(\cY)$(\S\ref{sec:relation_Bregman}). For specific polytopes, that projection
can often be computed
directly (\S\ref{sec:polytope_examples}). More generally, it can always be
computed to arbitrary precision given access to a MAP oracle
\begin{equation}
\argmax_{\y \in \conv(\cY)} \DP{\s}{\y}=
\argmax_{\y \in \cY} \DP{\s}{\y} \eqqcolon \map(\s),
\end{equation}
thanks to conditional gradient algorithms (\S\ref{sec:computing_yOmega}).
Since MAP inference is a cornerstone of structured prediction,
efficient algorithms have been developed for many kinds of structures
(\S\ref{sec:polytope_examples}).
In addition, as conditional gradient algorithms maintain a convex combination of
vertices, they can also return a (not necessarily unique) distribution $\p \in
\triangleY$ such that $\yHatOmega(\s) = \EE_\p[Y]$.
From Carath\'{e}odory's theorem, the support of $\p$
contains at most $d \ll |\cY|$ elements. This ensures that $\yHatOmega(\s)$ can
be written as a ``small'' number of elementary structures.
The price to pay for this computational tractability is that the underlying
distribution $\p$ is not necessarily unique.

\subsection{Examples}
\label{sec:polytope_examples}

Deriving loss functions for structured outputs can be challenging.  In this
section, we give several examples of polytopes $\conv(\cY)$ for which
efficient computational oracles (MAP, marginal inference, projection) are
available, thus allowing to obtain a Fenchel-Young loss for learning over these
polytopes.  A summary is given in Table \ref{table:polytopes}.

\begin{table}[t]
\label{table:polytopes}
\caption{{\bf Examples of convex polytopes and computational cost of the three main
    computational oracles:} MAP, marginal inference and projection (in the Euclidean
    distance sense or more generally in the Bregman divergence sense). 
For polytopes for which
direct marginal inference or projection algorithms are not available, we can
always use conditional gradient algorithms to compute an optimal solution to
arbitrary precision --- see \S\ref{sec:computing_yOmega}.
}
\begin{center}
\begin{small}
\begin{tabular}{lccccc}
\toprule
Polytope & Vertices & Dim. & MAP & Marginal & Projection \\
\midrule
Probability simplex & Basis vectors & $d$ & $O(d)$ & $O(d)$ & $O(d)$ \\
                    & Sequences & $nm^2$ & $O(nm^2)$ & $O(nm^2)$ & N/A \\
Arborescence & Spanning trees & $n(n-1)$ & $O(n^2)$ & $O(n^3)$ & N/A \\
             & Alignments & $nm$ & $O(nm)$ & $O(nm)$ & N/A \\
Permutahedron & Permutations & $d$ & $O(d \log d)$ & N/A & $O(d \log d)$ \\
Birkhoff & Permutation matrices & $n^2$ & $O(n^3)$ & \#P-complete & $O(n^2 /
\epsilon)$ \\
\bottomrule
\end{tabular}
\end{small}
\end{center}
\end{table}

\paragraph{Sequences.}

We wish to tag a sequence $(\x_1, \dots, \x_n)$ of vectors in $\RR^p$
(\textit{e.g.,} word representations) with the most likely output sequence
(\textit{e.g.,} entity tags) $\bm{s} = (s_1, \dots, s_n) \in [m]^n$.  It is
convenient to represent each sequence $\bm{s}$ as a $n \times m \times m$ binary
tensor $\bm{y} \in \cY$, such that $y_{t,i,j} = 1$ if $\y$ transitions from node
$j$ to node $i$ on time $t$, and~$0$ otherwise.  The potentials $\s$ can
similarly be organized as a $n \times m \times m$ real tensor, such that
$\theta_{t,i,j} = \phi_t(\x_t, i, j)$, where $\phi_t$ is a potential function.
Using the above binary tensor representation, the Frobenius inner product
$\langle \s, \y \rangle$ is equal to $\sum_{t=1}^{n} \phi_t(\x_t, s_{t},
s_{t-1})$, the cumulated score of $\bm{s}$.

MAP inference, $\argmax_{\y \in \cY} \DP{\y}{\s}$, seeks the
highest-scoring sequence and can be computed using
Viterbi's algorithm \citep{viterbi_error_1967} in $O(nm^2)$ time.  Marginal
inference can be computed in the same cost using the forward-backward algorithm
\citep{baum_1966} or equivalently using backpropagation
\citep{eisner_inside-outside_2016,differentiable_dp}.

When $\Omega$ is defined over $\conv(\cY)$ as in \eqref{eq:Omega_CRF},
the Fenchel-Young loss corresponds to a
linear-chain conditional random field loss \citep{Lafferty2001}. In recent
years, this loss has been used to train various end-to-end natural language
pipelines based on neural networks \citep{collobert_2011,lample_2016}.

\paragraph{Alignments.}

Let $\bm{A} \in \RR^{m \times p}$ and $\bm{B} \in \RR^{n \times p}$ be two
time-series of lengths $m$ and $n$, respectively.  We denote by $\bm{a}_i \in
\RR^p$ and $\bm{b}_j \in \RR^p$ their $i^{\text{th}}$ and $j^\text{th}$
observations.  Our goal is to find an alignment between $\bm{A}$ and $\bm{B}$,
matching their observations.
We define $\s$ as a $m \times n$ matrix, such that
$\theta_{i,j}$ is the similarity between observations $\bm{a}_i$ and
$\bm{b}_j$.
Likewise, we represent an alignment $\y$ as a $m \times n$ binary matrix, such
that $y_{i,j}=1$ if $\bm{a}_i$ is aligned with $\bm{b}_j$, and $0$ otherwise. 
We write $\cY$ the set of all monotonic alignment matrices, such that the path
that connects the upper-left $(1,1)$ matrix entry to the lower-right $(m,n)$
one uses only $\downarrow,\rightarrow,\searrow$ moves. 

In this setting, MAP inference, $\argmax_{\y \in \cY} \DP{\y}{\s}$, corresponds
to seeking the maximal similarity (or equivalently,  minimal cost) alignment
between the two time-series. It can be computed in $O(mn)$ time using dynamic
time warping, DTW \citep{Sakoe78}. Marginal inference can be computed
in the same cost by backpropagation, if we replace
the hard minimum with a soft one in the DTW recursion. This algorithm is known
as soft-DTW \citep{soft_dtw,differentiable_dp}.

Structured SVMs were combined with DTW to learn to predict music-to-score
alignments \citep{garreau_metric_2014}. Our framework easily enables extensions
of this work, such as replacing DTW with soft-DTW, which
amounts to introducing entropic regularization \wrt the probability
distribution over alignments.

\paragraph{Spanning trees.}

When $\cY$ is the set of possible directed spanning trees (arborescences) of a
complete graph $\mathcal{G}$ with $n$ vertices, the convex hull $\conv(\cY) \subset \RR^{n(n-1)}$ is known as the
arborescence polytope \citep{andre-concise} (each $\y \in \cY$ is a binary vector which indicates which arcs belong to the arborescence). MAP inference may be performed by
maximal arborescence algorithms \citep{Chu1965,Edmonds1967} in $O(n^2)$ time
\citep{tarjan}, and the Matrix-Tree theorem \citep{Kirchhoff1847} provides an
$O(n^3)$ marginal inference algorithm
\citep{koo-mt,DSmithSmith2007}.
Spanning tree structures are often used in natural language processing for
(non-projective) \emph{dependency parsing}, with graph edges corresponding to dependencies
between the words in a sentence.
\citep{mstparser,kg}.

\paragraph{Permutations.}

We view ranking as a structured prediction problem.
Let $\cY$ be the set of $d$-permutations of a prescribed vector $\w \in \RR^d$,
i.e., $\cY = \{ [w_{\pi_1}, \dots, w_{\pi_d}] \in \RR^d \colon \bpi \in \Pi\}$,
where $\Pi$ denotes the permutations of $(1, \dots, d)$.
We assume without loss of generality that $\w$ is sorted in descending order,
i.e., $w_1 \ge \dots \ge w_d$.
MAP inference seeks the permutation of $\w$ whose inner product
with $\s$ is maximized:
\begin{equation}
\max_{\bm{y} \in \cY} \DP{\s}{\y} 
= \max_{\bm{y} \in \cY} \sum_{i=1}^d \ss_i y_i
= \max_{\bpi \in \Pi} \sum_{i=1}^d \ss_i w_{\pi_i}
= \max_{\bpi \in \Pi} \sum_{i=1}^d \ss_{\pi_i} w_i.
\end{equation}
Since $\w$ is assumed sorted, an optimal solution $\bpi^\star$ of the last
optimization problem is a permutation sorting $\s$ in descending order.
The function $\s \mapsto \sum_{i=1}^d \theta_{\pi^\star_i} w_i$ is called
ordered weighted averaging (OWA) operator \citep{yager_1988} and includes the
mean and max operator as special cases. MAP
inference over $\cY$ can be seen as the variational formulation of the OWA
operator. An optimal solution 
$\y^\star$ is simply $\w$ sorted using the inverse permutation of $\bpi^\star$.
The overall computational cost is therefore $O(d \log d)$, for sorting $\s$.

The convex hull of $\cY$, $\conv(\cY)$, is known as the \emph{permutahedron} 
when $\w = [d,\dots,1]$. For arbitrary $\w$, we follow 
\citet{projection_permutahedron} and call $\conv(\cY)$ the permutahedron induced
by $\w$.
Its vertices correspond to permutations of $\w$.

We can define mean-regularized ranking prediction functions $\yHatOmega(\s)$ if we
choose $\w$ such that each $w_i$ represents the ``preference'' of being in
position $i$. 
Intuitively, the score vector $\s \in \RR^d$ should be such that each $\ss_i$
is the score of instance $i$ (e.g., a document or a label).
We give several examples below.
\begin{itemize}[topsep=0pt,itemsep=3pt,parsep=3pt,leftmargin=8pt]
    \item Choosing $\w = [1, 0, \dots, 0]$, $\conv(\cY)$ is equal to
        $\triangle^d$, the probability simplex. We thus recover probabilistic
        classification as a natural special case.

    \item Choosing $\w = \frac{1}{k} [\underbrace{1, \dots, 1}_{k \text{
        times}}, 0, \dots, 0]$, $\conv(\cY)$ is equal to $\{\p \in \triangle^d
        \colon \|\p\|_\infty \le \frac{1}{k} \}$, sometimes referred to as the
        capped probability simplex
        \citep{warmuth_2008,projection_permutahedron}. This setting corresponds
        to predicting $k$-subsets.

    \item Choosing $\w = \frac{2}{d (d+1)} [d, d-1, \dots, 1]$ corresponds to
        predicting full rankings.

    \item Choosing $\w = \frac{2}{k (k+1)} [k, k-1, \dots, 1, 
        \underbrace{0, \dots, 0}_{d - k - 1 \text{ times}}]$
        corresponds to predicting partial rankings.
\end{itemize}
The corresponding polytopes are illustrated in Figure \ref{fig:permutahedron}.
In all the examples above, $\w \in \triangle^d$, implying $\conv(\cY)
\subseteq \triangle^d$. Therefore, $\yHatOmega(\s)$ outputs a probability
distribution.

\begin{figure}[t]
\centering
\includegraphics[scale=3.5]{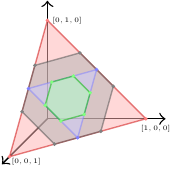}
\caption{{\bf Examples of instances of permutahedron induced by $\w$.}
Round circles indicate vertices of the permutahedron, permutations of $\w$.
Choosing $\w = [1,0,0]$ recovers the probability simplex (red) while choosing
$\w = [\frac{1}{2},\frac{1}{2},0]$ recovers the capped probability simplex
(blue).  The other two instances are obtained by $\w =
[\frac{2}{3},\frac{1}{3},0]$ (gray) and
$\w=[\frac{1}{2},\frac{1}{3},\frac{1}{6}]$ (green). Euclidean projection onto
these polytopes can be cast as isotonic regression. More generally, Bregman
projection reduces to isotonic \textit{optimization}.}
\label{fig:permutahedron}
\end{figure}

As discussed in \S\ref{sec:relation_Bregman}, computing the regularized
prediction function $\yHatOmega(\s)$ is equivalent
to a Bregman projection when $\Omega = \Psi + I_\cC$, where $\Psi$ is Legendre type.
The Euclidean projection onto $\cC=\conv(\cY)$ reduces to isotonic regression
\citep{zeng_2014,orbit_regul}. The computational cost is $O(d \log d)$. More
generally, Bregman projections reduce to isotonic \textit{optimization}
\citep{projection_permutahedron}. This provides a unified way to compute
$\yHatOmega(\s)$ efficiently, regardless of $\w$.

The generated Fenchel-Young loss $L_\Omega(\s; \y)$, is illustrated in Figure
\ref{fig:ranking_loss} for various choices of $\Omega$. When $\Omega=0$, as
expected, the loss is zero as long as the
predicted ranking is correct.  Note that in order to define a meaningful loss,
it is necessary that $\y \in \cY$ or more generally $\y \in \conv(\cY)$. That
is, $\y$ should belong to the convex hull of the permutations of $\w$.

Permutahedra have been used to derive online learning to rank algorithms
\citep{yasutake_2011,ailon_2016} but it is not obvious how to extract a loss
from these works.
Ordered weighted averaging (OWA) operators have been used to define related
top-$k$ multiclass losses \citep{usunier_2009,lapin_2015} but without
identifying the connection with permutahedra. 
Our construction follows directly from the general Fenchel-Young loss framework
and provides a novel geometric perspective.

\begin{figure}[t]
\centering
\includegraphics[width=\linewidth]{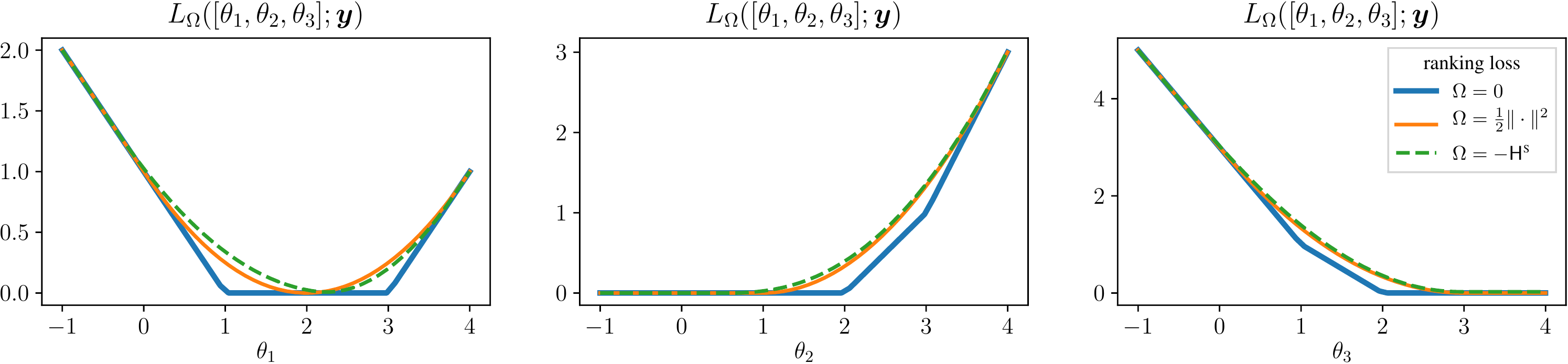}
\caption{\textbf{Ranking losses} generated by restricting $\dom(\Omega)$ to the
    permutahedron induced by $\w = [3, 2, 1]$. We set the ground-truth $\y=
[2,1,3]$, i.e., a permutation of $\w$. We set $\s=\y$ and inspect how the loss
changes when varying each $\ss_i$.
When $\Omega=0$, the loss is zero as expected when $\theta_1 \in [1,3]$,
$\theta_2 \in (-\infty,2]$ and $\theta_3 \in [2,\infty)$, since the ground-truth
ranking is still correctly predicted in these intervals.
When $\Omega=\frac{1}{2}\|\cdot\|^2$ and $\Omega=-\HHs$, we obtain a smooth
approximation. Although not done here for clarity, if $\w$ is normalized 
such that $\w \in \triangle^d$, then $\yHatOmega(\s) \in \triangle^d$ as well.
}
\label{fig:ranking_loss}
\end{figure}

\paragraph{Permutation matrices.}

Let $\cY \subset \{0,1\}^{n \times n}$ be the set of $n \times n$ permutation
matrices.
MAP inference is the solution of the linear assignment problem
\begin{equation}
\max_{\bm{y} \in \cY} \DP{\s}{\y}
= \max_{\bpi \in \Pi} \sum_{i=1}^n \theta_{i, \pi_i},
\end{equation}
where $\Pi$ denotes the permutations of $(1,\dots,n)$.
The problem can be solved exactly in $O(n^3)$ time using the Hungarian
algorithm~\citep{hungarian} or the Jonker-Volgenant algorithm~\citep{lapjv}.
Noticeably, marginal inference is known to be \#P-complete
\citep[Section 3.5]{valiant1979complexity,taskar-thesis}. This makes it an open
problem how to solve marginal inference for this polytope.

In contrast, projections on 
the convex hull of $\cY$, 
the set of doubly stochastic matrices known as Birkhoff polytope
\citep{birkhoff}, can be computed efficiently. Since the
Birkhoff polytope is a special case of transportation polytope, we can leverage
algorithms from the optimal transport
literature to compute the mean-regularized prediction function.
When $\Omega(\bmu) = \DP{\bmu}{\log \bmu}$, $\yHatOmega(\s)$ can be computed
using the Sinkhorn algorithm \citep{sinkhorn,cuturi_2013}. For other
regularizers, we can use Dykstra's algorithm \citep{rot_mover} or dual
approaches \citep{blondel_2018}.  The cost of obtaining an
$\epsilon$-approximate solution is typically $O(n^2/\epsilon)$.

The Birkhoff polytope has been used to define continuous relaxations of
non-convex ranking losses \citep{sinkprop}. In contrast, the Fenchel-Young loss
over the Birkhoff polytope (new to our knowledge) is convex
by construction.
Although working with permutation matrices is more computationally expensive
than working with permutations, it brings different modeling power, since it
allows to take into account similarity between instances (e.g., text documents)
through the similarity matrix $\s$.  It also enables other applications, such as
learning to match.

\subsection{Structured separation margins}
\label{sec:structured_margin}

We end this section by extending some of the results in \S\ref{sec:margin} for the structured prediction case, showing that {\bf there is also a relation between structured entropies and margins.} 
In the sequel, we assume that structured outputs are contained in a sphere of radius $r$, i.e., $\|\y\| = r$ for any $\y \in \cY$. 
This holds for all the examples above (sequences, alignments, spanning trees, permutations of a given vector, and permutation matrices). In particular, it holds whenever outputs are represented as binary vectors with a constant number of entries set to 1; this includes overcomplete parametrizations of discrete graphical models 
\citep{wainwright_2008}. 

\smallskip

\begin{definition}{\label{def:structured_margin}Structured separation margin}

Let $L(\s; \y)$ be a loss function over $\RR^d \times \cY$.
We say that $L$ has
the \emph{structured separation margin property} if there exists $m>0$ such that, for any $\y \in \cY$:
\begin{equation}
\DP{\s}{\y} \ge \max_{\y' \in \cY} \left( \DP{\s}{\y'} + \frac{m}{2}\|\y - \y'\|^2\right) \quad \Rightarrow \quad L(\s; \y) = 0.
\label{eq:margin_definition_struct}
\end{equation}
The smallest possible $m$ that satisfies \eqref{eq:margin_definition_struct} is called 
the \emph{margin} of $L$, denoted $\mathrm{margin}(L)$.
\end{definition}

Note that this definition generalizes the unstructured case (Definition~\ref{def:margin}), which is recovered when $\cY = \{\e_i\}_{i=1}^d$. 
Note also that, when outputs are represented as binary vectors, 
the term $\|\y - \y'\|^2$ is a {\bf Hamming distance}, which counts how many bits need to be flipped to transform $\y'$ into $\y$. 
The most famous example of a loss with a structured separation margin in the {\bf structured hinge loss} used in structured support vector machines \citep{taskar-thesis,structured_hinge}.

The next proposition extends Proposition~\ref{prop:margin}. We defer its proof to 
\S\ref{appendix:proof_structured_margin}. 
\vspace{0.5em}
\begin{proposition}\label{prop:structured_margin}
Assume $\Omega$ is convex and $\cY$ is contained in a sphere of radius $r$. Then:
\begin{enumerate}
\item The loss $L_{\Omega}$ has a structured separation margin iff there is  a $m>0$ such that, for any $\y \in \cY$, $m\y \in \partial \Omega(\y)$. 
\item If the above holds, then the margin of $L_{\Omega}$ is given by the smallest
such $m$ or, equivalently,
\begin{equation}\label{eq:structured_margin}
\mathrm{margin}(L_{\Omega}) = \sup_{\bmu \in \conv(\cY),\, \y \in \cY} \frac{\Omega(\y) - \Omega(\bmu)}{r^2 - \DP{\y}{\bmu}}.
\end{equation}
\end{enumerate}
\end{proposition}

\paragraph{Unit margin of the SparseMAP loss.} 
We can invoke Proposition~\ref{prop:structured_margin} to show that the SparseMAP loss of \citet{sparsemap} has a structured margin of $1$, a novel result that follows directly from our construction. 
Indeed, from \eqref{eq:structured_margin}, we have that, for $\Omega(\bmu) = \frac{1}{2}\|\bmu\|^2$:
\begin{equation}
\mathrm{margin}(L_{\Omega}) = \sup_{\bmu \in \conv(\cY),\, \y \in \cY} \frac{\frac{1}{2}r^2 - \frac{1}{2}\|\bmu\|^2}{r^2 - \DP{\y}{\bmu}} = 1 - \inf_{\bmu \in \conv(\cY),\, \y \in \cY} \frac{\frac{1}{2}\|\y-\bmu\|^2}{\DP{\y}{\y - \bmu}} \le 1,
\end{equation}
where the last inequality follows from the fact that both the numerator and denominator in the second term are non-negative, the latter due to the Cauchy-Schwartz inequality. 
We now show that, for any $\y \in \cY$, we have $\inf_{\bmu \in \conv(\cY)} \frac{\frac{1}{2}\|\y-\bmu\|^2}{\DP{\y}{\y - \bmu}} = 0$. Choosing $\bmu = t\y' + (1-t)\y$, for an arbitrary $\y' \in \cY \setminus \{\y\}$, and letting $t \rightarrow 0^+$, we obtain $\frac{\frac{1}{2}\|\y-\bmu\|^2}{\DP{\y}{\y - \bmu}} = \frac{\frac{1}{2}t\|\y-\y'\|^2}{\DP{\y}{\y - \y'}} \rightarrow 0$.



\section{Algorithms for learning with Fenchel-Young losses}
\label{sec:training_algorithms}

In this section, we present generic primal (\S\ref{sec:primal_training}) and
dual (\S\ref{sec:dual_training}) algorithms for training predictive models with
a Fenchel-Young loss $L_\Omega$ for arbitrary $\Omega$. In doing so, we
obtain unified algorithms for training models with a wealth of existing and new
loss functions.
We then discuss algorithms for computing regularized prediction functions
(\S\ref{sec:computing_yOmega}) and proximity operators
(\S\ref{sec:computing_prox}).
We summarize these ``computational oracles'' in Figure~\ref{fig:oracles}.  

\begin{figure}[t]
\begin{center}
\includegraphics[width=\linewidth]{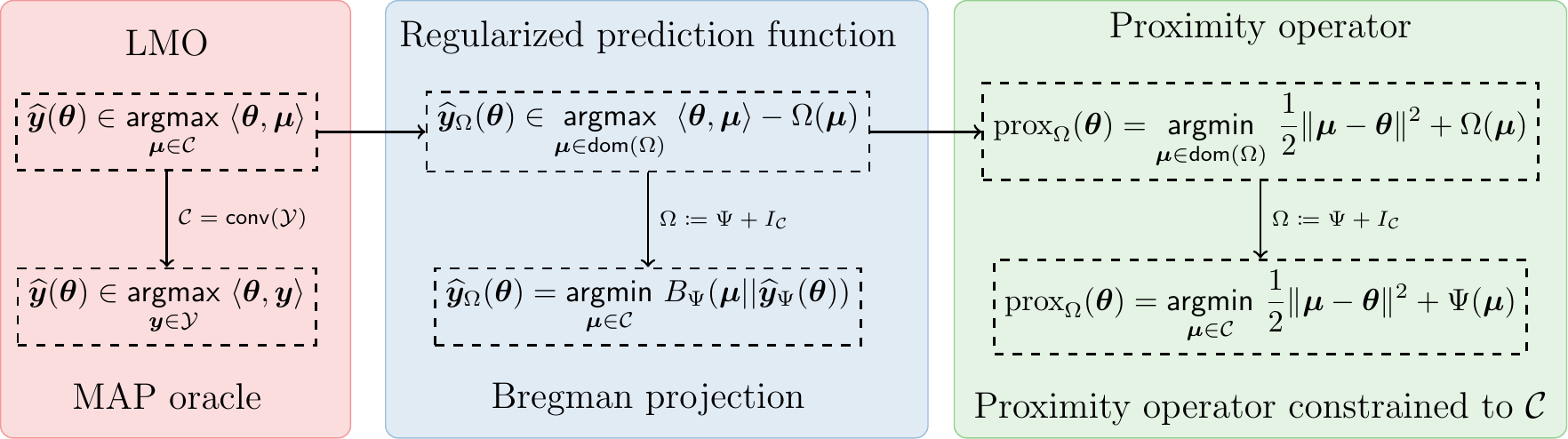}
\end{center}
\caption{\textbf{Summary of computational oracles:}
$\widehat{\y}(\s)$ is used to compute unregularized predictions or as a linear
maximization oracle in conditional gradient algorithms;
$\yHatOmega(\s)$ is used to compute regularized predictions,
loss values $L_{\Omega}(\s; \y) = f_{\s}(\y) - f_{\s}(\yHatOmega(\s))$,
where $f_{\s}(\bmu) \coloneqq \Omega(\bmu) - \DP{\s}{\bmu}$,
and loss gradients $\nabla L_{\Omega}(\s; \y) = \widehat{\y}_{\Omega}(\s) - \y$;
finally, $\prox_\Omega(\s)$ is used by training algorithms, in particular
block coordinate ascent algorithms. In the above, we assume $\Psi$ is
Legendre-type and $\cC \subseteq \dom(\Psi)$.}
\label{fig:oracles}
\end{figure}

\subsection{Primal training}
\label{sec:primal_training}

Let $\bm{f}_{W}:\cX \rightarrow \RR^d$ be a model parametrized by $W$. To learn
$W$ from training data $\{(\x_i, \y_i)\}_{i=1}^n$, $(\x_i,\y_i) \in \cX \times
\cY$, we minimize the regularized empirical risk
\begin{equation}
\min_{W} F(W) + G(W) 
\coloneqq \sum_{i=1}^n L_\Omega(\bm{f}_W(\x_i); \y_i) + G(W),
\label{eq:primal}
\end{equation}
where $G$ is a regularization function \wrt parameters $W$. Typical examples are
Tikhonov regularization, $G(W) = \frac{\lambda}{2} \|W\|_F^2$, and elastic net
regularization,
$G(W) = \frac{\lambda}{2} \|W\|^2_F + \lambda \rho \|W\|^2$, for some
$\lambda > 0$ and $\rho \ge 0$.

The objective in \eqref{eq:primal} is very broad and allows
to learn, e.g., linear models or neural networks using Fenchel-Young losses. Since $L_\Omega$ is
convex, if $\bm{f}_W$ is a linear model and $G$ is convex, then
\eqref{eq:primal} is convex as well.  Assuming further that $\Omega$ is strongly
convex, then $L_\Omega$ is smooth and \eqref{eq:primal} can be solved globally
using proximal gradient algorithms \citep{Wright2009,fista,saga}.  From
Proposition~\ref{prop:fy_losses}, the gradient of $L_\Omega(\s; \y)$ \wrt $\s$
is the residual vector
\begin{equation}
\nabla L_{\Omega}(\s; \y) = \widehat{\y}_{\Omega}(\s) - \y \in \RR^d.
\end{equation}
Using the chain rule, the gradient of $F$ \wrt $W$ is
\begin{equation}
    \nabla F(W) = \sum_{i=1}^n D^\top_{\bm{f}_W(\x_i)} 
    \nabla L_{\Omega}(\bm{f}_W(\x_i); \y_i),
\end{equation}
where $D_{\bm{f}_W(\x_i)}$ is the Jacobian of $\bm{f}_W(\x_i)$ \wrt $W$, a
linear map from the space of $W$ to $\RR^d$.
For linear models, we set $\s_i = \bm{f}_W(\x_i) = W \x_i$, where $W \in \RR^{d
\times p}$ and $\x_i \in \RR^p$, and thus get
\begin{equation}
\nabla F(W) = (\widehat{Y}_\Omega - Y)^\top X,
\end{equation}
where $\widehat{Y}_\Omega$, $Y$ and $X$ are matrices whose rows
gather $\yHatOmega(W \x_i)$, $\y_i$ and $\x_i$, for $i=1,\dots,n$.

In summary, the two main computational
ingredients to solve \eqref{eq:primal} are $\yHatOmega$ and the proximity
operator of $G$ (when $G$ is non-differentiable). 
This separation of concerns results in very modular implementations.

\paragraph{Remark on temperature scaling.}

When $\bm{f}_{W}$ is a linear model and $G$ is a homogeneous function, it is
easy to check that the temperature scaling parameter $t > 0$ from Proposition
\ref{prop:fy_losses} and the regularization strength $\lambda > 0$ above are
redundant. Hence, we can set $t=1$ at training time, without loss of generality.
However, at test time, using $\widehat{\y}_{t \Omega}(\s)$ and tuning $t$ can
potentially improve accuracy. For sparse prediction functions, $t$ also gives
control over sparsity (the smaller, the sparser).

\subsection{Dual training}
\label{sec:dual_training}

We now derive dual training of Fenchel-Young losses.
Although our treatment follows closely \citet{accelerated_sdca},
it is different in that we put the output regularization $\Omega$ at the center
of all computations, leading to a new perspective.

\paragraph{Dual objective.}

For linear models $\bm{f}_W(\x) = W \x$, where $W \in \RR^{d \times p}$,
it can sometimes be more computationally efficient to solve the corresponding
dual problem, especially in the $n \ll p$ setting.
From Fenchel's duality theorem (see for instance \citet[Theorem
3.3.5]{borwein_2010}),
we find that the dual problem of \eqref{eq:primal} is
\begin{equation}
\max_{\alpha \in \RR^{n \times d}}
-\sum_{i=1}^n L_\Omega^*(-\balpha_i; \y_i) - G^*\left(\alpha^\top X\right),
\end{equation}
and where $L_\Omega^*$ is the convex conjugate of $L_\Omega$ in its first
argument.

We now rewrite the dual problem using the specific form of $L_\Omega$.
Using $L_\Omega^*(-\balpha; \y) = \Omega(\y - \balpha) - \Omega(\y)$ and using
the change of variable $\bmu_i \coloneqq \y_i - \balpha_i$, we get
\begin{equation}
    \max_{\mu \in \RR^{n \times d}} -D(\mu) 
\text{ s.t. } \bmu_i \in \dom(\Omega) ~ \forall i \in [n],
\label{eq:dual}
\end{equation}
where we defined
\begin{equation}
    D(\mu) \coloneqq \sum_{i=1}^n \Omega(\bmu_i) - \Omega(\y_i) + G^*(V(\mu))  
    \quad \text{and} \quad
    V(\mu) \coloneqq (Y - \mu)^\top X.
\end{equation}
This expression is informative as we can interpret each $\bmu_i$ as regularized
predictions, i.e., $\bmu_i$ should belong to $\dom(\Omega)$,
the same domain as $\yHatOmega$. 
The fact that $G^*(V(\mu))$ is a function of the predictions $\mu$ is similar
to the value regularization framework of \citet{rifkin_2007}.  A key difference
with the regularization $\Omega$, however, is that $G^*(V(\mu))$
depends on the training data $X$ through $V(\mu)$.
When $G(W) = \frac{\lambda}{2} \|W\|^2_F \Leftrightarrow G^*(V) =
\frac{1}{2\lambda} \|V\|^2_F$, we obtain $G^*(V(\mu)) = \frac{1}{2\lambda}
\text{Trace}((Y - \mu)^\top K (Y - \mu))$, where $K \coloneqq X X^\top$ is the
Gram matrix.

\paragraph{Primal-dual relationship.}

Assuming $G$ is a $\lambda$-strongly convex regularizer,
given an optimal dual solution $\mu^\star$ to \eqref{eq:dual}, we may retrieve
the optimal primal solution $W^\star$ by
\begin{equation}
W^\star = \nabla G^*(V(\mu^\star)). 
\end{equation}

\paragraph{Coordinate ascent.}

We can solve \eqref{eq:dual}
using block
coordinate ascent algorithms. At every iteration, we pick $i \in [n]$ and
consider the following sub-problem associated with $i$:
\begin{equation}
\argmin_{\bmu_i \in \dom(\Omega)} 
\Omega(\bmu_i) + G^*(V(\mu)).
\label{eq:dual_subproblem}
\end{equation}
Since this problem could be hard to solve, we follow \citet[Option
I]{accelerated_sdca} and consider instead a quadratic upper-bound.
If $G$ is $\lambda$-strongly convex, $G^*$ is $\frac{1}{\lambda}$-smooth \wrt
the dual norm $\|\cdot\|$ and it holds that
\begin{equation}
G^*(V(\mu)) \le
G^*(V(\bar{\mu})) +
\langle \nabla G^*(V(\bar{\mu})), V(\mu) -
V(\bar{\mu}) \rangle + \frac{1}{2 \lambda} \|
V(\mu) - V(\bar{\mu})\|^2,
\end{equation}
where $\bar{\mu}$ denotes the current iterate of $\mu$.
Using $V(\mu) - V(\bar{\mu}) = \sum_{i=1}^n (\bar{\bmu}_i - \bmu_i)
\x_i^\top$, we get
\begin{equation}
G^*(V(\mu)) \le
G^*(V(\bar{\mu})) +
\sum_{i=1}^n
\langle\nabla G^*(V(\bar{\mu})) \x_i,
\bar{\bmu}_i - \bmu_i \rangle
+ \frac{\sigma_i}{2} \|
\bar{\bmu}_i - \bmu_i \|^2,
\end{equation}
where $\sigma_i \coloneqq \frac{\|\x_i\|^2}{\lambda}$.
Substituting $G^*(V(\mu))$ by the above upper bound into
\eqref{eq:dual_subproblem} and ignoring constant terms, we get
the following approximate sub-problem:
\begin{equation}
\argmin_{\bmu_i \in \dom(\Omega)} 
\Omega(\bmu_i) - \bmu_i^\top \bm{v}_i + \frac{\sigma_i}{2} \|\bmu_i\|^2
=
\prox_{\frac{1}{\sigma_i}\Omega}\left(\frac{\bm{v}_i}{\sigma_i}\right),
\label{eq:approx_subproblem}
\end{equation}
where $\bm{v}_i \coloneqq
\nabla G^*(V(\bar{\mu})) \x_i + \sigma_i \bar{\bmu_i}$.
Note that when $G^*$ is a quadratic function, \eqref{eq:approx_subproblem} is an
optimal solution of \eqref{eq:dual_subproblem}.

\paragraph{Examples of parameter regularization $G$.}

We now give examples of $G^*(W)$ and $\nabla G^*(W)$ for two commonly used
regularization: squared $\ell_2$ and elastic-net regularization.

When $G(W) = \frac{\lambda}{2} \|W\|^2_F$, we obtain 
\begin{equation}
G^*(V) = \frac{1}{2\lambda} \|V\|^2_F
\quad \text{and} \quad
\nabla G^*(V) = \frac{1}{\lambda} V.
\end{equation}
When $G(W) = \frac{\lambda}{2} \|W\|^2_F + \lambda \rho R(W)$, for some
other regularizer $R(W)$, we obtain
\begin{align}
\nabla G^*(V) 
&= \argmax_W ~ \langle W, V \rangle - 
\frac{\lambda}{2} \|W\|^2_F - \lambda \rho R(W) \\
&= \argmin_W ~ \frac{1}{2} \left\|W - \frac{V}{\lambda}\right\|^2_F + \rho R(W)
\\
&= \prox_{\rho R}(V / \lambda).
\end{align}
The conjugate is equal to $G^*(V) = \langle 
\nabla G^*(V), V \rangle - G(\nabla G^*(V))$.

For instance, if we choose $R(W) = \|W\|_1$, then
$G(W)$ is the \textbf{elastic-net} regularization and $\prox_{\rho R}$ is
the well-known \textbf{soft-thresholding} operator
\begin{equation}
    \prox_{\rho R}(V) = \text{sign}(V) [|V| - \rho]_+,
\end{equation}
where all operations above are performed element-wise.

\paragraph{Summary.}

To summarize, on each iteration, we pick $i \in [n]$ and
perform the update
\begin{equation}
\bmu_i \leftarrow
\prox_{\frac{1}{\sigma_i}\Omega}(\bm{v}_i / \sigma_i),
\end{equation}
where we defined
$\bm{v}_i \coloneqq \nabla G^*(V(\mu)) \x_i + \sigma_i \bmu_i$
and
$\sigma_i \coloneqq \frac{\|\x_i\|^2}{\lambda}$.
This update does not require choosing any learning rate.
Interestingly, if $\prox_{\tau\Omega}$ is sparse, then \textbf{so are the
dual variables} $\{\bmu_i\}_{i=1}^n$. Sparsity in the dual variables is
particularly useful
when kernelizing models, as it makes computing predictions more efficient.

When $i$ is picked uniformly at random, this block coordinate ascent algorithm
is known to converge to an optimal solution $W^\star$, at a linear rate if
$L_\Omega$ is smooth (i.e., if $\Omega$ is strongly convex)
\citep{accelerated_sdca}.  When $\dom(\Omega)$ is a compact set, another option
to solve \eqref{eq:dual} that does not involve proximity operators is the block
Frank-Wolfe algorithm \citep{block_fw}.

\subsection{Regularized prediction functions}
\label{sec:computing_yOmega}

The regularized prediction function
$\yHatOmega(\s)$ does not generally enjoy a closed-form expression and
one must resort to generic algorithms to compute it. In this section, we first 
discuss two such algorithms: projected gradient and conditional gradient
methods. Then, we present a new more efficient algorithm when $\dom(\Omega)
\subseteq \triangle^d$ and $\Omega$ is uniformly separable.

\paragraph{Generic algorithms.} In their greater generality, regularized
prediction functions involve the following optimization problem
\begin{equation}
\min_{\bmu \in \dom(\Omega)} f_{\s}(\bmu) \coloneqq \Omega(\bmu) -
\DP{\s}{\bmu}.
\label{eq:generic_obj}
\end{equation}
Assuming $\Omega$ is convex and smooth (differentiable and with
Lipschitz-continuous gradient), we can solve this problem to arbitrary precision
using projected gradient methods. Unfortunately, the projection onto
$\dom(\Omega)$, which is needed by these algorithms, is often as challenging to
solve as \eqref{eq:generic_obj} itself.
This is especially the case when $\dom(\Omega)$ is
$\conv(\cY)$, the convex hull of a combinatorial set of structured objects.

Assuming further that $\dom(\Omega)$ is a compact set,
an appealing alternative
which sidesteps these issues is provided by
conditional gradient (a.k.a.\ Frank-Wolfe) algorithms
\citep{frank_1956,dunn_1978,jaggi_2013}. Their main advantage stems from the
fact that they access $\dom(\Omega)$
only through the linear maximization oracle \\
$\argmax_{\y \in \dom(\Omega)} \DP{\s}{\y}$.
CG algorithms maintain the current solution as a sparse convex combination of
vertices of $\dom(\Omega)$.  At each iteration, the linear maximization oracle
is used to pick a new vertex to add to the combination. Despite its simplicity,
the procedure converges to an optimal solution, albeit at a sub-linear rate
\citep{jaggi_2013}. Linear convergence rates can be obtained using away-step,
pairwise and full-corrective variants \citep{lacoste_2015_cg}.
When $\Omega$ is a quadratic function (as in the case of SparseMAP), an
approximate correction step can achieve finite convergence efficiently. The
resulting algorithm is known as the active set method \citep[chapters
16.4 \& 16.5]{nocedalwright}, and is a generalization of Wolfe's min-norm point
algorithm \citep{mnp}.  See also \citet{vinyes} for a more detailed discussion
on these algorithms and \citet{sparsemap} for an instantiation, in the specific
case of SparseMAP.

Although not explored in this paper, similar algorithms have been developed to
minimize a function over conic hulls \citep{greedy_cone}. Such algorithms allow
to compute a regularized prediction function that outputs a conic combination of
elementary structures, instead of a convex one. It can be seen as the structured
counterpart of the regularized prediction function over positive measures
(\S\ref{sec:positive_measures}).


\paragraph{Reduction to root finding.}

When $\Omega$ is a strictly convex regularization function over $\triangle^d$
and is uniformly separable, i.e., $\Omega(\p) = \sum_{i=1}^d g(p_i)$ for some
strictly convex function $g$, we now show that $\yHatOmega(\s)$ can be computed
in linear time.
\vspace{0.5em}
\begin{proposition}%
\label{prop:root_finding}%
Reduction to root finding

Let $g \colon [0, 1] \rightarrow \RR_+$ be a strictly convex and differentiable
function. Then,
\[
\yHatOmega(\s) = 
\argmax_{\p \in \triangle^d} \DP{\p}{\s} - \sum_{j=1}^d g(\pp_j) = \p(\thresh^\star)
\]
where 
\begin{equation}
\p(\thresh) \coloneqq (g')^{-1}(\max\{\s - \thresh, g'(0)\})
\quad
\end{equation}
and where
$\thresh^\star$ is a root of
$\phi(\thresh) \coloneqq \DP{\p(\thresh)}{\ones} - 1$.

Moreover, $\tau^\star$ belongs to the tight search interval
$[\thresh_{\min},\thresh_{\max}]$, where 
\begin{equation}
\thresh_{\min} \coloneqq \max(\s) -g'(1)
\quad \text{and} \quad
\thresh_{\max} \coloneqq \max(\s) -g'\left(\nicefrac{1}{d}\right).\end{equation}
\end{proposition}
The equation of $\p(\tau)$ can be seen as a generalization of
\eqref{eq:sparsemax_thresholded}.
An approximate $\thresh$ such that $|\phi(\thresh)| \le \epsilon$ can be found in
$O(\nicefrac{1}{\log \epsilon})$ time by, e.g., bisection.
The related problem of Bregman projection
onto the probability simplex was recently studied by \citet{bregmanproj} but our
derivation is different and more direct (cf.\
\S\ref{appendix:proof_root_finding}).

For example, when $\Omega$ is the negative $\alpha$-Tsallis entropy
$-\HHt_\alpha$, which we saw can be written in separable form in
\eqref{eq:tsallis_separable}, we obtain
\[g(t) = 
\frac{t^\alpha - t}{\alpha(\alpha-1)},
\quad
g'(t) = \frac{t^{\alpha-1} - 1}{\alpha-1}
\quad \text{and} \quad
(g')^{-1}(s) = \big(1+(\alpha-1) s\big)^\frac{1}{\alpha-1},
\]
yielding
\[\p(\thresh) = \big(1 + (\alpha-1) 
\max(\s - \thresh, \nicefrac{-1}{\alpha-1})\big)^{\frac{1}{\alpha-1}}.\]
From the root $\thresh^\star$ of
$\phi(\thresh) = \DP{\p(\thresh)}{\ones} - 1$, we obtain
$\widehat{\y}_{-\HHt_{\alpha}}(\s) = \p(\thresh^\star)$.

\subsection{Proximity operators}
\label{sec:computing_prox}

Computing the proximity operator $\prox_{\tau \Omega}(\bm{\eta})$, defined in
\eqref{eq:prox}, usually involves a more challenging optimization problem than
$\yHatOmega(\s)$.  For instance, when $\Omega$ is Shannon's negative entropy
over $\triangle^d$,
$\yHatOmega$ enjoys a closed-form solution (the softmax) but not $\prox_\Omega$.
However, we can always compute $\prox_{\tau \Omega}(\bm{\eta})$ by first-order
gradient methods given access
to $\yHatOmega$. Indeed, using the Moreau decomposition, we have
\begin{equation}
    \prox_{\tau \Omega}(\bm{\eta}) = \bm{\eta} - \prox_{\frac{\Omega^*}{
    \tau}}(\bm{\eta} / \tau).
\label{eq:Moreau_decomposition}
\end{equation}
Since $\dom(\Omega^*) = \RR^d$, computing
$\prox_{\frac{\Omega^*}{\tau}}(\bm{\eta} / \tau)$ only involves an unconstrained
optimization problem, $\argmin_{\s \in \RR^d} \frac{1}{2} \|\s -
\bm{\eta}/\tau\|^2 + \frac{1}{\tau} \Omega^*(\s)$.
Since $\nabla \Omega^* = \yHatOmega$, that optimization problem can easily be
solved by any first-order gradient method
given access to $\yHatOmega$.
For specific choices of $\Omega$, $\prox_{\tau \Omega}(\bm{\eta})$
can be computed directly more efficiently, as we now discuss.

\paragraph{Closed-form expressions.}

We now give examples of commonly-used loss functions for which
$\prox_{\tau \Omega}$ enjoys a closed-form expression.

For the squared loss,
we choose $\Omega(\bmu)=\frac{1}{2}\|\bmu\|^2$.  Hence:
\begin{equation}
    \prox_{\tau \Omega}(\bm{\eta})
= \argmin_{\bmu \in \RR^d}
\frac{1}{2} \|\bmu - \bm{\eta} \|^2 + \frac{\tau}{2} \|\bmu\|^2
= \frac{\bm{\eta}}{\tau + 1}.
\end{equation}
For the perceptron loss \citep{perceptron,structured_perceptron},
we choose $\Omega=I_{\triangle^d}$. Hence:
\begin{equation}
    \prox_{\tau \Omega}(\bm{\eta})
= \argmin_{\p \in \triangle^d}
\frac{1}{2} \|\p - \bm{\eta} \|^2.
\end{equation}
For the sparsemax loss \citep{sparsemax}, we choose
$\Omega = \frac{1}{2} \|\cdot\|^2 + I_{\triangle^d}$. Hence:
\begin{equation}
    \prox_{\tau \Omega}(\bm{\eta})
= \argmin_{\p \in \triangle^d}
\frac{1}{2} \|\p - \bm{\eta} \|^2 + \frac{\tau}{2} \|\p\|^2
= \argmin_{\p \in \triangle^d}
\frac{1}{2} \left\|\p - \frac{\bm{\eta}}{\tau + 1} \right\|^2.
\end{equation}
For the cost-sensitive multiclass hinge loss,
we choose
$\Omega = I_{\triangle^d} -\langle \cdot, \cc_{\y} \rangle$. Hence:
\begin{equation}
    \prox_{\tau \Omega}(\bm{\eta})
= \argmin_{\p \in \triangle^d}
\frac{1}{2} \|\p - \bm{\eta} \|^2 - \tau \langle \p, \cc_{\y} \rangle
= \argmin_{\p \in \triangle^d}
\frac{1}{2} \|\p - (\bm{\eta} + \tau \cc_{\y}) \|^2.
\end{equation}
Choosing $\cc_{\y} = \ones - \y$, where $\y \in
\{\bm{e}_i\}_{i=1}^d$ is the ground-truth label, gives the proximity operator
for the usual multiclass hinge loss \citep{multiclass_svm}.

For the Shannon entropy and $1.5$-Tsallis entropy, we show that $\prox_{-\HH}$
reduces to root finding in Appendix \ref{appendix:derivation_prox}.

\section{Experiments}
\label{sec:experiments}

In this section, we demonstrate one of the key features of Fenchel-Young losses:
their ability to induce \textbf{sparse} probability distributions. We focus on
two tasks: label proportion estimation (\S\ref{sec:exp_label_prop}) and
dependency parsing (\S\ref{sec:exp_dependency_parsing}).

\subsection{Label proportion estimation experiments}
\label{sec:exp_label_prop}

As we saw in \S\ref{sec:proba_clf}, $\alpha$-Tsallis entropies generate a family
of losses, with the logistic ($\alpha \to 1$) and sparsemax losses ($\alpha=2$)
as important special cases.
In addition, they are
twice differentiable for $\alpha \in
[1,2)$, produce sparse probability distributions for $\alpha > 1$ and are
computationally efficient for any $\alpha \ge 1$, thanks to Proposition~
\ref{prop:root_finding}. In this section, we demonstrate their effectiveness on
the task of label proportion estimation and compare different solvers for
computing the regularized prediction function $\widehat{\y}_{-\HHt_\alpha}$.

\paragraph{Experimental setup.}

Given an input vector $\x \in \cX \subseteq \RR^p$, where $p$ is the number of
features, our goal is to estimate a vector of label proportions $\y \in
\triangle^d$, where $d$ is the number of classes.  If $\y$ is sparse, we expect
the superiority of Tsallis losses over the conventional logistic loss on this
task.  At training time, given a set of $n$ pairs $(\x_i,\y_i)$, we estimate a
matrix $W \in \RR^{d \times p}$ by minimizing the convex objective
\begin{equation}
R(W) \coloneqq 
\sum_{i=1}^n L_\Omega(W \x_i; \y_i) + \frac{\lambda}{2} \|W\|_F^2.
\end{equation}
We optimize the loss using L-BFGS \citep{lbfgs}.
From Proposition \ref{prop:fy_losses} and using the chain rule, we obtain the
gradient expression
$\nabla R(W) = (\widehat{Y}_\Omega - Y)^\top X + \lambda W$,
where $\widehat{Y}_\Omega$, $Y$ and $X$ are matrices whose rows
gather $\yHatOmega(W \x_i)$, $\y_i$ and $\x_i$, for $i=1,\dots,n$.
At test time, we predict label proportions by 
$\p = \widehat{\y}_{-\HHt_\alpha}(W \x)$.

\paragraph{Real data experiments.}

We ran experiments on $7$ standard multi-label benchmark datasets --- see
Table \ref{tab:datasets} for dataset characteristics\footnote{The datasets can
be downloaded from \url{http://mulan.sourceforge.net/datasets-mlc.html} and
\url{https://www.csie.ntu.edu.tw/~cjlin/libsvmtools/datasets/}.}.
For all datasets, we removed
samples with no label, normalized samples to have zero mean unit variance, and
normalized labels to lie in the probability simplex.  We chose $\lambda \in
\{10^{-4},10^{-3},\dots,10^4\}$ and $\alpha \in \{1,1.1,\dots,2\}$ against the
validation set.  We report the test set mean Jensen-Shannon divergence,
$\text{JS}(\p, \y) \coloneqq \nicefrac{1}{2}\operatorname{KL}(\p \| \frac{\p+\y}{2}) +
\nicefrac{1}{2}\operatorname{KL}(\y \| \frac{\p+\y}{2})$, and the mean squared error
$\nicefrac{1}{2}\,\|\p - \y\|^2$ in Table \ref{tab:label_proportion}. 
As can be seen, the loss with tuned $\alpha$ achieves the best averaged rank
overall.  Tuning $\alpha$ allows to choose the best loss in the family in a
data-driven fashion. 

\begin{table}[t]
    \caption{Dataset statistics}
    \small
    \centering
    \begin{tabular}{r c c c c c c c}
        \toprule
        Dataset & Type & Train & Dev & Test & Features &
        Classes & Avg.\ labels \\
        \midrule
        Birds & Audio & 134 & 45 & 172 & 260 & 19 & 2 \\
        Cal500 & Music & 376 & 126 & 101 & 68 & 174 & 25 \\
        Emotions & Music & 293 & 98 & 202 & 72 & 6 & 2 \\
        Mediamill & Video & 22,353 & 7,451 & 12,373 & 120 & 101 & 5 \\
        Scene & Images & 908 & 303 & 1,196 & 294 & 6 & 1\\
        SIAM TMC & Text & 16,139 & 5,380 & 7,077 & 30,438 & 22 & 2\\
        Yeast & Micro-array & 1,125 & 375 & 917 & 103 & 14 & 4\\
        \bottomrule
    \end{tabular}
    \label{tab:datasets}
\end{table}

\paragraph{Synthetic data experiments.}

We follow \citet{sparsemax} and generate a document
$\x \in \RR^p$ from a mixture of multinomials and label proportions $\y \in
\triangle^d$ from a multinomial. The number of words in $\x$ and labels in $\y$
is sampled from a Poisson distribution --- see \citet{sparsemax} for a precise
description of the generative process.  We use 1200 samples as training set, 200
samples as validation set and 1000 samples as test set.
We tune $\lambda \in \{10^{-6}, 10^{-5}, \dots, 10^0\}$ 
and $\alpha \in \{1.0, 1.1, \dots, 2.0\}$ against the validation set.
We report the Jensen-Shannon divergence in Figure
\ref{fig:label_proportions_synth}.  Results using the mean squared error (MSE)
were entirely similar.
When the number of classes is 10, Tsallis and sparsemax losses
perform almost exactly the same, both outperforming softmax.
When the number of classes is 50, Tsallis losses outperform both sparsemax and
softmax.

\begin{table}[t]
    \caption{\textbf{Test-set performance of Tsallis losses for
        various $\alpha$ on the task of sparse label proportion estimation:}
        average Jensen-Shannon divergence (left) and mean squared error (right).
        Lower is better.}
\centering
\small
    \begin{tabular}{r c c c c}
        \toprule
        & $\alpha=1$ (logistic) & $\alpha=1.5$ & $\alpha=2$ (sparsemax) & tuned
        $\alpha$ \\
        \midrule
Birds & 0.359 / 0.530 & 0.364 / 0.504 & 0.364 / 0.504 & {\bf 0.358} / {\bf 0.501} \\
Cal500 & 0.454 / {\bf 0.034} & 0.456 / 0.035 & {\bf 0.452} / 0.035 & 0.456 / {\bf 0.034} \\
Emotions & 0.226 / 0.327 & 0.225 / {\bf 0.317} & 0.225 / {\bf 0.317} & {\bf 0.224} / 0.321 \\
Mediamill & 0.375 / 0.208 & 0.363 / 0.193 & {\bf 0.356} / {\bf 0.191} & 0.361 / 0.193 \\
Scene & {\bf 0.175} / {\bf 0.344} & 0.176 / 0.363 & 0.176 / 0.363 & {\bf 0.175} / 0.345 \\
TMC & 0.225 / 0.337 & 0.224 / {\bf 0.327} & 0.224 / {\bf 0.327} & {\bf 0.217} / 0.328 \\
Yeast & {\bf 0.307} / {\bf 0.183} & 0.314 / 0.186 & 0.314 / 0.186 & {\bf 0.307} / {\bf 0.183} \\
\midrule
Avg.\ rank & 2.57 / 2.71 & 2.71 / 2.14 & 2.14 / 2.00 & \textbf{1.43} /
\textbf{1.86} \\
        \bottomrule
    \end{tabular}
\label{tab:label_proportion}
\end{table}

\paragraph{Solver comparison.}

\begin{wrapfigure}{r}{55mm}
\centering%
\vspace{-5mm}
\input{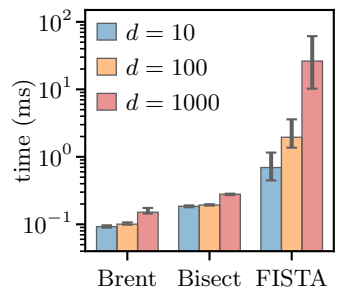}
\caption{Solver comparison.}
\label{fig:solver_comparison}
\end{wrapfigure}
We also compared bisection (binary search) and Brent's method for solving
\eqref{eq:prediction} by root finding (Proposition~\ref{prop:root_finding}).
We focus on $\HHt_{1.5}$, i.e.\ the $1.5$-Tsallis entropy, and also
compare against using a generic projected gradient algorithm
(FISTA) to solve \eqref{eq:prediction} naively.
We measure the time needed 
to reach a solution $\p$ with $\| \p - \p^\star \|_2 < 10^{-5}$,
over 200 samples $\s \in \RR^d \sim \mathcal{N}(\bm{0}, \sigma
\bm{I})$ with $\log\sigma\sim\mathcal{U}(-4, 4)$.
Median and 99\% CI times reported in 
Figure~\ref{fig:solver_comparison}
reveal that root finding scales better,
with Brent's method outperforming FISTA by one to two orders of magnitude.

\begin{figure}[t]
\centering
\input{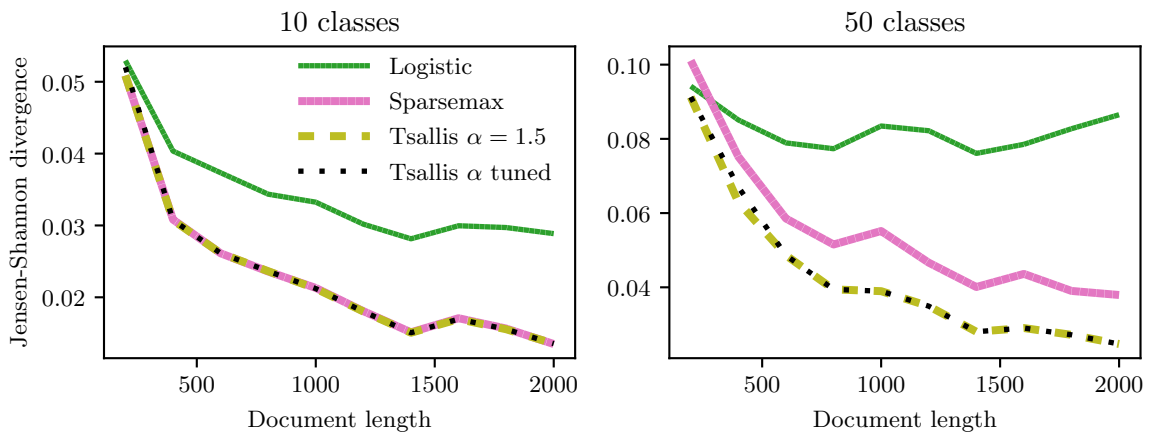}
\caption{Jensen-Shannon divergence between predicted and true label
    proportions, when varying document length, of various losses generated by a
Tsallis entropy.}
\label{fig:label_proportions_synth}
\end{figure}

\subsection{Dependency parsing experiments}
\label{sec:exp_dependency_parsing}

We next compare Fenchel-Young losses for structured
output prediction, as formulated in \S\ref{sec:structured_prediction},
for non-projective {\em dependency parsing}.
The application consists of predicting the directed tree of grammatical
dependencies between words in a sentence~\citep[Chapter 14]{jurafskymartin}.
We tackle it by learning structured models over the arborescence
polytope, using the structured SVM, CRF, and SparseMAP losses \citep{sparsemap},
which we saw to all be instances of Fenchel-Young loss.
Training data,
as in the CoNLL 2017 shared task \citep{conll17}, comes from 
the Universal Dependency project \citep{ud}.
To isolate the effect of the loss, we use the provided gold
tokenization and part-of-speech tags.
We follow closely the bidirectional LSTM arc-factored parser of
\citet{kg}, using the same model configuration;
the only exception is not using externally pretrained embeddings.
Parameters are trained using Adam \citep{adam}, tuning the
learning rate on the grid $\{.5, 1, 2, 4, 8\} \times 10^{-3}$,
expanded by a factor of 2 if the best model is at either end.

We experiment with 5 languages, diverse both in terms of family and in
terms of the amount of training data (ranging from 1,400 sentences
for Vietnamese to 12,525 for English). 
Test set results (Table~\ref{tab:parse}) indicate that SparseMAP outperforms
the CRF loss and is competitive with the structured SVM loss,
outperforming it substantially on Chinese. 
The Fenchel-Young perspective sheds further light onto the empirical 
results, as the SparseMAP loss and the structured SVM (unlike the CRF) both enjoy
the margin property with margin 1, while both SparseMAP and the CRF loss
(unlike structured SVM) are differentiable.

\begin{table}[t]
    \caption{\label{tab:parse}Unlabeled attachment accuracy scores for
        dependency parsing, using a bi-LSTM model~\citep{kg}.
        For context, we include the scores of the
        CoNLL 2017 UDPipe baseline \citep{udpipe}.}

    \centering
    \begin{tabular}{r c c c c c}
        \toprule
        Loss & en & zh & vi & ro & ja \\
        \midrule
        Structured SVM &     87.02 &      81.94 &     69.42 &{\bf 87.58}&{\bf 96.24}\\
        CRF            &     86.74 &      83.18 &     69.10 &     87.13 &     96.09 \\
        SparseMAP      &     86.90 & {\bf 84.03}&{\bf 69.71}&     87.35 &     96.04 \\
        \midrule
        UDPipe baseline& {\bf 87.68} & 82.14 & 69.63 & 87.36 & 95.94 \\
        \bottomrule
    \end{tabular}
\end{table}

\section{Related work}
\label{sec:related_work}

\subsection{Loss functions}
\label{sec:proper_losses}

\paragraph{Proper losses.}

Proper losses, a.k.a. proper scoring rules, are a well-studied object 
in statistics
(\citet{grunwald_2004}; \citet{gneiting_2007}; references therein)
and more recently in machine learning \citep{reid_composite_binary,vernet_2016}.
A loss $\ell(\p; \y)$, between a ground truth
$\y \in \triangle^d$ and a probability forecast $\p \in
\triangle^d$, where $d$ is the number of classes,
is said to be proper if it is minimized when
estimating the true $\y$. Formally, 
a proper loss $\ell$ satisfies
\begin{equation}
\ell(\y; \y) \le \ell(\p; \y) \quad \forall \p, \y \in \triangle^d.
\end{equation}
It is strictly proper if the inequality is strict when $\p \neq \y$,
implying that it is uniquely minimized by predicting the correct probability.
Strictly proper losses induce \textbf{Fisher consistent} estimators of
probabilities \citep{vernet_2016}.
%
A key result, which dates back to \citet{savage_1971} (see also
\citet{gneiting_2007}), is that given a regularization function $\Omega \colon
\triangle^d \to \RR$, the function $\ell_\Omega \colon \triangle^d \times
\triangle^d \to \RR_+$ defined by
\begin{equation}
    \ell_\Omega(\p; \y) 
\coloneqq
\langle \nabla \Omega(\p), \y - \p \rangle - \Omega(\p)
\label{eq:S_Omega}
\end{equation}
is proper. 
It is easy to see that
$\ell_\Omega(\p; \y) 
= B_\Omega(\y\|\p) - \Omega(\y)$,
which recovers
the well-known  
Bregman divergence representation of proper losses.
For example, using the Gini index $\HH(\p) = 1 - \|\p\|^2$
generates the \textbf{Brier score} \citep{brier_1950}
\begin{equation}
\ell_{-\HH}(\p; \e_k) = \sum_{i=1}^d ([[k = i]] - p_i)^2,
\end{equation}
showing that the sparsemax loss and the Brier score share the same generating
function.

More generally, while a proper loss $\ell_\Omega$ is related to a \textbf{primal-space} Bregman
divergence, a Fenchel-Young loss $L_\Omega$ can be seen as a
\textbf{mixed-space} Bregman divergence (\S\ref{sec:fy_losses}). 
This difference has a number of
important consequences. First,
$\ell_\Omega$ is \textbf{not necessarily convex} in $\p$
(\citet[Proposition 17]{vernet_2016} show that it is in fact quasi-convex).
In contrast,
$L_\Omega$ \textbf{is always} convex in $\s$. Second, the first argument is
\textbf{constrained} to $\triangle^d$ for $\ell_\Omega$, while
\textbf{unconstrained} for $L_\Omega$.

In practice, proper losses are often
composed with an \textbf{invertible} link function $\bpsi^{-1} \colon \RR^d \to
\triangle^d$. This form of a loss, $\ell_\Omega(\bpsi^{-1}(\s); \y)$, is sometimes
called composite \citep{buja_2005,reid_composite_binary,vernet_2016}.  
However, the composition of $\ell_\Omega(\cdot; \y)$ and $\bpsi^{-1}(\s)$ is not
necessarily convex in $\s$.  The \textbf{canonical link function}
\citep{buja_2005} of $\ell_\Omega$ is a link function that ensures the convexity of
$\ell_\Omega(\bpsi^{-1}(\s); \y)$ in $\s$. It also plays a key role in
generalized linear models \citep{glm,mccullagh_1989}.  
Following Proposition \ref{prop:Bregman_div},
when $\Omega$ is Legendre type, we obtain
\begin{equation}
L_\Omega(\s; \y) = B_\Omega(\y \| \yHatOmega(\s)) = \ell_\Omega(\yHatOmega(\s);
\y) + \Omega(\y).
\end{equation}
Thus, in this case, Fenchel-Young losses and proper
composite losses coincide up to the constant term $\Omega(\y)$ (which vanishes
if $\y = \e_k$ and $\Omega$ satisfies assumption A.1), with $\bpsi^{-1} =
\yHatOmega$ the canonical inverse link function.  Fenchel-Young losses, however,
require neither invertible link nor Legendre type assumptions, allowing to
express losses (e.g., hinge or sparsemax) that are not expressible in composite
form.  Moreover, as seen in \S\ref{sec:margin}, a Legendre-type $\Omega$
precisely precludes sparse probability distributions and losses enjoying a
margin.
However, the decoupling between loss and link as promoted by proper composite
losses also has some merits \citep{reid_composite_binary}. For instance, using
a non-canonical link is useful to express the exponential loss of Adaboost
\citep{nowak_2019}.

\paragraph{Other related losses.} 

\cite{nock_2009} proposed a binary classification loss construction.
Technically, their loss is based on the Legendre transformation (a subset of
Fenchel conjugate functions), precluding non-invertible mappings.
\citet{masnadi_2011} studied the Bayes consistency of related binary
classification loss functions.
\citet[Proposition 3]{duchi_2016} derived the multi-class loss
\eqref{eq:fy_loss_multiclass2}, a special case of Fenchel-Young loss over the
probability simplex, and showed (Proposition 4) that any strictly concave
generalized entropy generates a classification-calibrated loss. 
\citet{amid_2017} proposed a different family of losses based on the Tsallis
divergence, to interpolate between convex and non-convex losses, for
robustness to label noise. 
Finally, several works have explored connections between 
another divergence, the $f$-divergence \citep{csiszar_1975}, and 
surrogate loss functions \citep{nguyen_2009,reid_2011,garcia_2012,duchi_2016}.

\paragraph{Consistency.}

In this paper, we have not addressed the question of the consistency of
Fenchel-Young losses when used as a surrogate for a (possibly non-convex) loss.
Although consistency has been widely studied in the multiclass setting
\citep{zhang_2004,bartlett_2006,tewari_2007,mroueh_2012} and in other specific
settings \citep{duchi_2010,ravikumar_2011}, it is only recently that it was
studied for general losses in a fully general structured prediction setting.
Since the publication of this paper, sufficient conditions for consistency have
been established for composite losses \citep{nowak_2019} and for
projection-based losses \citep{blondel_2019}, both a subset of Fenchel-Young
losses. As shown in these works, the strong convexity of $\Omega$ plays a
crucial role.

\subsection{Fenchel duality in machine learning}

In this paper, we make extensive use of Fenchel duality to derive loss
functions.  Fenchel duality has also played a key role in several machine learning
studies before. It was
used to provide a unifying perspective on convex empirical risk minimization and
representer theorems \citep{rifkin_2007}.  It was also used for deriving regret
bounds \citep{shalev_2007,shalev_2009}, risk bounds \citep{kakade_2009}, and
unified analyses of boosting \citep{shen_2010,shalev_boosting,telgarsky_2012}.
However, none of these works propose loss functions, as we do.  More closely
related to our work are smoothing techniques \citep{nesterov_smooth,beck_2012},
which have been used extensively to create smoothed versions of existing losses
\citep{song_2014,accelerated_sdca}. However, these techniques were applied on a
per-loss basis and were not connected to an induced probability distribution. In
contrast, we propose a generic loss construction, with clear links between
smoothing / regularization and the probability distribution produced by
regularized prediction functions.

\subsection{Approximate inference with conditional gradient algorithms}

In this paper, we suggest conditional gradient (CG) algorithms as a powerful
tool for computing regularized prediction functions and Fenchel-Young losses
over potentially complex output domains.  CG algorithms have been used before to
compute an approximate solution to intractable marginal inference problems
\citep{belanger2013marginal,barrierfw} or to sample from an intractable
distribution so as to approximate its mean \citep{bach_herding,lacoste_2015}.
When combined with reproducing kernel Hilbert spaces (RKHS), these ideas are
closely related maximum mean discrepancy \citep{gretton_2012}.

\section{Conclusion}

We showed that the notion of output regularization and Fenchel duality provide
simple core principles, unifying many existing loss functions, and allowing to
create useful new ones easily, on a large spectrum of tasks.  We characterized a
tight connection between sparse distributions and losses with a separation
margin, and showed that these losses are precisely the ones that cannot be
written in proper composite loss form.  We established the computational tools
to efficiently learn with Fenchel-Young losses, whether in unstructured or
structured settings. We expect that this groundwork will enable the creation of
many more novel losses in the future, by exploring other convex polytopes and
regularizers.

\acks{MB thanks Arthur Mensch, Gabriel Peyr\'{e} and Marco Cuturi
for numerous fruitful discussions and Tim Vieira for
introducing him to generalized exponential families. 
This work was 
supported by the European Research Council (ERC
StG DeepSPIN 758969) 
and by the Funda\c{c}\~ao para a Ci\^encia e Tecnologia through contracts UID/EEA/50008/2013 and CMUPERI/TIC/0046/2014 ~(GoLocal).}


\appendix

\section{Additional materials}

\subsection{Computing proximity operators by root finding}
\label{appendix:derivation_prox}

In this section, we first show how the proximity operator of uniformly separable
generalized entropies reduces to unidimensional root finding. We then
illustrate the case of two important entropies.
The proximity operator is defined as
\begin{align}
\prox_{-\frac{1}{\sigma}\HH}(\x) 
&= \argmin_{\p \in \triangle^d}
\frac{\sigma}{2} \|\p-\x\|^2 - \HH(\p) \\
&= \argmax_{\p \in \triangle^d}
\DP{\p}{\x} + \frac{1}{\sigma} \HH(\p) - \frac{1}{2}\|\p\|^2 \\
&= \argmax_{\p \in \triangle^d} \DP{\p}{\x}
- \sum_{i=1}^d g(p_i),
\end{align}
where
\begin{equation}
g(t) \coloneqq \frac{1}{2} t^2 - \frac{1}{\sigma} h(t).
\end{equation}
Note that $g(t)$ is strictly convex even when $h$ is only concave.
We may thus apply Proposition~\ref{prop:root_finding} to compute
$\prox_{-\frac{1}{\sigma}\HH}$.
The ingredients necessary for using Algorithm~\ref{algo:bisect}
are $g'(t) = t - \frac{1}{\sigma}h'(t)$, and its inverse
$(g')^{-1}(x)$. We next derive closed-form expressions for the
inverse for two important entropies.

\paragraph{Shannon entropy.} We have 
\[
g(t)\coloneqq 
\frac{t^2}{2}+ \frac{t \log t}{\sigma}
\quad \text{and} \quad
g'(t) = t + \frac{1+\log t}{\sigma}.
\]
To compute the inverse,
$\left(g'\right)^{-1}(x) = t$,
we solve for $t$ satisfying
\begin{equation}
\frac{1+\log t}{\sigma} + t = x.
\end{equation}
Making the change of variable $s\coloneqq\sigma t$ gives 
\begin{equation}
s + \log s = \sigma x - 1 + \log\sigma 
\iff
s = \omega(\sigma x - 1 + \log \sigma).
\end{equation}
Therefore, $\left(g'\right)^{-1}(x) = 
\frac{1}{\sigma} \omega(\sigma x - 1 + \log \sigma),$
where $\omega$ denotes the Wright omega function.

\paragraph{Tsallis entropy with {\boldmath $\alpha=1.5$}.}
Up to a constant term, we have
\[
g(t)
=
\frac{t^2}{2} +\frac{4}{3\sigma}t^{\nicefrac{3}{2}}
\quad \text{and} \quad
g'(t)
=
t + \frac{2}{\sigma} \sqrt{t}.
\]
To obtain the inverse,
$\left(g'\right)^{-1}(x) = t$
we seek $t$ satisfying
\begin{equation}
    \frac{2}{\sigma} \sqrt{t} + t = x \iff 
\sqrt{t} = \frac{\sigma}{2} \left(x-t\right). 
\end{equation}
If $t > x$ there are no solutions, otherwise we may square both sides, yielding
\begin{equation}
\frac{\sigma^2t^2}{4} - t\left(\frac{\sigma^2x}{2}+1\right)
+\frac{\sigma^2x^2}{4} = 0.
\end{equation}
The discriminant is
\begin{equation}
\Delta=1+\sigma^2x > 0,
\end{equation}
resulting in two solutions
\begin{equation}
t_\pm = x + \frac{2}{\sigma^2}\left(1 \pm \sqrt{1+\sigma^2x}\right).
\end{equation}
However, $t_{+} > x$, therefore
\begin{equation}
\left(g'\right)^{-1}(x) = 
x + \frac{2}{\sigma^2}\left(1 - \sqrt{1+\sigma^2x}\right).
\end{equation}

\subsection{Loss ``Fenchel-Youngization''}
\label{sec:loss_fenchel_yougization}

Not all loss functions proposed in the literature can be written in
Fenchel-Young loss form. In this section, we present a natural method to
approximate (in a sense we will clarify) any loss $\ell \colon \RR^d \times
\{\e_i\}_{i=1}^d \to \RR_+$ with a Fenchel-Young loss. This has two main
advantages. First, the resulting loss is convex even when $\ell$ is not.
Second, the associated regularized prediction function can be used for
probabilistic prediction even if $\ell$ is not probabilistic.

\paragraph{From loss to entropy.}

Equations \eqref{eq:expected_loss_simplex} and \eqref{eq:Bayes_risk_simplex},
which relate entropy and conditional Bayes risk,
suggest a reverse construction from entropy to loss function. This
direction has been explored in previous works
\citep{grunwald_2004,duchi_2016}. The entropy
$\HH_\ell$ generated from $\ell$ is defined as follows:
\begin{equation}
\HH_\ell(\p) 
\coloneqq \inf_{\s \in \RR^d} \EE_{\p}[\ell(\s; Y)]
= \inf_{\s \in \RR^d} \sum_{i=1}^d p_i \ell(\s;\e_i).
\label{eq:H_L}
\end{equation}
This is the infimum of a linear and thus concave function of $\p$. 
Hence, by
Danskin's theorem \citep[Proposition B.25]{danskin_theorem,bertsekas_book},
$H_\ell(\p)$ is concave (note that this is true even if $\ell$ is not convex).
As an example, \citet[Example 1]{duchi_2016} show that the entropy associated
with the zero-one loss is $\HH_\ell(\p) = 1 - \max_j p_j$.
As we discussed in \S\ref{sec:examples_entropy},
this entropy is known as the Berger-Parker dominance
index  \citep{Berger1970}, and is a special case of norm entropy
\eqref{eq:norm_entropy}.

The following new proposition shows how to compute the entropy
of pairwise hinge losses. 
\vspace{0.5em}
\begin{proposition}{Entropy generated by pairwise hinge losses}
    
Let $\ell(\s; \e_k) \coloneqq \sum_{j \neq k} \phi(\ss_j - \ss_k)$
and $\tau(\p) \coloneqq \min_j 1 - p_j$.
Then, for all $\p \in \triangle^d$:
\begin{enumerate}[topsep=0pt,itemsep=3pt,parsep=3pt,leftmargin=17pt]

\item If $\phi(t)$ is the hinge function $[1 + t]_+$ 
    then $\HH_\ell(\p) = \tau(\p) d$.

\item If $\phi(t)$ is the smoothed hinge function
    \eqref{eq:smoothed_hinge}, then
$\HH_\ell(\p) 
= \frac{-\tau(\p)^2}{2} \sum_{j=1}^d \frac{1}{1 - p_j} + \tau(\p) d$.

\item If $\phi(t)$ is the squared hinge function $\frac{1}{2} [1 + t]^2_+$,
then $\HH_\ell(\p) = \frac{\frac{1}{2} d^2}{\sum_{j=1}^d 1/(1-p_j)}$.

\end{enumerate}
\label{prop:pairwise_entropies}
\end{proposition}

See Appendix \ref{appendix:proof_pairwise_entropies} for a proof.
The first one recovers \cite[Example 5]{duchi_2016} with a simpler proof. The
last two are new.
These entropies are illustrated in Figure \ref{fig:pairwise_ent}.
For the last two choices, $\HH_\ell$ is strictly concave over $\triangle^d$ and the
associated prediction function $\widehat{\y}_{-\HH_\ell}(\s)$ is typically
sparse, although we are not aware of a closed-form expression.

\begin{figure}[t]
\centering
\includegraphics[width=\textwidth]{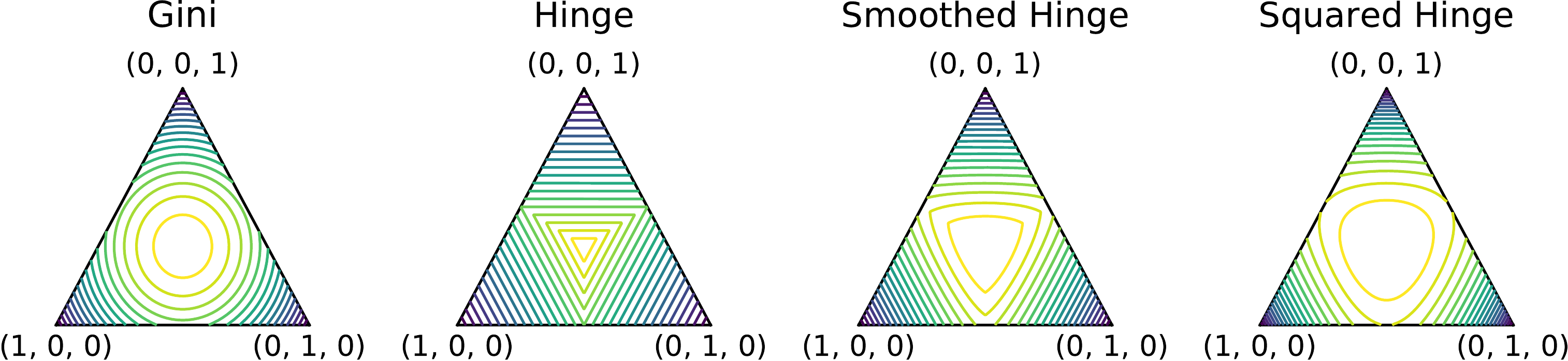}
\caption{Contour comparison of Gini index $\HHt_2$ and pairwise hinge entropies
$\HH_\ell$.}
\label{fig:pairwise_ent}
\end{figure}

\paragraph{And back to Fenchel-Young loss.}

A natural idea is then to construct a Fenchel-Young from $\HH_\ell$.
The resulting loss is convex by construction.
Let us assume for simplicity that $\ell$ achieves its minimum at $0$, 
implying $\HH_\ell(\e_k)=0$.
Combining \eqref{eq:fy_loss_multiclass2} and \eqref{eq:H_L}, we obtain
\begin{align}
L_{-\HH_\ell}(\s; \e_k)
&= (-\HH_\ell)^*(\s) - \theta_k \\
&= \max_{\p \in \triangle^d} \DP{\s}{\p} + \HH_\ell(\p) - \theta_k \\
&= \max_{\p \in \triangle^d} \min_{\s' \in \RR^d} \sum_{i=1}^d p_i (\ell(\s';
\e_i) + \theta_i) - \theta_k.
\end{align}
By weak duality, we have for every (possibly non-convex) $\ell$
\begin{align}
L_{-\HH_\ell}(\s; \e_k)
&\le  \min_{\s' \in \RR^d} \max_{\p \in \triangle^d} \sum_{i=1}^d p_i (\ell(\s';
\e_i) + \theta_i) - \theta_k \\
&= \min_{\s' \in \RR^d} \max_{i \in [d]} \ell(\s'; \e_i) + \theta_i - \theta_k.
\end{align}
We can see $(\s; \e_k) \mapsto \max_{i \in [d]} \ell(\s'; \e_i) + \theta_i -
\theta_k$ as a family of cost-augmented hinge losses \citep{structured_hinge}
parametrized by $\s'$. Hence, $L_{-\HH_\ell}(\s; \e_k)$ is upper-bounded by the
tightest loss in this family.

If $\ell$ is convex, strong duality holds and we obtain equality
\begin{align}
L_{-\HH_\ell}(\s; \e_k)
= \min_{\s' \in \RR^d} \max_{i \in [d]} \ell(\s'; \e_i) + \theta_i - \theta_k.
\end{align}

\paragraph{Margin.} 

We can easily upper-bound the margin of $L_{-\HH_\ell}$.
Indeed, assuming $\ell(\s; \e_i)$ is upper-bounded by $m$ for all $\s$ and $i
\in [d]$, we get
\begin{equation}
0 \le L_{-\HH_\ell}(\s; \e_k)
\le \max_{i \in [d]} \ell(\s; \e_i) + \theta_i - \theta_k
\le \max_{i \in [d]} m + \theta_i - \theta_k.
\end{equation}
Hence, from Definition \ref{def:margin}, we have $\mathrm{margin}(L_{-\HH_\ell})
\le m$.  Note that if
$\mathrm{margin}(L_{-\HH_\ell}) \le \mathrm{margin}(\ell)$,
then $L_{-\HH_\ell}(\s; \e_k) = 0 \Rightarrow \ell(\s; \e_k) = 0$.

\section{Proofs}

\subsection{Proof of Proposition \ref{prop:prediction_func}}
\label{appendix:proof_prediction_func}

\paragraph{Effect of a permutation.}

Let $\Omega(\bmu)$ be symmetric.  We first prove that $\Omega^*$ is
symmetric as well. Indeed, we have
\begin{equation}
\Omega^*(\bm{P}\s) 
= \sup_{\bmu \in \dom(\Omega)} (\bm{P} \s)^\top \bmu - \Omega(\bmu) 
= \sup_{\bmu \in \dom(\Omega)} \s^\top \bm{P}^\top
\bmu - \Omega(\bm{P}^\top \bmu) = \Omega^*(\s).  
\end{equation}
The last equality was obtained by
a change of variable $\bmu' = \bm{P}^\top \bmu$, from which $\bmu$ is
recovered as $\bmu = \bm{P}\bmu'$, which proves $\nabla \Omega^*(\bm{P} \bmu)
= \bm{P} \nabla \Omega^*(\bmu)$.  

\paragraph{Order preservation.}

Since $\Omega^*$ is convex, the gradient operator $\nabla
\Omega^*$ is monotone, i.e.,
\begin{equation}
(\s' - \s)^{\top} (\bmu'-\bmu) \ge 0
\end{equation}
for any
$\s, \s' \in \RR^d$, $\bmu=\nabla \Omega^*(\s)$ and $\bmu'=\nabla
\Omega^*(\s')$.  Let $\s'$ be obtained from $\s$ by swapping two coordinates,
i.e., $\ss_j'=\ss_i$, $\ss_i'=\ss_j$, and $\ss_k'=\ss_k$ for any $k \notin
\{i,j\}$. Then, since $\Omega$ is symmetric, we obtain:
\begin{eqnarray}
2(\ss_j - \ss_i)(\mu_j-\mu_i)\ge 0,
\end{eqnarray}
which implies $\ss_i > \ss_j \Rightarrow \mu_i \ge \mu_j$ and $\mu_i > \mu_j
\Rightarrow \ss_i \ge \ss_j$. To fully prove the claim, we need to show that the last
inequality is strict: to do this, we simply invoke
$\nabla \Omega^*(\bm{P} \bmu) = \bm{P} \nabla \Omega^*(\bmu)$
with a matrix $\bm{P}$ that permutes $i$ and $j$, from which we must have $\ss_i
= \ss_j \Rightarrow \mu_i = \mu_j$. 

\paragraph{Approximation error.} 

Let $\y^\star \coloneqq \argmax_{\y \in \cY}
\DP{\s}{\y}$ and $\bmu^\star \coloneqq \argmax_{\bmu \in \dom(\Omega)}
\DP{\s}{\bmu} - \Omega(\bmu)$. If $f$ is a differentiable
$\gamma$-strongly convex function with unique minimizer $\bmu^\star$, then
\begin{equation}
\frac{\gamma}{2} \|\y - \bmu^\star\|^2_2 
\le f(\y) - f(\bmu^\star) 
\quad \forall \y \in \dom(f).
\end{equation}
We assume that $\cY \subseteq \dom(\Omega)$ so it suffices to upper-bound
$f(\y^\star) - f(\bmu^\star)$, where $f(\bmu) \coloneqq \Omega(\bmu) -
\DP{\s}{\bmu}$.
Since $L \le \Omega(\bmu) \le U$ for all $\bmu \in \dom(\Omega)$, we have
$\DP{\s}{\bmu^\star} \ge -f(\bmu^\star) + L$ and
$-f(\y^\star) + U \ge \DP{\s}{\y^\star}$. Together with
$\DP{\s}{\y^\star} \ge \DP{\s}{\bmu^\star}$, this implies $f(\y^\star) -
f(\bmu^\star) \le U - L$. Therefore, $\frac{1}{2} \|\y^\star - \bmu^\star\|^2
\le \frac{U-L}{\gamma}$.

\paragraph{Temperature scaling \& constant invariance.} 

These immediately follow from properties of the $\argmax$ operator.


\subsection{Proof of Proposition \ref{prop:Bregman_div}}
\label{appendix:proof_Bregman_div}

We set $\Omega \coloneqq \Psi + I_\cC$.

\paragraph{Bregman projections.}

If $\Psi$ is Legendre type, then $\nabla \Psi(\nabla \Psi^*(\s)) = \s$ for
all $\s \in \interior(\dom(\Psi^*))$. Using this and our assumption that
$\dom(\Psi^*)=\RR^d$, we get for all $\s \in \RR^d$:
\begin{equation}
B_\Psi(\p \| \nabla \Psi^*(\s)) 
= \Psi(\p) - \DP{\s}{\p} + \DP{\s}{\nabla \Psi^*(\s)} - \Psi(\nabla
\Psi^*(\s)).
\label{eq:Bregman_div_nabla_Omega_conj}
\end{equation}
The last two terms are independent of $\p$ and therefore
\begin{equation}
\yHatOmega(\s)
= \argmax_{\p \in \cC} \DP{\s}{\p} - \Psi(\p)
= \argmin_{\p \in \cC} B_\Psi(\p \| \nabla \Psi^*(\s)),
\end{equation}
where $\cC \subseteq \dom(\Psi)$.
The r.h.s.\ is the
Bregman projection of $\nabla \Psi^*(\s)$ onto $\cC$.

\paragraph{Difference of Bregman divergences.}

Let $\p = \yHatOmega(\s)$.
Using \eqref{eq:Bregman_div_nabla_Omega_conj}, we obtain
\begin{align}
B_\Psi(\y \| \nabla \Psi^*(\s)) - 
B_\Psi(\p \| \nabla \Psi^*(\s))
&= \Psi(\y) - \DP{\s}{\y} + \DP{\s}{\p} - \Psi(\p)
\\
&= \Omega(\y) - \DP{\s}{\y} + \Omega^*(\s) \\
&= L_{\Omega}(\s; \y),
\label{eq:Bregman_diff}
\end{align}
where we assumed $\y \in \cC$ and $\cC \subseteq \dom(\Psi)$,
implying $\Psi(\y) = \Omega(\y)$.

If $\cC = \dom(\Psi)$ (i.e., $\Omega = \Psi$), then $\p = \nabla
\Psi^*(\s)$ and $B_\Psi(\p \| \nabla \Psi^*(\s)) = 0$. We thus get
the \textbf{composite form} of Fenchel-Young losses
\begin{equation}
B_\Omega(\y \| \nabla \Omega^*(\s))    
= B_\Omega(\y \| \yHatOmega(\s))    
= L_\Omega(\s; \y).
\end{equation}

\paragraph{Bound.}

Let $\p = \yHatOmega(\s)$. Since $\p$ is the Bregman projection of $\nabla
\Psi^*(\s)$ onto $\cC$,
we can use the well-known Pythagorean theorem for
Bregman divergences (see, e.g., \citet[Appendix A]{bregman_clustering}) to
obtain for all $\y \in \cC \subseteq \dom(\Psi)$:
\begin{equation}
B_\Psi(\y \| \p) + B_\Psi(\p \| \nabla \Psi^*(\s))
\le B_\Psi(\y \| \nabla \Psi^*(\s)).
\end{equation}
Using \eqref{eq:Bregman_diff}, we obtain for all $\y \in \cC$:
\begin{equation}
    0 \le B_\Psi(\y \| \p) \le L_{\Omega}(\s; \y).
\end{equation}

Since $\Omega$ is a l.s.c.\ proper convex
function, from Proposition \ref{prop:fy_losses}, we immediately get
\begin{equation}
\p = \y 
\Leftrightarrow
L_{\Omega}(\s; \y) = 0
\Leftrightarrow
B_\Psi(\y \| \p) = 0.
\end{equation}

\subsection{Proof of Proposition \ref{prop:generalized_entropy}}
\label{appendix:proof_generalized_entropy}

The two facts stated in Proposition \ref{prop:generalized_entropy} ($\HH$ is
always non-negative and maximized by the uniform distribution) follow directly
from Jensen's inequality.  Indeed, for all $\p \in \triangle^d$:
\begin{itemize}

\item $\HH(\p) \ge \sum_{j=1}^d p_j \HH(\e_j) = 0$; 

\item $\HH(\mathbf{1}/ d) = \HH\left(\sum_{\bm{P} \in \mathcal{P}} \frac{1}{d!}
\bm{P}\p \right) \ge \sum_{\bm{P} \in \mathcal{P}} \frac{1}{d!}
\HH(\bm{P} \p) = \HH(\p)$,

\end{itemize}
where $\mathcal{P}$ is the set of $d \times d$ permutation matrices. 
Strict concavity ensures that $\p = \ones / d$ is the unique maximizer.

\subsection{Proof of Proposition~\ref{prop:margin}}\label{appendix:proof_margin}

We start by proving the following lemma. 
\vspace{0.5em}
\begin{lemma}\label{prop:inverse_map}
Let $\HH$ satisfy assumptions A.1--A.3. 
Then:
\begin{enumerate}
\item We have $\s \in \partial (-\HH)(\e_k)$ iff $\ss_k = (-\HH)^*(\s)$. That is:
\begin{equation}
\partial (-\HH)(\e_k) = \{\s \in \RR^d \colon \ss_k \ge \DP{\s}{\p} + \HH(\p), \,\,\forall \p \in \triangle^d\}.
\end{equation}
\item If $\s \in \partial (-\HH)(\e_k)$, 
then, we also have $\s' \in \partial (-\HH)(\e_k)$ 
for any $\s'$ such that $\ss_k' = \ss_k$ and 
$\ss_i' \le \ss_i$, for all $i \ne k$.
\end{enumerate}
\end{lemma}

\paragraph{Proof of the lemma:} 
Let $\Omega = -\HH$. 
From Proposition~\ref{prop:prediction_func} (order preservation), we can consider $\partial \Omega(\e_1)$ without loss of generality, in which case any $\s \in \partial \Omega(\e_1)$ satisfies $\ss_1 = \max_{j} \ss_j$. 
We have $\s \in \partial \Omega(\e_1)$ iff 
$\Omega(\e_1) = \DP{\s}{\e_1} - \Omega^*(\s) = \ss_1 - \Omega^*(\s)$. 
Since  $\Omega(\e_1)=0$, 
we must have $\ss_1 = \Omega^*(\s) \ge \sup_{\p\in \triangle^d} \DP{\s}{\p} - \Omega(\p)$, which proves part 1. 
To see 2, note that we have $\ss_k'=\ss_k \ge \DP{\s}{\p} - \Omega(\p) \ge \DP{\s'}{\p} - \Omega(\p)$, for all $\p \in \triangle^d$, from which the result follows. \hfill $\blacksquare$

We now proceed to the proof of Proposition~\ref{prop:margin}. 
Let $\Omega = -\HH$, and suppose that $L_{\Omega}$ has the separation margin property. Then, $\s = m\e_1$ satisfies the margin condition $\ss_1 \ge m + \max_{j\ne 1} \ss_j$, hence $L_{\Omega}(m\e_1, \e_1) = 0$. 
From the first part of Proposition~\ref{prop:fy_losses}, this implies $m\e_1 \in \partial \Omega(\e_1)$. 

Conversely, 
let us assume that $m\e_1 \in \partial \Omega(\e_1)$.
From the second part of Lemma~\ref{prop:inverse_map}, 
this implies that 
$\s \in \partial \Omega(\e_1)$ for 
any $\s$ such that $\ss_1 = m$ and $\ss_i \le 0$ for all $i\ge 2$; and more generally we have $\s + c\mathbf{1} \in \partial \Omega(\e_1)$. 
That is, any $\s$ with $\ss_1 \ge m + \max_{i\ne 1} \ss_i$ satisfies 
$\s \in \partial \Omega(\e_1)$. From Proposition~\ref{prop:fy_losses}, 
this is equivalent to $L_{\Omega}(\s;\e_1) = 0$. 

Let us now determine the margin of $L_{\Omega}$, i.e., the 
smallest $m$ such that 
$m\e_1 \in \partial \Omega(\e_1)$.  
From Lemma~\ref{prop:inverse_map}, 
this is equivalent to 
$m \ge m \pp_1 - \Omega(\p)$ for any $\p \in \triangle^d$, i.e., 
$\frac{-\Omega(\p)}{1-\pp_1} \le m$. 
Note that by Proposition~\ref{prop:prediction_func} the ``most competitive'' $\p$'s are sorted as $\e_1$, so we may write $\pp_1 = \|\p\|_\infty$ without loss of generality. The margin of $L_{\Omega}$ is the smallest possible such margin, given by \eqref{eq:margin}.

\subsection{Proof of Proposition~\ref{prop:full_simplex}}\label{appendix:proof_full_simplex}

Let us start by showing that conditions 1 and 2 are equivalent. 
To show that 2 $\Rightarrow$ 1, take an arbitrary $\p \in \triangle^d$. From Fenchel-Young duality and the Danskin's theorem, we have that $\nabla(-\HH)^*(\s) = \p \,\, \Rightarrow \,\, \s \in \partial(-\HH)(\p)$, which implies the subdifferential set is non-empty everywhere in the simplex. Let us now prove that 1 $\Rightarrow$ 2. Let $\Omega = -{\HH}$, and assume that $\Omega$ has non-empty subdifferential everywhere in $\triangle^d$.  
We need to show that for any $\p \in \triangle^d$, there is some $\s \in \RR^d$ such that 
$\p \in \argmin_{\p' \in \triangle^d} \Omega(\p') - \DP{\s}{\p'}$. 
The Lagrangian associated with this minimization problem is:
\begin{equation}
\mathcal{L}(\p, \bm{\mu}, \lambda) = \Omega(\p) - \DP{\s + \bm{\mu}}{\p} + \lambda (\mathbf{1}^\top \p - 1).
\end{equation}
The KKT conditions are:
\begin{eqnarray}
\left\{
\begin{array}{l}
0 \in \partial_p \mathcal{L}(\p, \bm{\mu}, \lambda) = 
\partial \Omega(\p) - \s - \bm{\mu} + \lambda \mathbf{1}\\
\DP{\p}{\bm{\mu}} = 0\\
\p \in \triangle^d, \,\, \bm{\mu} \ge 0.
\end{array}
\right.
\end{eqnarray}
For a given $\p \in \triangle^d$, we seek $\s$ such that 
$(\p, \bm{\mu}, \lambda)$ are a solution to the KKT conditions 
for some $\bm{\mu} \ge 0$ and $\lambda \in \RR$. 

We will show that such $\s$ exists by simply choosing 
$\bm{\mu}=\mathbf{0}$ and $\lambda=0$. 
Those choices are dual feasible and guarantee that the slackness complementary condition is satisfied. 
In this case, we have from the first condition that 
$\s \in \partial \Omega(\p)$. 
From the assumption that $\Omega$ has non-empty subdifferential in all the simplex, we have that for any $\p \in \triangle^d$ we can find a 
$\s \in \RR^d$ such that $(\p, \s)$ are a dual pair, i.e., $\p = \nabla \Omega^*(\s)$, which proves that $\nabla \Omega^*(\RR^d) = \triangle^d$.  

Next, we show that condition $1 \Rightarrow 3$. Since $\partial (-\HH)(\p) \ne \varnothing$ everywhere in the simplex, we can take an arbitrary $\s \in \partial (-\HH)(\e_k)$. From  Lemma~\ref{prop:inverse_map}, item 2, we have that $\s' \in \partial (-\HH)(\e_k)$ for $\ss'_k = \ss_k$ and $\ss'_j = \min_\ell \ss_\ell$; since $(-\HH)^*$ is shift invariant, we can without loss of generality have $\ss' = m\e_k$ for some $m>0$, which implies from Proposition~\ref{prop:margin} that $L_\Omega$ has a margin. 

Let us show that, if $-\HH$ is separable, then $3 \Rightarrow 1$, which establishes equivalence between all conditions 1, 2, and 3. 
From Proposition~\ref{prop:margin}, the existing of a separation margin implies that there is some $m$ such that $m\e_k \in \partial (-\HH)(\e_k)$. 
Let $\HH(\p) = \sum_{i=1}^d h(p_i)$, with $h:[0,1]\rightarrow \mathbb{R}_+$ concave. Due to assumption A.1, $h$ must satisfy $h(0)=h(1)=0$. 
Without loss of generality, suppose $\p = [\tilde{\p}; \mathbf{0}_k]$, where  $\tilde{\p} \in \relint(\triangle^{d-k})$ and $\mathbf{0}_k$ is a vector with $k$ zeros. 
We will see that there is a vector $\bm{g} \in \RR^d$ such that 
$\bm{g} \in \partial (-\HH)(\p)$, i.e., satisfying 
\begin{equation}
-\HH(\p') \ge -\HH(\p) + \langle \bm{g}, \p'-\p \rangle, \quad \forall \p' \in \triangle^d. 
\label{eq:subgrad_ineq}
\end{equation}
Since $\tilde{\p} \in \relint(\triangle^{d-k})$, we have $\tilde{p}_i \in ]0,1[$
for $i \in \{1,\ldots,d-k\}$, hence $\partial (-h) (\tilde{p}_i)$ must be
nonempty, since $-h$ is convex and $]0,1[$ is an open set. We show that the following $\bm{g} = (g_1, \ldots, g_d) \in \RR^d$ is a subgradient of $-\HH$ at $\p$:
\begin{equation*}
g_i = \left\{
\begin{array}{ll}
\partial (-h) (\tilde{p}_i), & \text{$i=1,\ldots,d-k$}\\
m, &  \text{$i=d-k+1,\ldots,d$}.
\end{array}
\right.
\end{equation*}
By definition of subgradient, we have 
\begin{equation}
-\psi(p_i') \ge -\psi(\tilde{p}_i) + \partial (-h) (\tilde{p}_i)(p_i'-\tilde{p}_i), \quad \text{for $i=1,\ldots,d-k$}.
\label{eq:subgrad_ineq1}
\end{equation}
Furthermore, since $m$ upper bounds the separation margin of $\HH$, we have from Proposition~\ref{prop:margin} that 
$m \ge \frac{\HH([1-p_i', p_i', 0, \ldots, 0])}{1 - \max\{1-p_i', p_i'\}} =
\frac{h(1-p_i') + h(p_i')}{\min\{p_i', 1-p_i'\}} \ge \frac{h(p_i')}{p_i'}$ for any $p_i' \in ]0,1]$. Hence, we have 
\begin{equation}
-\psi(p_i') \ge -\psi(0) - m(p_i' - 0), \quad \text{for $i=d-k+1,\ldots,d$}.\label{eq:subgrad_ineq2}
\end{equation}
Summing all inequalities in Eqs.~\eqref{eq:subgrad_ineq1}--\eqref{eq:subgrad_ineq2}, we obtain the expression in Eq.~\eqref{eq:subgrad_ineq}, which finishes the proof.

\subsection{Proof of Proposition~\ref{prop:margin_separable}}\label{appendix:proof_margin_separable}

Define $\Omega = -\HH$.
Let us start by writing the margin expression \eqref{eq:margin} as a unidimensional optimization problem. This is done by noticing that the max-generalized entropy problem constrained to $\max(\p)=1-t$ gives $\p=\left[1-t, \frac{t}{d-1}, \ldots, \frac{t}{d-1}\right]$, for $t \in \left[0, 1-\frac{1}{d}\right]$ by a similar argument as the one used in Proposition~\ref{prop:generalized_entropy}. We obtain:
\begin{equation}
\mathrm{margin}(L_{\Omega}) = \sup_{t \in \left[0, 1-\frac{1}{d}\right]} \frac{-\Omega\left(\left[1-t, \frac{t}{d-1}, \ldots, \frac{t}{d-1}\right]\right)}{t}.
\end{equation}

We write the argument above as $A(t) = \frac{-\Omega(\e_1 + t\bm{v})}{t}$, where $\bm{v} := [-1, \frac{1}{d-1}, \ldots, \frac{1}{d-1}]$. 
We will first prove that $A$ is decreasing in $[0, 1-\frac{1}{d}]$, which implies that the supremum (and the margin) equals $A(0)$.  
Note that we have the following expression for the derivative of any function $f(\e_1 + t\bm{v})$:
\begin{equation}
(f(\e_1 + t\bm{v}))' = \bm{v}^\top \nabla f(\e_1 + t\bm{v}).
\end{equation}
Using this fact, we can write the derivative $A'(t)$ as:
\begin{equation}
A'(t) = \frac{-t\bm{v}^\top \nabla \Omega(\e_1 + t\bm{v}) + \Omega(\e_1 + t\bm{v})}{t^2} := \frac{B(t)}{t^2}.
\end{equation}
In turn, the derivative $B'(t)$ is:
\begin{eqnarray}
B'(t) &=& -\bm{v}^\top \nabla \Omega(\e_1 + t\bm{v}) -t(\bm{v}^\top\nabla \Omega(\e_1 + t\bm{v}))' + \bm{v}^\top \nabla \Omega(\e_1 + t\bm{v}) \nonumber\\
&=& -t(\bm{v}^\top\nabla \Omega(\e_1 + t\bm{v}))' \nonumber\\
&=& -t\bm{v}^\top\nabla\nabla \Omega(\e_1 + t\bm{v})\bm{v} \nonumber\\
&\le& 0,
\end{eqnarray}
where we denote by $\nabla\nabla \Omega$ the Hessian of $\Omega$, 
and used the fact that it is positive semi-definite, due to the convexity of $\Omega$. 
This implies that $B$ is decreasing, hence for any $t \in [0,1]$, $B(t) \le B(0) = \Omega(\e_1) = 0$, where we used the fact 
$\|\nabla \Omega(\e_1)\| < \infty$, assumed as a condition of Proposition~\ref{prop:full_simplex}. 
Therefore, we must also have $A'(t) = \frac{B(t)}{t^2} \le 0$ for any $t \in [0,1]$, 
hence $A$ is decreasing, and $\sup_{t \in [0, 1-1/d]} A(t) = \lim_{t\rightarrow 0+} A(t)$. 
By L'H\^opital's rule:
\begin{eqnarray}
\lim_{t\rightarrow 0+} A(t) &=& \lim_{t\rightarrow 0+} (-\Omega(\e_1 + t\bm{v}))' \nonumber\\
&=& -\bm{v}^\top \nabla \Omega(\e_1) \nonumber\\
&=& {\nabla_1 \Omega(\e_1)} - \frac{1}{d-1}\sum_{j\ge 2}{\nabla_j \Omega(\e_1)}\nonumber\\ 
&=& {\nabla_1 \Omega(\e_1)} - {\nabla_2 \Omega(\e_1)},
\end{eqnarray} 
which proves the first part. 

If $\Omega$ is separable, then 
$\nabla_j \Omega(\p) = -h'(\pp_j)$, in particular
${\nabla_1\Omega(\e_1)} = -h'(1)$ and 
${\nabla_2\Omega(\e_1)} = -h'(0)$, yielding 
$\mathrm{margin}(L_{\Omega}) = h'(0) - h'(1)$. 
Since $h$ is twice differentiable, this equals $-\int_{0}^{1} h''(t)dt$, completing the proof.

\subsection{Proof of Proposition~\ref{prop:structured_margin}}\label{appendix:proof_structured_margin}

We start by proving the following lemma, which generalizes Lemma~\ref{prop:inverse_map}. 
\vspace{0.5em}
\begin{lemma}\label{prop:inverse_map_structured}
Let $\Omega$ be convex. 
Then:
\begin{enumerate}
\item We have 
$\partial \Omega(\y) = \{\s \in \RR^d \colon \DP{\s}{\y} - \Omega(\y) \ge \DP{\s}{\bmu} - \Omega(\bmu), \,\,\forall \bmu \in \conv(\cY)\}$.
\item If $\s \in \partial \Omega(\y)$, 
then, we also have $\s' \in \partial \Omega(\y)$ 
for any $\s'$ such that $\DP{\s'-\s}{\y'} \le \DP{\s'-\s}{\y}$, for all $\y' \in \cY$.
\end{enumerate}
\end{lemma}

\paragraph{Proof of the lemma:} 
The first part comes directly from the definition of Fenchel conjugate. 
To see the second part, note that, if $\s \in \partial \Omega(\y)$, then 
$\DP{\s'}{\y} - \Omega(\y) = \DP{\s'-\s}{\y} + \DP{\s}{\y} - \Omega(\y) \ge \DP{\s'-\s}{\y} + \DP{\s}{\bmu} - \Omega(\bmu)$ for every $\bmu \in \conv(\cY)$. Let $\bmu = \sum_{\y'\in\cY} p(\y') \y'$ for some distribution $\p \in \triangleY$. 
Then, we have $\DP{\s'-\s}{\y} + \DP{\s}{\bmu}  - \Omega(\bmu) = \DP{\s'-\s}{\y} + \sum_{\y'\in\cY} p(\y') \DP{\s}{\y'}  - \Omega(\bmu) \ge \DP{\s'-\s}{\y} + \sum_{\y'\in\cY} p(\y') (\DP{\s'}{\y'} - \DP{\s'-\s}{\y}) - \Omega(\bmu) = \DP{\s'}{\bmu} - \Omega(\bmu)$.

We now proceed to the proof of Proposition~\ref{prop:structured_margin}. 
Suppose that $L_{\Omega}$ has the separation margin property, and let $\s = m\y$. Then, for any $\y' \in \cY$, we have 
$\DP{\s}{\y'} + \frac{m}{2}\|\y-\y'\|^2 = m\DP{\y}{\y'} + \frac{m}{2}\|\y-\y'\|^2 = \frac{m}{2}\|\y\| + \frac{m}{2}\|y'\|^2 = mr^2 = \DP{m\y}{\y}=\DP{\s}{\y}$, which implies that $L_{\Omega}(\s,\y)=0$. From the first part of Proposition~\ref{prop:fy_losses}, this implies $m\y \in \partial \Omega(\y)$. 

Conversely, 
let us assume that $m\y \in \partial \Omega(\y)$.
From Lemma~\ref{prop:inverse_map_structured}, 
this implies that 
$\s \in \partial \Omega(\y)$ for 
any $\s$ such that $\DP{\s-m\y}{\y'} \le \DP{\s-m\y}{\y}$ for all $\y'\in \cY$. 
Therefore, we have $\DP{\s}{\y} - mr^2 \ge \DP{\s}{\y'} - m\DP{\y}{\y'}$, that is, $\DP{\s}{\y} = \DP{\s}{\y'} + \frac{m}{2}\|\y-\y'\|^2$. That is, any such $\s$ satisfies 
$\s \in \partial \Omega(\y)$. From Proposition~\ref{prop:fy_losses}, 
this is equivalent to $L_{\Omega}(\s;\y) = 0$. 

Let us now determine the margin of $L_{\Omega}$, i.e., the 
smallest $m$ such that 
$m\y \in \partial \Omega(\y)$.  
From Lemma~\ref{prop:inverse_map_structured}, 
this is equivalent to 
$mr^2 - \Omega(\y) \ge \DP{m\y}{\bmu} - \Omega(\bmu)$ for any $\bmu \in \conv(\cY)$, i.e., 
$\frac{-\Omega(\bmu)+\Omega(\y)}{r^2-\DP{\y}{\bmu}} \le m$, which leads to the expression \eqref{eq:structured_margin}.

\subsection{Proof of Proposition~\ref{prop:root_finding}}
\label{appendix:proof_root_finding}

Let $\Omega(\p) = \sum_{j=1}^d g(p_j) + I_{\triangle^d}(\p)$, where
$g \colon [0,1]\rightarrow\RR_+$ is a
non-negative, strictly convex, differentiable function.
Therefore, $g'$ is strictly monotonic on $[0,1]$, thus invertible.
We show how computing $\nabla(\Omega)^*$ reduces to finding the root of
a monotonic scalar function, for which efficient algorithms 
are available.

From strict convexity and the definition of the convex conjugate,
\begin{equation}
    \nabla(\Omega)^*(\s) = \argmax_{\p\in\triangle^d}
    \DP{\p}{\s} - \sum_j g(\pp_j).
\end{equation}
The constrained optimization problem above has Lagrangian
\begin{equation}
\mathcal{L}(\p, \bm{\nu}, \thresh) \coloneqq 
\sum_{j=1}^d g(\pp_j) - \DP{\s + \bm{\nu}}{\p} + \thresh (\mathbf{1}^\top \p - 1).
\end{equation}
A solution
$(\p^\star, \bm{\nu}^\star, \thresh^\star)$
must satisfy the KKT conditions
\begin{empheq}[left=\empheqlbrace]{align}
    g'(\pp_j) - \ss_j - \nu_j + \thresh &=0 \qquad \forall j \in [d]
        \label{eqn:stationarity_separable}\\
    \DP{\p}{\bm{\nu}}&= 0
        \label{eqn:complementary_slack_separable}\\
    \p \in \triangle^d, \,\, \bm{\nu}&\ge 0.
\end{empheq}
Let us define
\[
    \thresh_{\min} \coloneqq \max(\s) - g'(1)
    \quad \text{and} \quad
    \thresh_{\max} \coloneqq \max(\s) - g'\left(\frac{1}{d}\right).
\]
Since $g$ is strictly convex, $g'$ is increasing and so
$\thresh_{\min}<\thresh_{\max}$. For any $\thresh \in
[\thresh_{\min}, \thresh_{\max}]$,
we construct $\bm{\nu}$ as
\[
    \nu_j \coloneqq \begin{cases}
        0, & \ss_j - \thresh \geq g'(0) \\
        g'(0) - \ss_j + \thresh, & \ss_j - \thresh < g'(0) \\
    \end{cases}
\]
By construction, $\nu_j \geq 0$, satisfying dual feasibility.
Injecting $\nu$ into \eqref{eqn:stationarity_separable}
and combining the two cases, we obtain
\begin{equation}\label{eq:bisect-station-each}
    g'(\pp_j) = \max\{\ss_j - \thresh,~g'(0)\}.
\end{equation}

We show that i) the stationarity conditions have a unique solution given
$\thresh$, and ii) $[\thresh_{\min}, \thresh_{\max}]$ forms a
sign-changing bracketing interval, and thus contains $\tau^\star$, which can
then be found by one-dimensional search. The solution verifies all KKT conditions,
thus is globally optimal.

\paragraph{Solving the stationarity conditions.}
Since $g$ is strictly convex, its derivative $g'$ is continuous and strictly
increasing, and is thus a one-to-one mapping between $[0, 1]$ and
$[g'(0), g'(1)]$. Denote by $(g')^{-1} \colon [g'(0), g'(1)] \rightarrow [0, 1]$
its inverse. If $\ss_j - \thresh \ge g'(0)$,
we have 
\begin{equation}
\begin{aligned}
    g'(0) \leq 
g'(\pp_j) = \ss_j - \thresh 
                                    &\leq \max(\s) - \thresh_{\min} \\
                                    &= \max(\s) - \max(\s) + g'(1) \\
                                    &= g'(1). \\
\end{aligned}
\end{equation}
Otherwise,
$g'(\pp_j) = g'(0)$. This verifies that the r.h.s.\ of \eqref{eq:bisect-station-each}
is always within the domain of $(g')^{-1}$. We can thus apply the inverse to both
sides to solve for $\pp_j$, obtaining
\begin{equation}\label{eq:bisectprimaldual}
    \pp_j(\thresh) = (g')^{-1}(\max\{\ss_j - \thresh,~g'(0)\}).
\end{equation}
Strict convexity implies the optimal $\p^\star$ is unique; it can be seen
that $\tau^\star$ is also unique. Indeed, assume optimal $\tau^\star_1,
\tau^\star_2$. Then, $\p(\tau^\star_1) = \p(\tau^\star_2)$, so
$\max(\s-\tau^\star_1, g'(0)) =  \max(\s-\tau^\star_2, g'(0))$.
This implies either $\tau^\star_1 = \tau^\star_2$, or $\s -
\tau^\star_{\{1,2\}} \leq g'(0)$, in which case $\p=\zeros \notin \triangle^d$,
which is a contradiction.

\paragraph{Validating the bracketing interval.}

Consider the primal infeasibility function
$\phi(\thresh) \coloneqq \DP{\p(\thresh)}{\ones} - 1$;
$\p(\thresh)$ is primal feasible iff $\phi(\thresh)=0.$
We show that $\phi$ is decreasing on $[\thresh_{\min}, \thresh_{\max}]$, and that
it has opposite signs at the two extremities.
From the intermediate value theorem, the unique
root $\thresh^\star$ must satisfy $\thresh^\star\in[\thresh_{\min}, \thresh_{\max}]$.

Since $g'$ is increasing, so is $(g')^{-1}$.
Therefore, for all $j$, $\pp_j(\thresh)$ is decreasing, and so is the
sum $\phi(\thresh) = \sum_j \pp_j(\thresh) - 1$. It remains to check the signs
at the boundaries.
\begin{equation}
\begin{aligned}
    \sum_i \pp_i(\thresh_{\max})
    &= \sum_i (g')^{-1}(\max\{\ss_i - \max(\s) +
    g'\left(\nicefrac{1}{d}\right),~g'(0)\})\\
    &\leq d ~ (g')^{-1}(\max\{g'\left(\nicefrac{1}{d}\right),~g'(0)\})\\
    &= d ~ (g')^{-1}\left(g'\left(\nicefrac{1}{d}\right)\right) = 1,\\
\end{aligned}
\end{equation}
where we upper-bounded each term of the sum by the largest one. At the other end,
\begin{equation}
\begin{aligned}
    \sum_i \pp_i(\thresh_{\min})
    &= \sum_i (g')^{-1}(\max\{\ss_i - \max(\s) + g'(1),~g'(0)\})\\
    &\geq (g')^{-1}(\max\{g'(1),~g'(0)\})\\
    &= (g')^{-1}(g'(1)) = 1,\\
\end{aligned}
\end{equation}
using that a sum of non-negative terms is no less than its largest term.
Therefore, $\phi(\thresh_{\min}) \ge 0$ and $\phi(\thresh_{\max}) \le 0$.  This
implies that there must exist $\thresh^\star$ in $[\thresh_{\min}, \thresh_{\max}]$
satisfying $\phi(\thresh^\star)=0$.
The corresponding triplet $\left(\p(\thresh^\star), \bm{\nu}(\thresh^\star),
\thresh^\star\right)$ thus satisfies all of the KKT conditions, confirming that
it is the global solution.

Algorithm~\ref{algo:bisect} is an example of a bisection algorithm for finding
an approximate solution; more advanced root finding methods can also be used. We
note that the resulting algorithm resembles the method provided
in \cite{bregmanproj}, with a non-trivial difference being the order of the 
thresholding and $(-g)^{-1}$ in Eq.~\eqref{eq:bisectprimaldual}. 

\begin{center}
\begin{minipage}[t]{0.5\linewidth}
\begin{algorithm}[H]
    \caption{Bisection for $\yHatOmega(\s) = \nabla\Omega^*(\s)$}
\begin{algorithmic}
    \STATE \textbf{Input:} $\s \in \RR^d$,
    $\Omega(\p)= I_{\triangle^d} + \sum_i g(\pp_i)$
\STATE $\p(\thresh) \coloneqq (g')^{-1}(\max\{\s - \thresh,g'(0)\})$
\STATE $\phi(\thresh) \coloneqq \DP{\p(\thresh)}{\ones} - 1$
\STATE $\thresh_{\min} \leftarrow \max(\s) - g'(1)$;
\STATE $\thresh_{\max} \leftarrow \max(\s) - g'\left(\nicefrac{1}{d}\right)$
\STATE  $\thresh \leftarrow (\thresh_{\min} + \thresh_{\max}) / 2$
\STATE \textbf{while} $|\phi(\thresh)| > \epsilon$
\STATE \hspace{0.4cm}\textbf{if} $\phi(\thresh) < 0$\quad$\thresh_{\max}\leftarrow\thresh$ 
\STATE \hspace{0.4cm}\textbf{else}\hspace{45pt}$\thresh_{\min}\leftarrow\thresh$ 
\STATE \hspace{0.35cm} $\thresh \leftarrow (\thresh_{\min} + \thresh_{\max}) / 2$
\STATE \textbf{Output:} $\nabla \yHatOmega(\s) \approx \p(\thresh)$
\end{algorithmic}
\label{algo:bisect}
\end{algorithm}
\end{minipage}
\end{center}

\subsection{Proof of Proposition \ref{prop:pairwise_entropies}}
\label{appendix:proof_pairwise_entropies}

We first need the following lemma.
\vspace{0.5em}
\begin{lemma}
Let $\ell(\s; \e_k)$ be defined as
\begin{equation}
\ell(\s; \e_k) \coloneqq
\begin{cases}
    \sum_j c_{k,j} \phi(\ss_j) & \mbox { if } \s^\top \ones = 0 \\
\infty & \mbox{ o.w. }
\end{cases},
\end{equation}
where $\phi \colon \RR \to \RR_+$ is convex. 
Then, $\HH_\ell$ defined in \eqref{eq:H_L} equals
\begin{equation}
-\HH_\ell(\p) 
= \min_{\tau \in \RR} \sum_{j} 
(\p^\top \cc_j) \phi^*\left(\frac{-\tau}{\p^\top \cc_j}\right).
\end{equation}
\label{lemma:pairwise_entropy}
\end{lemma}
\begin{proof}
We want to solve
\begin{equation}
\HH_\ell(\p) 
= \min_{\s^\top \ones = 0} \sum_{j} p_j \ell(\s; \e_j) 
= \min_{\s^\top \ones = 0} \sum_{j} p_j \sum_{i} c_{j,i}
\phi(\ss_i)
= \min_{\s^\top \ones = 0} \sum_{i} (\p^\top \cc_i) \phi(\ss_i),
\end{equation}
where $\cc_i$ is a vector gathering $c_{j,i}$ for all $j$.
Introducing a Lagrange multiplier we get
\begin{equation}
\HH_\ell(\p) 
= \min_{\s \in \RRY} \max_{\tau \in \RR}
\sum_{i} (\p^\top \cc_i) \phi(\ss_i)
+ \tau \sum_i \ss_i.
\end{equation}
Strong duality holds and we can swap the order of the min and max.
After routine calculations, we obtain
\begin{equation}
-\HH_\ell(\p) 
= \min_{\tau \in \RR} \sum_{i} 
(\p^\top \cc_i) \phi^*\left(\frac{-\tau}{\p^\top \cc_i}\right).
\label{eq:pairwise_entropy}
\end{equation}
\end{proof}
We now prove Proposition \ref{prop:pairwise_entropies}.  First, we rewrite
$\ell(\s; \e_k) = \sum_{j \neq k} \phi(\ss_j - \ss_k)$ as $\ell(\s; \e_k) = \sum_{j}
c_{k,j} \phi(\ss_j - \ss_k)$, where we choose $c_{k,j}=0$ if $k=j$, $1$
otherwise, leading to $\p^\top \cc_j = 1 - p_j$.  Because $\phi(\ss_j - \ss_k)$
is shift invariant \wrt $\s$, without loss of generality, we can further rewrite
the loss as $\ell(\s; \e_k) = \sum_{j} c_{k,j} \phi(\ss_j)$ with $\dom(\ell) = \{\s
\in \RRY \colon \s^\top \ones = 0\}$. Hence, Lemma
\ref{lemma:pairwise_entropy} applies.  We now derive closed form for
specific choices of $\phi$.

\paragraph{Hinge loss.} 

When $\phi(t) = [1 + t]_+$, the conjugate is
\begin{equation}
    \phi^*(u) =
    \begin{cases}
        -u & \mbox{ if } u \in [0, 1] \\
        \infty & \mbox{ o.w. }
    \end{cases}.
\end{equation}

The constraint set for $\tau$ is therefore $\cC \coloneqq
\bigcap_{j \in [d]} \left[-\p^\top \cc_j, 0 \right] =
\left[-\displaystyle{\min_{j \in [d]}} ~ \p^\top \cc_j, 0\right]$.
Hence
\begin{equation}
-\HH_\ell(\p) 
= \min_{\tau \in \cC} d \tau
= -d \min_{j \in [d]} \p^\top \cc_j.
\end{equation}
This recovers \citet[Example 5, \S A.6]{duchi_2016} with a simpler proof.  We
next turn to the following new results.

\paragraph{Smoothed hinge loss.} 

We add quadratic regularization to the conjugate \citep{accelerated_sdca}:
\begin{equation}
\phi^*(u) = 
\begin{cases}
    -u + \frac{1}{2} u^2 & \mbox { if } u \in [0, 1] \\
    \infty & \mbox{ o.w. }
\end{cases}.
\end{equation}
Going back to $\phi$, we obtain:
\begin{equation}
\phi(t) = 
\begin{cases}
    0 & \mbox { if } t \le -1 \\
     t + \frac{1}{2} & \mbox { if } t \ge 0 \\
    \frac{1}{2} (1 + t)^2 & \mbox{ o.w. }
\end{cases}.
\label{eq:smoothed_hinge}
\end{equation}
The constraint set for $\tau$ is the same as before,
$\cC = \left[-\displaystyle{\min_{j \in [d]}} ~ \p^\top \cc_j, 0\right]$.

Plugging $\phi^*$ into \eqref{eq:pairwise_entropy}, we obtain
\begin{equation}
-\HH_\ell(\p) 
= \min_{\tau \in \cC} 
\frac{\tau^2}{2} \sum_{j=1}^d \frac{1}{\p^\top \cc_j} + d ~ \tau.
\label{eq:H_L_smoothed_hinge}
\end{equation}

Since the problem is unidimensional, let us solve for $\tau$ unconstrained
first:
\begin{equation}
    \tau = -d \Big / (\sum_{j=1}^d 1 / (\p^\top \cc_j)).
\label{eq:tau_unconstrained}
\end{equation}
We notice that $\tau \le -\min_j \p^\top \cc_j$ for $\cc_j \ge
\zeros$ since
$\sum_j \frac{\min_i \p^\top \cc_i}{\p^\top \cc_j} \in [0,d]$. 
This expression of $\tau$ is not feasible.
Hence the optimal
solution is at the boundary and $\tau^\star = -\min_j \p^\top \cc_j$.
Plugging that expression back into \eqref{eq:H_L_smoothed_hinge} gives the
claimed expression of $\HH_\ell$.

\paragraph{Squared hinge loss.} 

When $\phi(t) = \frac{1}{2} [1 + t]^2_+$, the conjugate is
\begin{equation}
\phi^*(u) = 
\begin{cases}
    -u + \frac{1}{2} u^2 & \mbox { if } u \ge 0 \\
    \infty & \mbox{ o.w. }
\end{cases}.
\end{equation}
The constraint is now $\tau \le 0$. Hence, the optimal solution of the
problem \wrt $\tau$ in \eqref{eq:pairwise_entropy} is now
\eqref{eq:tau_unconstrained} for all $\p \in \triangle^d, \cc_j \ge \zeros$.
Simplifying, we get
\begin{equation}
\HH_\ell(\p) 
= \frac{\frac{1}{2} d^2}{\sum_{j=1}^d 1/(\p^\top \cc_j)}.
\end{equation}

\newpage


\end{document}